\documentclass{article} 
\usepackage{iclr2027_conference,times}

\usepackage{amsmath,amsfonts,bm}

\def\eqref#1{equation~\ref{#1}}

\def\1{\bm{1}}

\DeclareMathAlphabet{\mathsfit}{\encodingdefault}{\sfdefault}{m}{sl}
\SetMathAlphabet{\mathsfit}{bold}{\encodingdefault}{\sfdefault}{bx}{n}

\usepackage{hyperref}
\usepackage{url}
\usepackage{graphicx}
\usepackage{subcaption}
\usepackage{booktabs}
\usepackage{xcolor}
\usepackage{amssymb}

\newcommand{\nl}{\mathrm{nl}}

\newcommand{\FV}{\mathrm{FV}}
\newcommand{\LP}{\mathrm{LP}}
\newcommand{\MC}{\mathrm{MC}}
\newcommand{\FQ}{\mathrm{FQ}}
\newcommand{\OA}{\mathrm{OA}}
\newcommand{\LCB}{\mathrm{LCB}}
\newcommand{\soft}{\mathrm{soft}}

\usepackage{amsthm}
\theoremstyle{plain}
\newtheorem{theorem}{Theorem}[section]

\theoremstyle{definition}
\newtheorem{definition}[theorem]{Definition}
\newtheorem{assumption}[theorem]{Assumption}
\theoremstyle{remark}

\usepackage{algorithm}
\usepackage{algorithmic}

\usepackage{multirow}
\usepackage{arydshln}
\newcommand{\mycdashline}[1]{%
    \noalign{\vskip 2pt}  
    \cdashline{#1}        
    \noalign{\vskip 2pt}  
}

\newcommand{\inc}{$\mathbin{\textcolor{green!60!black}{+}}$}
\newcommand{\dec}{$\mathbin{\textcolor{red!70!black}{\bm{-}}}$}

\usepackage[skins,breakable,listings]{tcolorbox}

\title{Multi-Dimensional Reference-Free \\Refinement for Autoformalization}

\author{Lan Zhang\textsuperscript{1},
  Marco Valentino\textsuperscript{2},
  Andr\'e Freitas\textsuperscript{1,3}\\
  \textsuperscript{1}Department of Computer Science, University of Manchester, United Kingdom\\
  \textsuperscript{2}School of Computer Science, University of Sheffield, United Kingdom\\
  \textsuperscript{3}Idiap Research Institute, Switzerland\\
  \texttt{lan.zhang-6@postgrad.manchester.ac.uk}
}

\iclrfinalcopy 
\begin{document}

\maketitle
\lhead{Preprint}
\begin{abstract}
While statement autoformalization has advanced rapidly, full-theorem autoformalization remains largely unexplored. Existing iterative refinement methods for statement autoformalization typically improve isolated aspects of formalization quality, such as syntactic correctness, but struggle to jointly optimize multiple dimensions of quality, a key requirement for full-theorem autoformalization. We introduce a reference-free iterative refinement process for full-theorem autoformalization at inference time that leverages complementary feedback from theorem provers and LLM-based judges, without access to existing formalizations or human intervention. Our approach optimizes a masked composite objective over a hard dimension, formal validity as determined by a theorem prover, and a set of soft dimensions. We introduce an acceptance policy that monotonically improves the estimated composite objective and establish conditional theoretical guarantees characterizing its convergence and behavior under judge uncertainty. We also propose a responsiveness map that characterizes how different LLMs, acting in different roles, preferentially improve these dimensions, enabling principled generator selection. Experiments show that our method improves multiple dimensions of formalization quality, achieving overall scores of 90.27\% on miniF2F and 52.45\% on ProofNet, with 49.18\% and 17.74\% of instances having maximum scores across all evaluation dimensions, respectively.
\end{abstract}

\section{Introduction}
Recent advances in Large Language Models (LLMs) have substantially improved their ability to perform mathematical reasoning \citep{shao2024deepseekmathpushinglimitsmathematical, shao2025deepseekmathv2selfverifiablemathematicalreasoning, petrov2025proofbluffevaluatingllms, zhang2025deeptheoremadvancingllmreasoning}. However, reasoning trajectories expressed in natural-language mathematics are difficult to verify automatically, posing a major challenge for reliably evaluating instances generated through mathematical discovery~\citep{Romera-Paredes2024-pl,Naskrecki2025-uc}. Autoformalization~\citep{wu2022autoformalization, yang2025position} provides a crucial bridge between informal and formal mathematics by converting informal reasoning into transparent and mechanically verifiable formalizations. Existing autoformalization research has primarily focused on statement formalization~\citep{zhang-etal-2024-consistent,lu2024processdrivenautoformalizationlean4,li2024autoformalize,gao2025herald}, often omitting full formal proofs by relying on ``sorry'' placeholders. However, proof-level autoformalization is equally essential for faithfully verifying the correctness of natural-language mathematical reasoning and remains an underexplored yet critical direction.

Given a natural-language theorem and its accompanying proof, \textit{full-theorem autoformalization} aims to produce a complete formalization consisting of both a formal theorem statement and a machine-verified proof. This setting differs fundamentally from a pipeline that first autoformalizes the statement and then applies automated theorem proving (ATP). Such a pipeline treats the generated formal statement as fixed before proof generation begins. Consequently, if the generated statement is not syntactically well-formed in the target formal system, the ATP stage may not even be applicable. Even if the statement is formally well-formed, semantic errors introduced during statement formalization cannot be corrected or reconsidered in light of the proof. In contrast, full-theorem autoformalization jointly constructs the formal statement and proof from the natural-language theorem and its accompanying proof, allowing the two components to inform and constrain one another.

Full-theorem autoformalization inherently involves multiple dimensions of quality, motivating its formulation as a multi-dimensional optimization problem. \textit{Formal Validity} (FV), which requires the generated formal code to be well-formed and the proof to be mechanically verified by the theorem prover, is a hard requirement for any successful formalization. However, machine-verified validity alone does not guarantee that the resulting formalization faithfully captures the intended theorem. For example, \citet{openai2026navierstokes} presented a proposed solution to the Navier--Stokes Millennium Prize Problem, supported by a Lean4 formalization, but \citet{bastounis2026navierstokeslosttranslationlean} identified discrepancies between the informal mathematical arguments and their corresponding Lean formalizations, raising concerns about whether machine-checked proofs necessarily establish the correctness of the original mathematical claims. Moreover, these formalizations can span thousands or even tens of thousands of lines and contain substantial redundancy, making human inspection and error localization difficult. These observations suggest that full-theorem autoformalization systems should also consider softer dimensions~\citep{zhang2025goldstandardsepistemicensemble}, such as \textit{Logical Preservation} (LP) and \textit{Mathematical Consistency} (MC), which assess semantic alignment with the original informal theorem and proof, as well as \textit{Formal Quality} (FQ), which captures properties such as clarity, structure, concision, and reusability of the resulting formalization.

Given these multiple quality requirements, relying on a single LLM at test time, whether through vanilla prompting~\citep{wu2022autoformalization} or refinement-based methods, has limited capacity for simultaneously improving multiple dimensions~\citep{zhang-etal-2024-consistent}. Moreover, different LLMs used for refinement often exhibit distinct strengths across quality dimensions. This observation motivates the following research question: \textit{How can we leverage the complementary strengths of different LLMs to jointly improve multiple dimensions of full-theorem autoformalization at test time?}

\begin{figure}[!t]
    \centering
    \includegraphics[width=\textwidth]{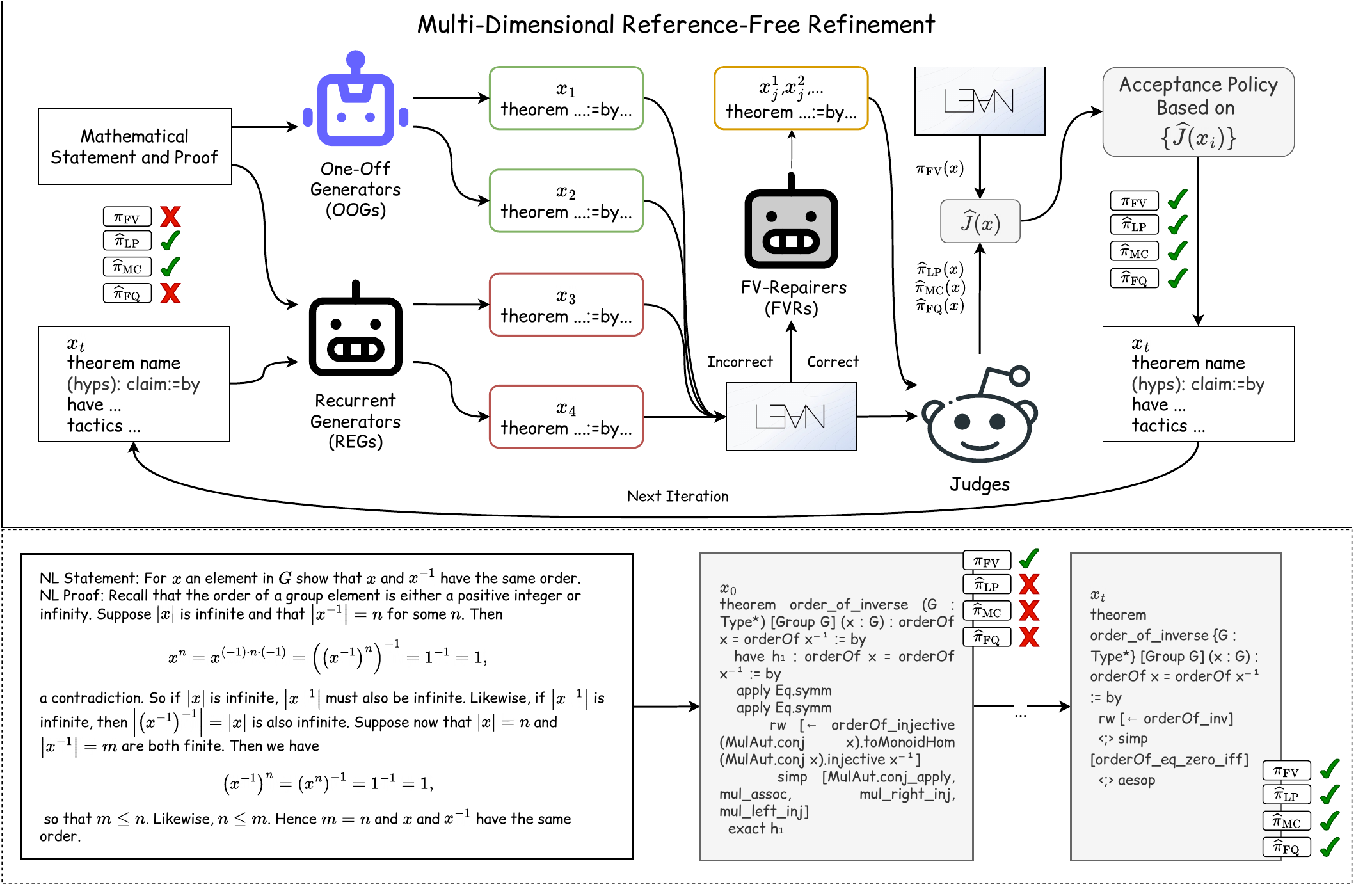}
    \caption{Schematic illustration of the multi-dimensional reference-free refinement process with an example. In this process, One-Off Generators (OOGs) produce formalizations from scratch, Recurrent Generators (REGs) refine the current best formalization using feedback from LLM judges, and FV-Repairers (FVRs) correct formally invalid candidates generated by OOGs and REGs. The acceptance policy retains only the formalization estimated to have the highest quality among all newly generated candidates and the previous best.}
    \label{fig:framework}
\end{figure}

To address this research question, we formulate a masked composite objective integrating hard and soft quality dimensions, and develop a multi-dimensional reference-free iterative refinement process at test time (Figure~\ref{fig:framework}). The process employs an acceptance policy that monotonically improves the estimated objective, with conditional theoretical guarantees characterizing its behavior under judge uncertainty. Our main contributions are summarized as follows:

\begin{enumerate}
    \item We formulate full-theorem autoformalization as a joint statement–proof generation task and introduce a masked composite objective that integrates theorem-prover-verified Formal Validity (FV) as a hard requirement with multiple soft quality dimensions. This formulation provides a unified criterion for evaluating and selecting candidate formalizations during iterative refinement.
    \item We decompose the roles of LLMs in full-theorem autoformalization into three categories: One-Off Generators (OOG), FV-Repairers (FVR), and Recurrent Generators (REG). The FVRs and REGs can leverage feedback to refine formalizations along different quality dimensions. We further introduce a \textit{responsiveness map} that models the expected, dimension-specific improvements induced by each generator, thereby enabling sample-efficient generator selection.
    \item We propose a multi-dimensional reference-free iterative optimization process that combines theorem-prover verification and LLM-based quality judgments through an acceptance policy, without requiring ground-truth formalizations or human involvement in candidate selection. The policy guarantees non-decreasing estimated objective scores, and our theoretical analysis establishes a conditional progress criterion: at each iteration, either the true objective does not decrease or its lower confidence bound increases.
    \item We empirically demonstrate that specialized mathematical LLMs and general-purpose LLMs exhibit complementary strengths across different generator roles, and that our iterative framework leverages these strengths to improve aggregate formalization quality across refinement iterations. For full-theorem autoformalization on the miniF2F and ProofNet benchmarks, our process achieves overall scores of 90.27\% and 52.45\%, based on combined judgments from the theorem prover and GPT-4.1-mini judges, with 49.18\% and 17.74\% of instances having maximum scores across all evaluation dimensions, respectively.
\end{enumerate}

\section{Multi-Dimensional Reference-Free Refinement}\label{sec:mono}
\subsection{Problem Statement}
The task of \textit{full-theorem autoformalization} can be formally defined as follows. Given a mathematical statement $s_\nl$ expressed in natural language and its corresponding proof $p_\nl$, the goal is to automatically construct a formalization $x=(s,p)$, where $s$ is a formal theorem statement and $p$ is a formal proof of $s$. The formal statement $s$ should be semantically aligned with $s_\nl$, while the proof $p$ should preserve the essential reasoning of $p_\nl$ without necessarily following exactly the same sequence of reasoning steps. The quality of a formalization $x=(s,p)$ can be characterized along multiple dimensions:

\textbf{Hard Dimension -- Formal Validity (FV).} A full-theorem formalization should be \textit{formally valid} such that the formal statement $s$ must be syntactically well-formed in the target formal system, and the formal proof $p$ must be accepted by the corresponding theorem prover as a valid proof of $s$. Formal validity can therefore be rigorously and automatically verified by the theorem prover, and is also commonly measured by the pass rate.

\textbf{Soft Dimensions} $\mathcal{D}_\soft$. Soft dimensions capture properties of a formalization that cannot be determined solely by verification from the theorem prover. Examples include: \textbf{Logical Preservation (LP):} the formal theorem $s$ should preserve the \textit{logical content and inferential structure} of $s_\nl$; \textbf{Mathematical Consistency (MC):} the formal theorem $s$ should faithfully represent the \textit{mathematical objects, relations, and operations} expressed in $s_\nl$; and \textbf{Formal Quality (FQ):} the full-theorem formalization $x$ should exhibit high intrinsic quality in terms of structural clarity and conciseness, avoiding unnecessary statements and redundant reasoning steps.

\paragraph{Dimension Scores.} Let $\mathcal{X}$ denote the space of all possible candidate formalizations. We score each candidate $x\in\mathcal{X}$ using the vector
\begin{equation}
\pi(x)=\Big(\pi_\FV(x),\big(\pi_i(x)\big)_{i\in\mathcal{D}_\soft}\Big),
\end{equation}
where $\pi_\FV(x)\in\{0,1\}$ indicates whether $x$ passes verification by the theorem prover, and $\pi_i(x)\in[0,1]$ measures the quality of $x$ along soft dimension $i$. The goal of full-theorem autoformalization is to identify an ideal formalization $x^\star\in\mathcal{X}$ that achieves the optimal score across all dimensions, i.e., $\pi(x^\star)=\mathbf{1}$.

\paragraph{Masked Composite Objective.} For the optimization process, we define an objective function:
\begin{equation}
    J_\OA(x)=\pi_\FV(x)\sum_{i\in\mathcal{D}_\soft} w_i\pi_i(x),
\label{eq:J}
\end{equation}
where $w_i>0$ and $\sum_i w_i=1$. The multiplicative factor $\pi_\FV(x)$ is a \emph{mask} since it enforces FV as a hard constraint while retaining a \emph{single scalar potential} for the soft dimensions. This objective function precludes selecting invalid-but-high soft scores and matches downstream utility where FV is binary and primary. $x^\star\in\mathcal{X}$ is an ideal formalization if and only if $J_\OA(x^\star)=1$.

\subsection{Process Formulation}
We propose an iterative refinement process $\{x_t\}$ that gradually generates candidate formalizations and selects the one most likely to achieve the highest quality.

\subsubsection{Candidate Generation}
Let $\mathcal{C}_t\subseteq\mathcal{X}$ denote the set of candidate formalizations obtained at step $t$. Maintaining diversity within $\mathcal{C}_t$ is crucial for covering the space of possible formalizations as comprehensively as possible. We define $\mathcal{G}=\{G_1,\dots,G_M\}$ as the set of generators to sample candidate formalizations from $\mathcal{X}$. We classify generators into three categories:

\paragraph{One-Off Generator (OOG).} The one-off generator $G$ assigns a fixed distribution over samples conditioned on $(s_\nl,p_\nl)$, generating outputs according to $x\sim p_G(x\mid s_\nl,p_\nl)$.

\paragraph{FV-Repairer (FVR).} Given that the objective places the highest priority on formal validity, the FV-Repairer $G$ is designed to correct instances in which $\pi_{\FV}(x)=0$, such that $\Pr\Big(\pi_{\FV}\big(G(x)\big)=1\Big)>0$.

\paragraph{Recurrent Generator (REG).} The recurrent generator $G$ assigns a fixed distribution over samples conditioned on $(s_\nl,p_\nl)$ and a previous candidate formalization $x_{t-1}$, generating outputs according to $x_t\sim p_G(x\mid s_\nl,p_\nl,x_{t-1})$.

In practice, the one-off generator can rely on few-shot prompting~\citep{wu2022autoformalization} or retrieval-augmented generation~\citep{zhang-etal-2025-autoformalization, liu2025rethinking}. The FV-repairer may be implemented using either LLM-based approaches~\citep{zhang-etal-2024-consistent} or rule-based systems~\citep{zhang-etal-2025-autoformalization}. The recurrent generator can be instantiated via refinement-based methods~\citep{zhang-etal-2025-masa}. When using prompt-based refinement methods either in FVR settings with feedback from theorem prover, or REG settings with feedback from LLM judges, we propose a method to model each generator’s dimension-wise response to the feedback.

\begin{definition}[Responsiveness Map]
For generator $G_m$ which leverages LLM with feedback $\mathrm{FB}$, we define the \textit{local responsiveness} for dimension $i\in\{\FV\}\cup\mathcal{D}_\soft$ at state $x$ as
\begin{equation}
    \rho_i(G_m\mid x)=\mathbb{E}_{x'\sim \mathrm{LLM}(x;\mathrm{FB}(x))}\big[\pi_i(x')-\pi_i(x)\big].
\end{equation}
\end{definition}

The responsiveness map captures the expected local gain for each dimension. For the FV dimension, this quantity is equivalent to the probability that the generator restores formal validity when configured as an FV-Repairer. In practice, $\pi_\FV(x)$ can be determined exactly by a theorem prover, whereas the true scores along the soft dimensions are generally intractable to compute and must instead be estimated. To this end, we introduce a fused estimator that aggregates judgments from multiple LLM judges.

\begin{definition}[Judge-Based Quality Estimation]
For each soft dimension $i\in\mathcal{D}_\soft$, we define the fused estimator as
\begin{equation}\label{eq:judge}
    \widehat{\pi}_i(x)=\sum_{k=1}^{K_i} w_{i,k}\widehat{\pi}^{(k)}_i(x),
\end{equation}
where $\widehat{\pi}^{(k)}_i(x)\in\{0,1\}$ denotes the binary judgment produced by the $k$-th LLM judge for dimension $i$, and the aggregation weights satisfy $w_{i,k}>0$ and $\sum_{k=1}^{K_i} w_{i,k}=1$.
\end{definition}

When $G_m$ proposes $B_m$ candidates $x_m^{(b)}$ for $x$, the estimation of local responsiveness of generator $G_m$ for $i\in\mathcal{D}_\soft$ could be achieved by aggregated differences: $\widehat{\rho}_i(G_m\mid x)=\frac{1}{B_m}\sum_{b=1}^{B_m}\big(\widehat{\pi}_i(x_m^{(b)})-\widehat{\pi}_i(x)\big)$ with some variance introduced by LLM judges, and for $\FV$, $\widehat{\rho}_\FV(G_m\mid x)=\frac{1}{B_m}\sum_{b=1}^{B_m}\big(\pi_\FV(x_m^{(b)})-\pi_\FV(x_t)\big)$. For sample efficiency, the proposal budgets should be set proportionally to generator $G_m$ with better responsiveness maps.

An example of sampling candidate formalizations from different generators is provided in Algorithm~\ref{alg:sample} in Appendix~\ref{app:alg}. We also provide a concrete process of how to determine the generator roles for a set of candidate LLMs in Appendix~\ref{app:role}.

\subsubsection{The Acceptance Policy}
We then define a plug-in estimator which indicates the overall performance as:
\begin{equation}\label{eq:est}
    \widehat{J}(x)=\max\big\{ \pi_\FV(x),\epsilon\big\}\sum_{i\in\mathcal{D}_\soft} w_i\widehat{\pi}_i(x),
\end{equation}
where $\epsilon>0$ is a sufficiently small constant satisfying $\epsilon<\min\{w_iw_{i,k}\}$. We propose the following acceptance policy:

\paragraph{Conservative Acceptance Policy.} Assume $\widehat{J}(x_0)<1$. We define the best candidate as:
\begin{equation}
    x_{t+1}^\star=\arg\max_{x\in\mathcal{C}_t}\widehat{J}(x).
\end{equation}
The acceptance rule at step $t+1$ is designed as:
\begin{equation}\label{eq:acc}
    x_{t+1}=
    \begin{cases}
    x_{t+1}^\star, & \text{if}\quad \widehat{J}(x_{t+1}^\star)>\widehat{J}(x_t)\\
    x_t,  & \text{else}
    \end{cases}.
\end{equation}

The process applying this acceptance rule is illustrated in Algorithm~\ref{alg:mono} in Appendix~\ref{app:alg}.

\subsection{Theoretical Properties of the Refinement Process}

The automatic and rigorous evaluation of soft dimensions such as LP, MC, FQ is normally intractable. Previous work~\citep{zhang2025goldstandardsepistemicensemble} indicates the ability of using LLM judges for the evaluation of certain aspects of formalizations, however, LLM judges can be ambiguous and should be accepted under uncertainty. We make the following assumptions for LLM judges:

\begin{assumption}\label{asm:judges}
We assume that the fused estimator in Eq.~\ref{eq:judge} satisfies the following:

(i) \textbf{Asymptotic stability}: After asymptotic calibration of the weights and LLM-based judgments, for example by using a sufficiently large number of samples from each LLM judge to obtain consistent evaluations, each $\widehat{\pi}_i(x)$ converges almost surely to a deterministic value for every $x$.

(ii) \textbf{Measurable uncertainty}: There exists a \textit{one-sided lower confidence bound} for the true $\pi_i(x)$ at uncertainty level $\delta_i\in(0,1)$, given by:
\begin{equation}
    \Pr\big(\widehat{\pi}_i(x)-m_i(\delta_i)\le\pi_i(x)\big)\ge1-\delta_i,
    \label{eq:lcb_i}
\end{equation}
where the margin term $m_i(\delta_i)>0$ absorbs calibration error, inter-judge correlation, and distributional drift.
\end{assumption}

Assumption~\ref{asm:judges} formalizes the reliability of LLM judges. Specifically, (i) asserts that their judgments are stable across comparisons, and (ii) limits the likelihood that low-quality formalizations are mistakenly evaluated favorably along soft dimensions, with such errors occurring with probability at most $\delta_i$. We provide a naive uncertainty estimation for LLM judges using Eq.~\ref{eq:lcb_i} in Appendix~\ref{app:unc}.

\begin{theorem}[Lower Confidence Bound of $J_\OA$]\label{thm:lcb}
Under Assumption~\ref{asm:judges}, there exists a lower confidence bound for the masked composite objective $J_\OA(x)$ defined in Eq.~\ref{eq:J} given by:
\begin{equation}
    \LCB_\delta(x)=\max\Big\{0,\pi_\FV(x)\sum_{i\in\mathcal{D}_\soft} w_i\big(\widehat{\pi}_i(x)-m_i(\delta_i)\big)\Big\}
    \label{eq:lcb}
\end{equation}
with uncertainty level $\delta=\sum_{i\in\mathcal{D}_\soft}\delta_i$.
\end{theorem}
\begin{proof}
    See Appendix~\ref{app:lcb}.
\end{proof}

This bound can be used to characterize the theoretical behavior of the refinement process. The conservative acceptance policy guarantees that the estimated composite objective defined in Eq.~\ref{eq:est} is non-decreasing across iterations. The following theorem further characterizes the process under the proposed uncertainty model.

\begin{theorem}[Monotonicity of the Lower Confidence Bound]\label{thm:mono}
Suppose that $\epsilon\leq\sum_i w_i m_i(\delta_i)$. Then, for all $t$, we have $\LCB_\delta(x_{t+1})\geq\LCB_\delta(x_t)$.
\end{theorem}
\begin{proof}
    See Appendix~\ref{app:mono}.
\end{proof}
Theorem~\ref{thm:mono} establishes that the non-negative lower confidence bound of the true objective is non-decreasing throughout the refinement process. This property ensures that the proposed acceptance policy preserves or improves the estimated lower bound on formalization quality under judge uncertainty.

\begin{theorem}\label{thm:con}
Suppose that $\epsilon\leq\sum_i w_i m_i(\delta_i)$.
The process $\{x_t\}$ satisfies the following properties:

(i) \textbf{Convergence}:
The sequence $\{\LCB_\delta(x_t)\}_{t\geq 0}$ converges.

(ii) \textbf{Process Terminability}:
If there exists $x\in\bigcup_t\mathcal{C}_t$ such that $\widehat{J}(x)=1$, then there exists a candidate $x^*\in\bigcup_t\mathcal{C}_t$ such that $\exists t_0,\forall t\geq t_0,x_t=x^*$, and $\Pr\Big(1-\sum_{i\in\mathcal{D}_\soft} w_i m_i(\delta_i)\le J_\OA(x^*)\Big)\ge1-\delta$.
\end{theorem}
\begin{proof}
    See Appendix~\ref{app:con}.
\end{proof}

Theorem~\ref{thm:con} establishes that the lower confidence bound converges and therefore tends to stabilize as the refinement process progresses. Moreover, when the process attains the maximum estimated objective, the selected candidate achieves a true objective value of at least $1-\sum_{i\in\mathcal{D}_\soft}w_i m_i(\delta_i)$ with probability at least $1-\delta$. This result suggests that perfect judges are not necessary for selecting high-quality formalizations, provided that their uncertainty is sufficiently controlled.

\section{Empirical Experiments}\label{sec:exp}
\subsection{Setup}\label{sec:setup}
Lean4~\citep{lean} (version v4.19.0-rc3 with REPL) is chosen as the formal language for all experiments. We evaluate our approach on miniF2F-Test~\citep{zheng2022miniff,jiang2023draft} and ProofNet-Test~\citep{azerbayev2023proofnetautoformalizingformallyproving}, as both benchmarks provide natural-language proofs. DeepSeek-Prover-V2-7B~\citep{ren2025deepseekproverv2advancingformalmathematical}, Goedel-Prover-V2-8B~\citep{lin2025goedelproverv2scalingformaltheorem}, Seed-Coder-8B~\citep{seed2025seedcoderletcodemodel}, Qwen2.5-7B~\citep{qwen2025qwen25technicalreport}, and GPT series~\citep{openai2024gpt4,openai2025gpt5} are selected as candidate LLMs for our experiments. The description and rationale behind the selection of benchmarks and models are provided in Appendix~\ref{app:rat}. 

We select Logical Preservation (LP), Mathematical Consistency (MC), and Formal Quality (FQ) as the target soft dimensions for full-theorem autoformalization in our experiments. For the fused estimators in Eq.~\ref{eq:judge}, we adopt the OAP setting from~\citep{zhang2025goldstandardsepistemicensemble}, as it naturally fits the formulation in Eq.~\ref{eq:judge} while providing better efficiency. For simplicity, we assign equal weights, setting $w_i=\frac{1}{3}$ for each soft dimension. These weights define a practical scalarization rather than a canonical measure of formalization quality. The framework allows alternative weighting schemes or additional constraints when different quality requirements are prioritized. We set $\epsilon=0.001$ in Eq.~\ref{eq:est}.

On the evaluation front, we use the theorem prover-derived score $\pi_\FV$ to assess formal validity. For the soft dimensions, we report $\widehat{\pi}_\LP$, $\widehat{\pi}_\MC$, and $\widehat{\pi}_\FQ$, as assessed by for-acceptance judges (GPT-4.1-mini), external LLM judges (Qwen2.5-Coder-7B), and human judges. This evaluation setting allows us to examine whether using the same LLM judges both during the formalization process and for final evaluation introduces evaluator bias or inflates the measured performance. We also report their aggregate score, $\widehat{J}$, as an overall evaluation metric. To provide a stricter measure of performance, we additionally report the proportion of perfect formalizations, defined as those satisfying $\widehat{J}=1$, i.e., receiving a score of 1 on every dimension.

We employ Algorithm~\ref{alg:sample} in Appendix~\ref{app:alg} to construct candidate formalizations at each step of the multi-dimensional reference-free refinement process. DeepSeek-Prover-7B and Goedel-Prover-8B are selected as the One-Off Generators and FV Repairers. For Recurrent Generators, we consider refinement using feedback on individual soft dimensions from LLM judges (REG-LP, REG-MC, and REG-FQ). Specifically, we use Seed-Coder-8B with REG-LP, Qwen2.5-7B with REG-LP, and DeepSeek-Prover-7B with REG-FQ in the refinement process. Further discussion of this setup is provided in Appendix~\ref{app:role}.

\begin{figure}[!t]
    \centering
    \begin{subfigure}{0.32\textwidth} 
      \centering
      \includegraphics[width=\textwidth]{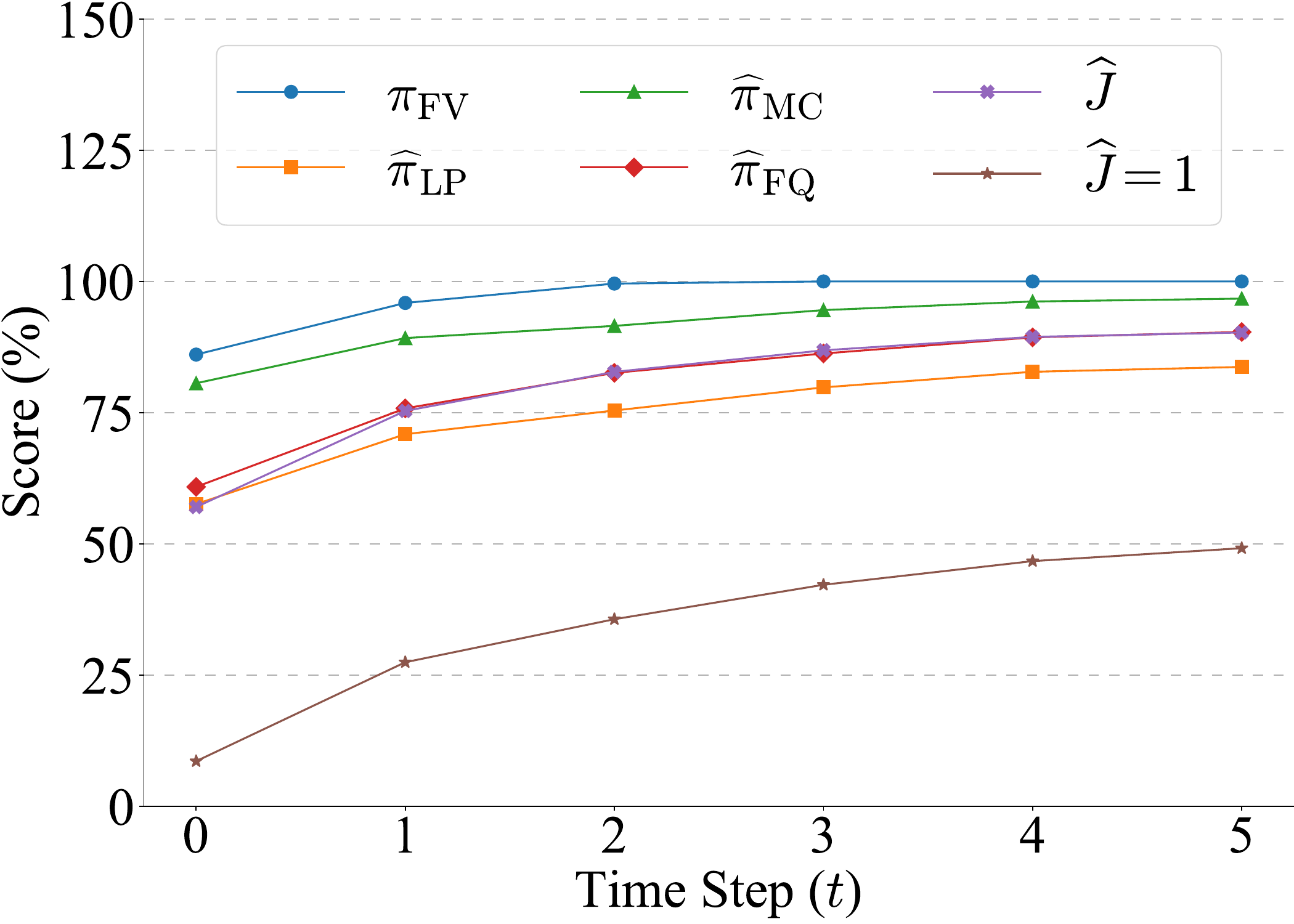}
      \caption{miniF2F (GPT-Eval)}
      \label{fig:mono_1}
    \end{subfigure}
    \begin{subfigure}{0.32\textwidth} 
      \centering
      \includegraphics[width=\textwidth]{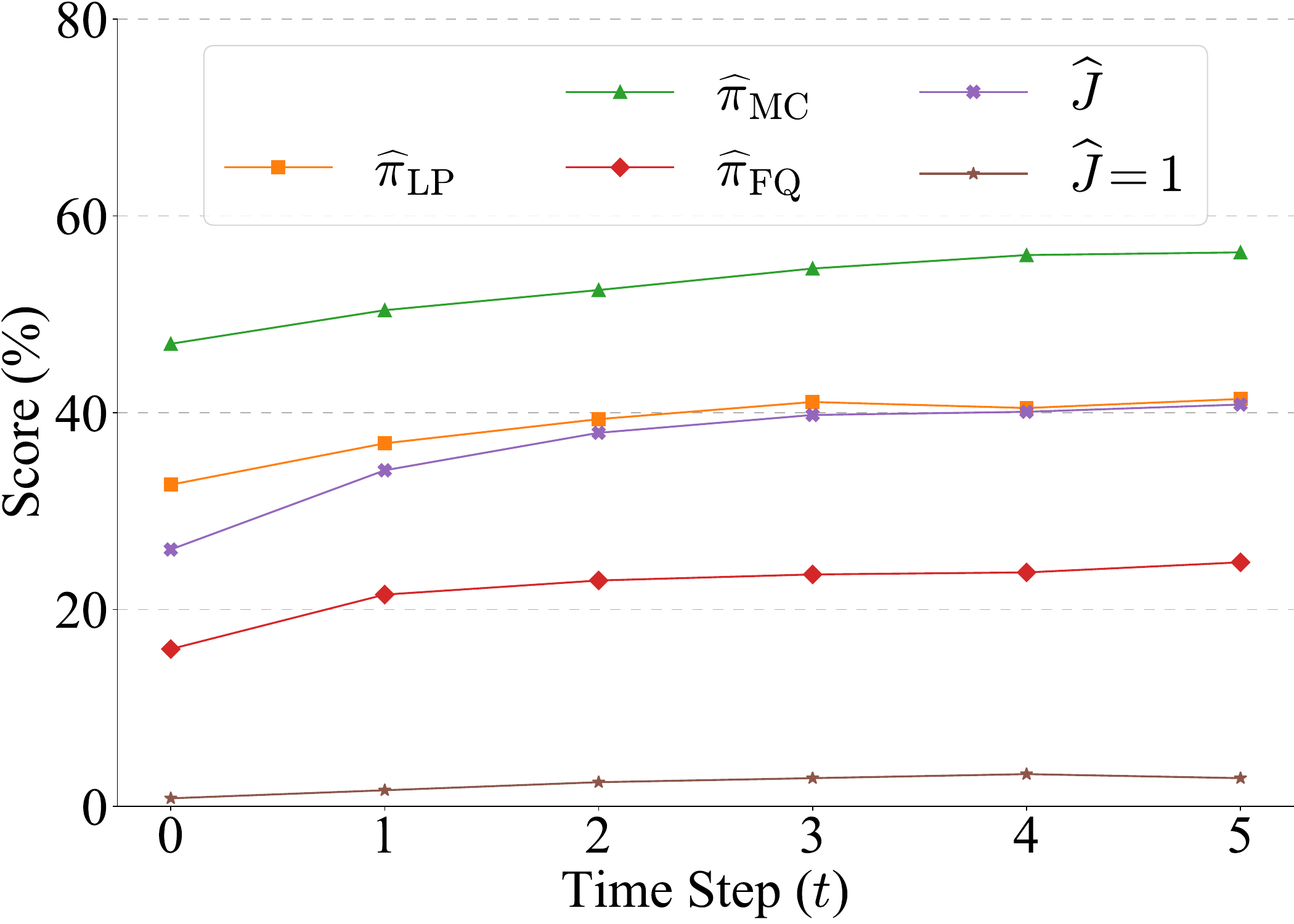}
      \caption{miniF2F (Qwen-Eval)}
      \label{fig:mono_1_qwen}
    \end{subfigure}
    \begin{subfigure}{0.32\textwidth} 
      \centering
      \includegraphics[width=\textwidth]{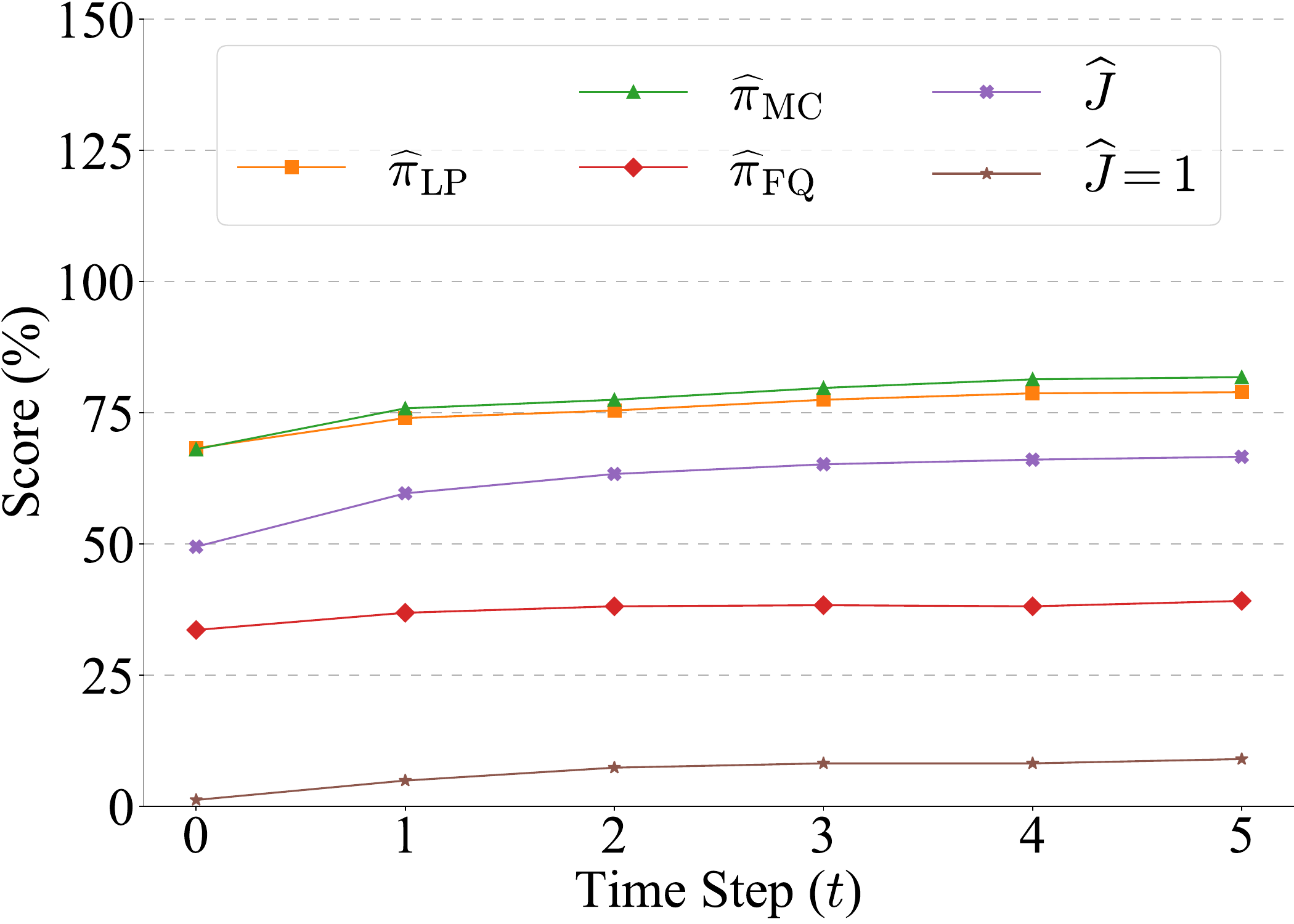}
      \caption{miniF2F (Human-Eval)}
      \label{fig:mono_1_human}
    \end{subfigure}
    \begin{subfigure}{0.32\textwidth} 
      \centering
      \includegraphics[width=\textwidth]{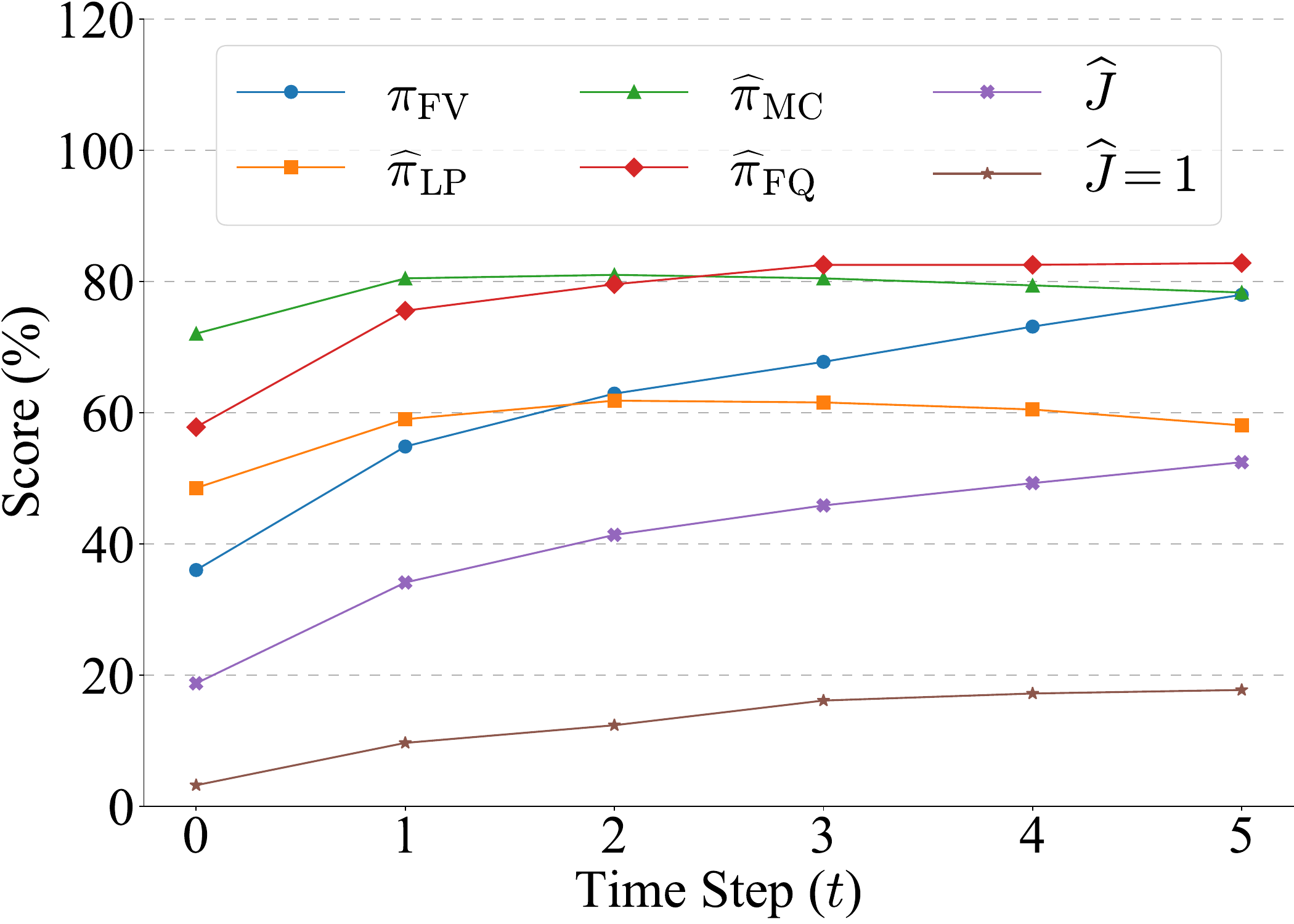}
      \caption{ProofNet (GPT-Eval)}
      \label{fig:mono_2}
    \end{subfigure}
    \begin{subfigure}{0.32\textwidth} 
      \centering
      \includegraphics[width=\textwidth]{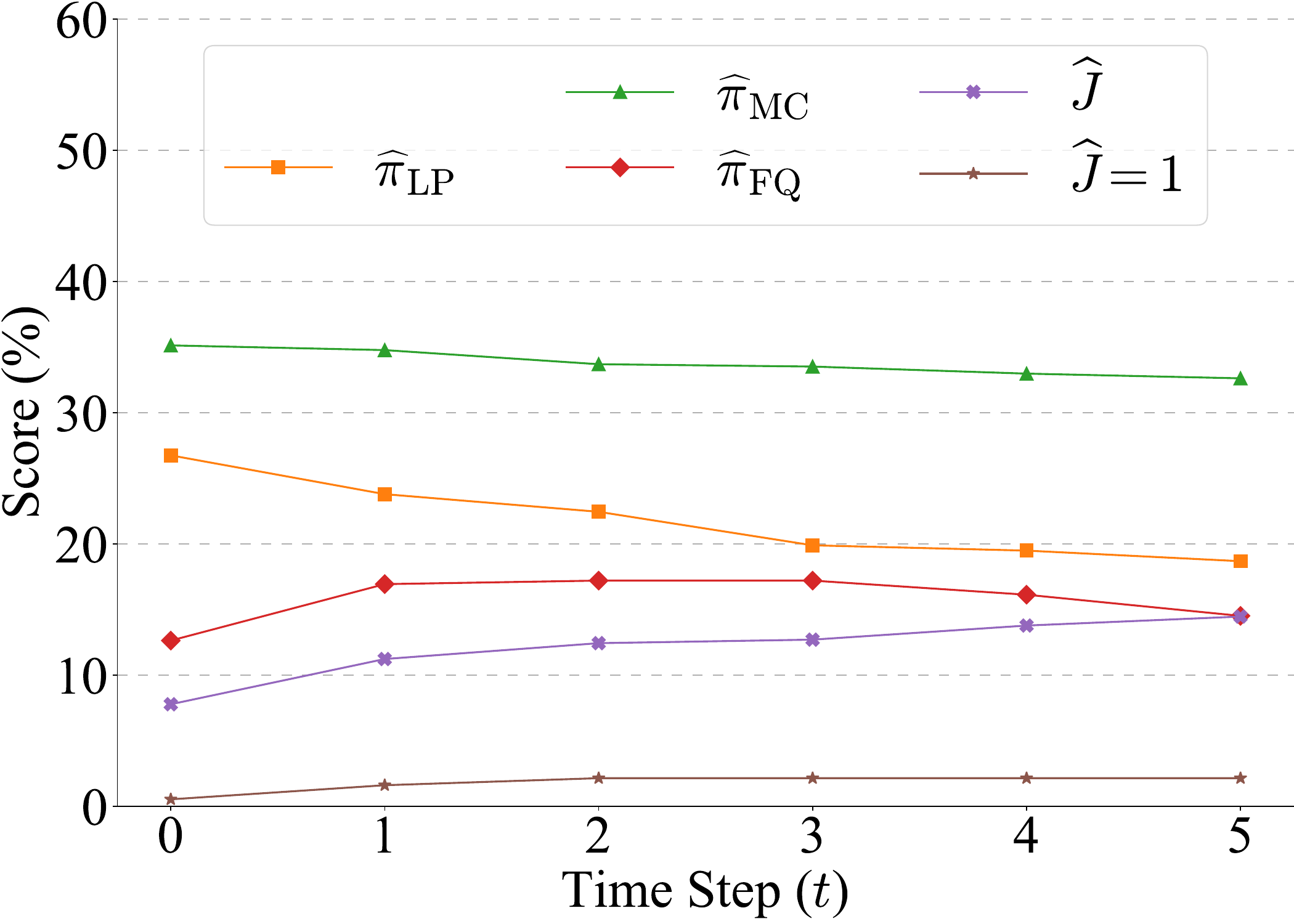}
      \caption{ProofNet (Qwen-Eval)}
      \label{fig:mono_2_qwen}
    \end{subfigure}
    \begin{subfigure}{0.32\textwidth} 
      \centering
      \includegraphics[width=\textwidth]{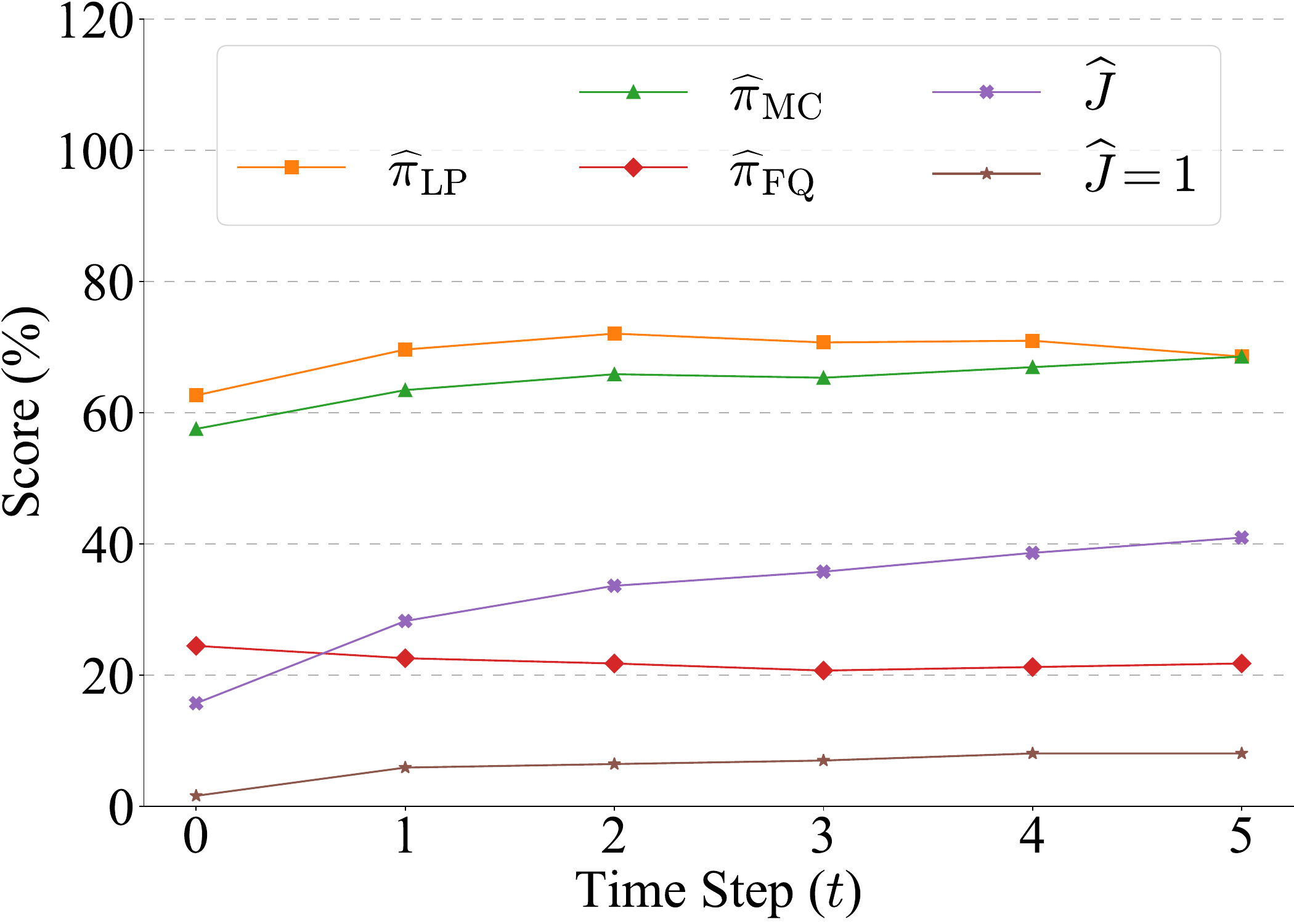}
      \caption{ProofNet (Human-Eval)}
      \label{fig:mono_2_human}
    \end{subfigure}
    \caption{Performance of the multi-dimensional reference-free refinement process using an ensemble of LLMs acting different generator roles. The evaluators of soft-dimensions includes: (\textbf{GPT-Eval}): GPT-4.1-mini judges; (\textbf{Qwen-Eval}): Qwen2.5-Coder-7B judges; (\textbf{Human-Eval}): one human expert.}
    \label{fig:mono}
\end{figure}

\subsection{Main Results}
We run the optimization process for six iterations on both miniF2F and ProofNet. Performance at each iteration is illustrated in Figure~\ref{fig:mono} with detailed numbers provided in Table~\ref{tab:mono} in Appendix~\ref{app:add}.

\textbf{The multi-dimensional reference-free refinement process achieves increasing aggregate scores under the acceptance judges, with diminishing improvements over successive iterations.} In Figure~\ref{fig:mono}, across both datasets, the overall evaluation scores consistently improve as the process proceeds. On miniF2F, the gains become relatively incremental after the fourth time step, indicating that the process tends to converge. At the final time step, the process achieves formal validity of 100.00\%, an overall score of 90.27\%, and 49.18\% perfect formalizations on miniF2F (Figure~\ref{fig:mono_1}). On the much harder benchmark ProofNet, the process leads to 77.96\% formal validity, 52.45\% overall score, and 17.74\% perfect formalizations (Figure~\ref{fig:mono_2}).

\textbf{Different evaluation dimensions exhibit distinct trends throughout the refinement process.} On ProofNet, while formal validity improves across all steps, some soft-dimensions tend to stabilize or fluctuate after the first three iterations on both datasets. However, this discrepancy still leads to a gradual increase in the overall objective, as validated by human evaluation, suggesting that the process produces a set of auto-formalizations with progressively higher quality. Such improvements indicate the potential of this approach to support mathematical discovery and to generate high-quality formalization data for further training of existing models.

\textbf{Performance improvement is not solely an artifact of using the same LLM judges for both acceptance and evaluation.} The process continues to exhibit a near non-decreasing trend and convergence when evaluated with LLM judges from a different model family (Figure~\ref{fig:mono_1_qwen}, \ref{fig:mono_2_qwen}).  While bias may arise from LLM-based evaluation, this observation suggests that the improvements are not purely driven by the preferences of a particular judge, but instead reflect a more consistent improvement behavior of the proposed refinement process. Human evaluation provides further evidence supporting the gradual improvement of the proposed process (Figure~\ref{fig:mono_1_human}, \ref{fig:mono_2_human}). However, individual quality dimensions do not necessarily improve monotonically under these independent evaluations.

\begin{figure}[!t]
    \centering
    \begin{subfigure}{0.24\textwidth} 
      \centering
      \includegraphics[width=\textwidth]{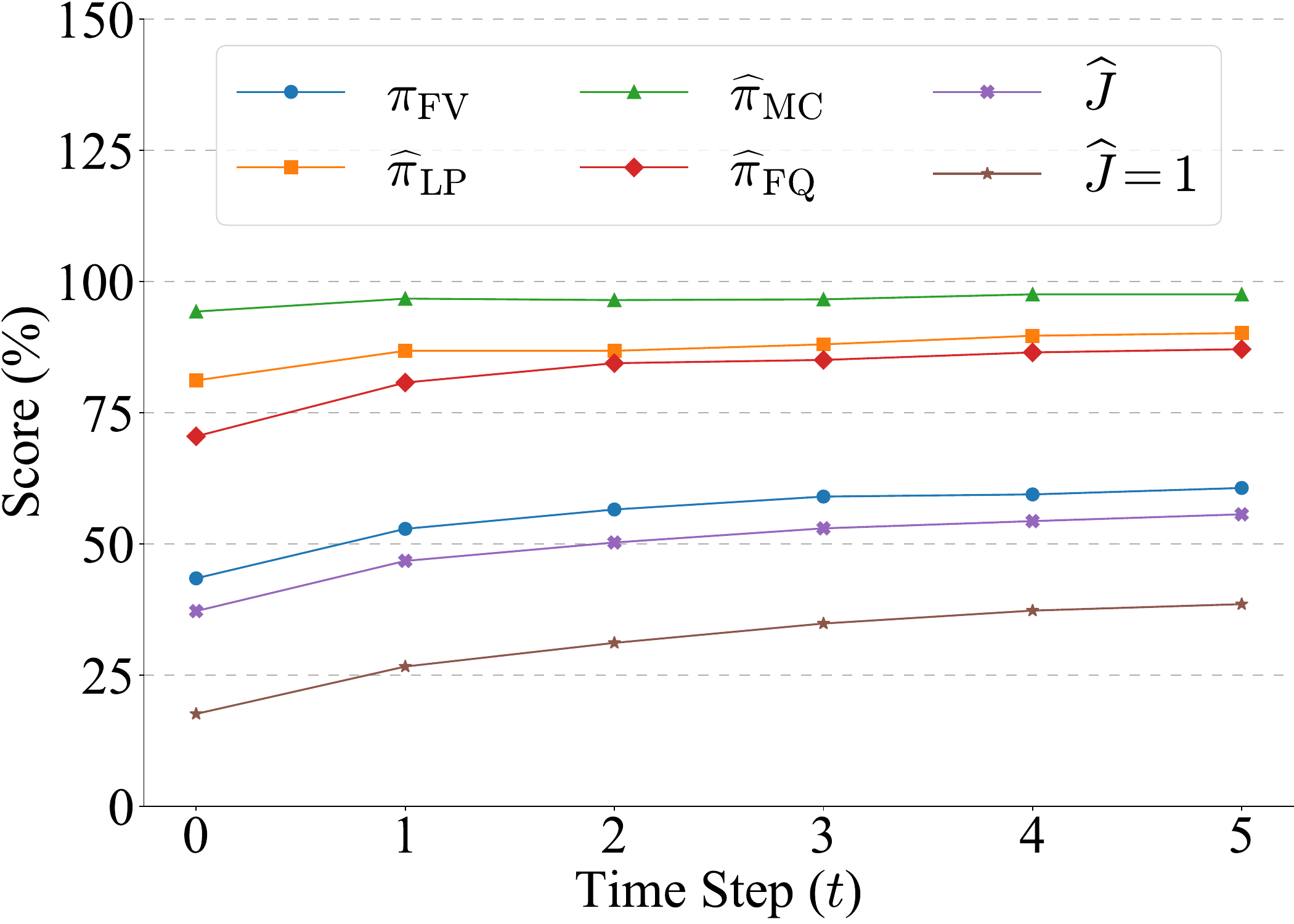}
      \caption{miniF2F (GPT-Eval)}
      \label{fig:mono_gpt_1}
    \end{subfigure}
    \begin{subfigure}{0.24\textwidth} 
      \centering
      \includegraphics[width=\textwidth]{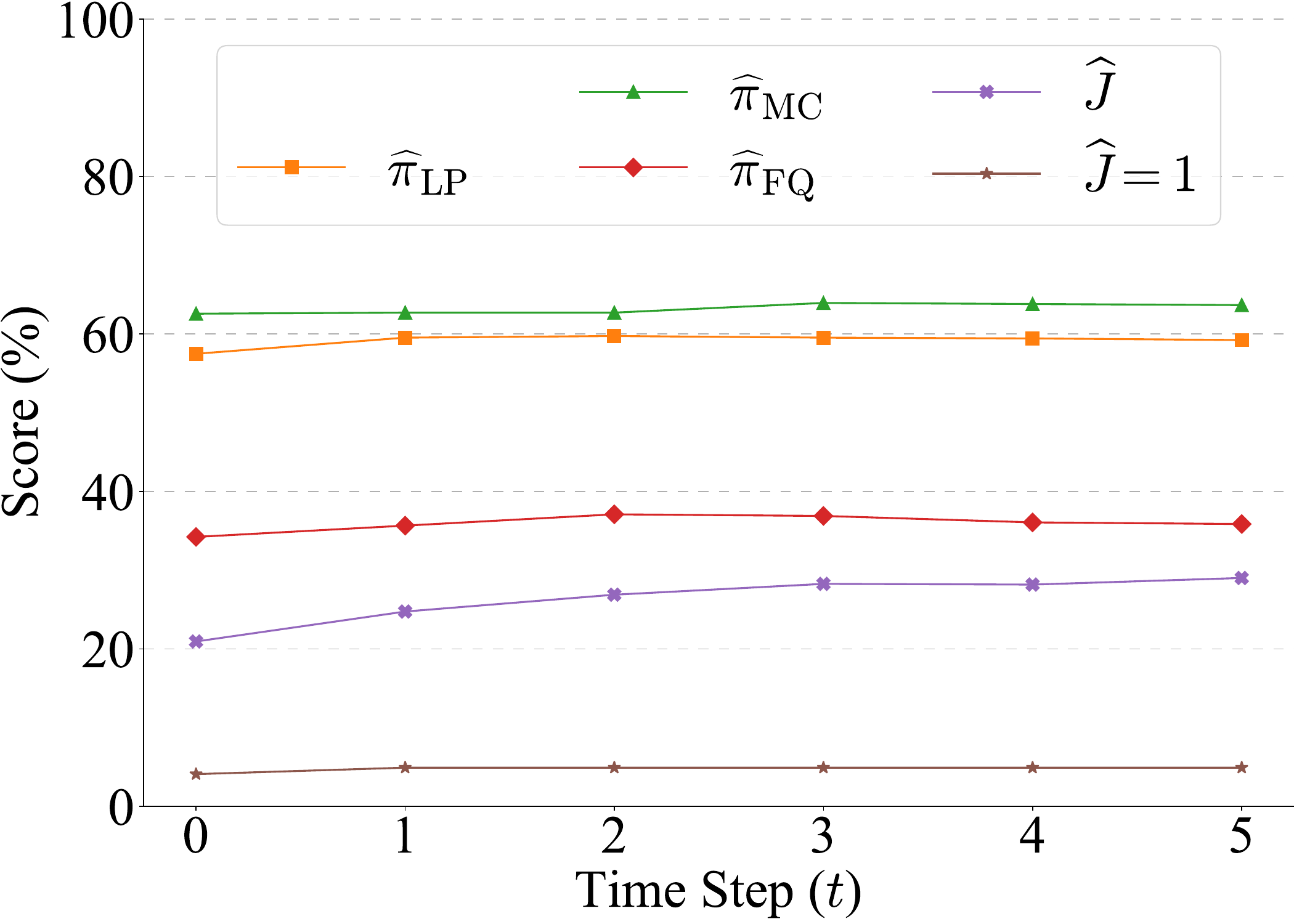}
      \caption{miniF2F (Qwen-Eval)}
      \label{fig:mono_gpt_1_qwen}
    \end{subfigure}
    \begin{subfigure}{0.24\textwidth} 
      \centering
      \includegraphics[width=\textwidth]{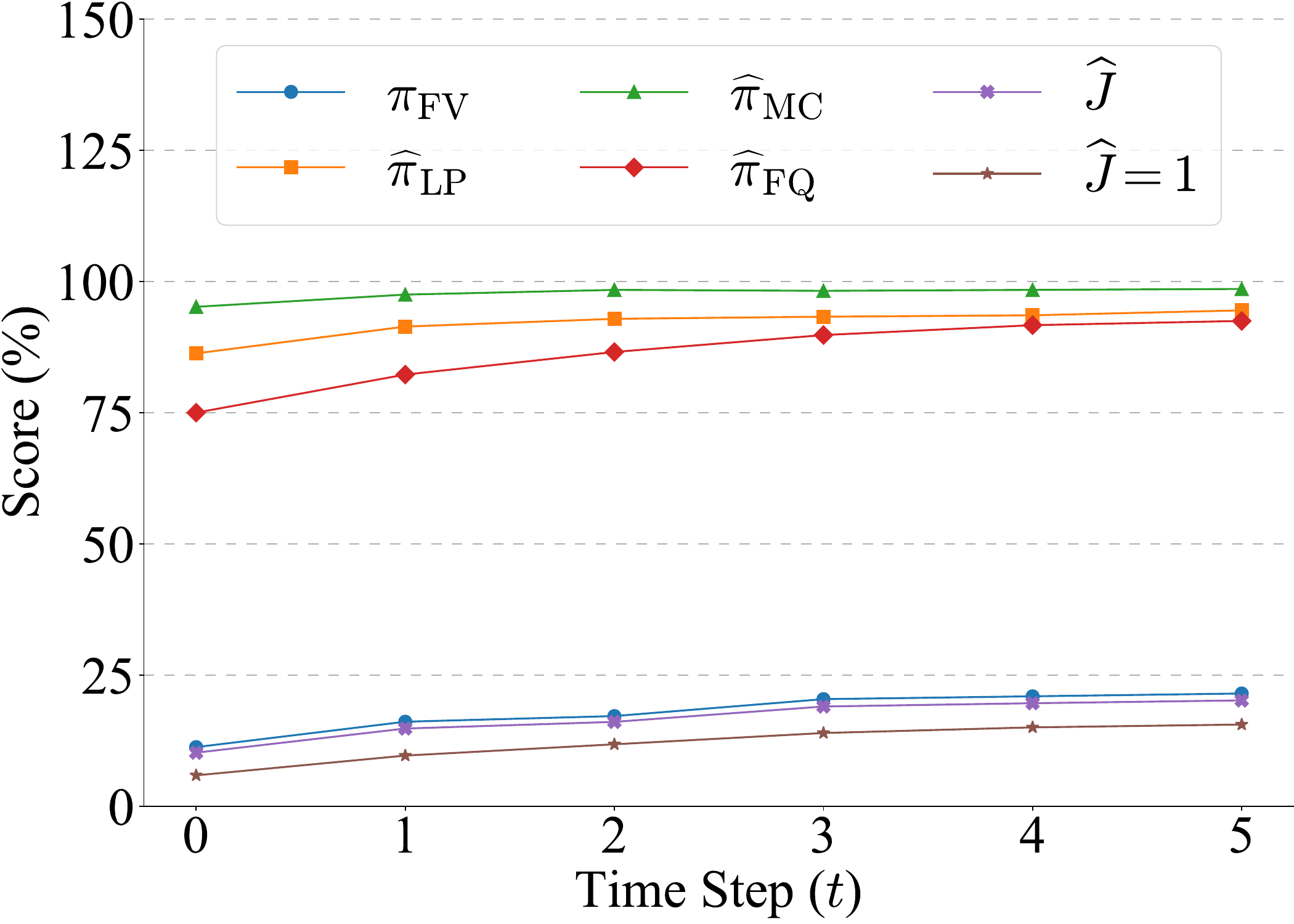}
      \caption{ProofNet (GPT-Eval)}
      \label{fig:mono_gpt_2}
    \end{subfigure}
    \begin{subfigure}{0.24\textwidth} 
      \centering
      \includegraphics[width=\textwidth]{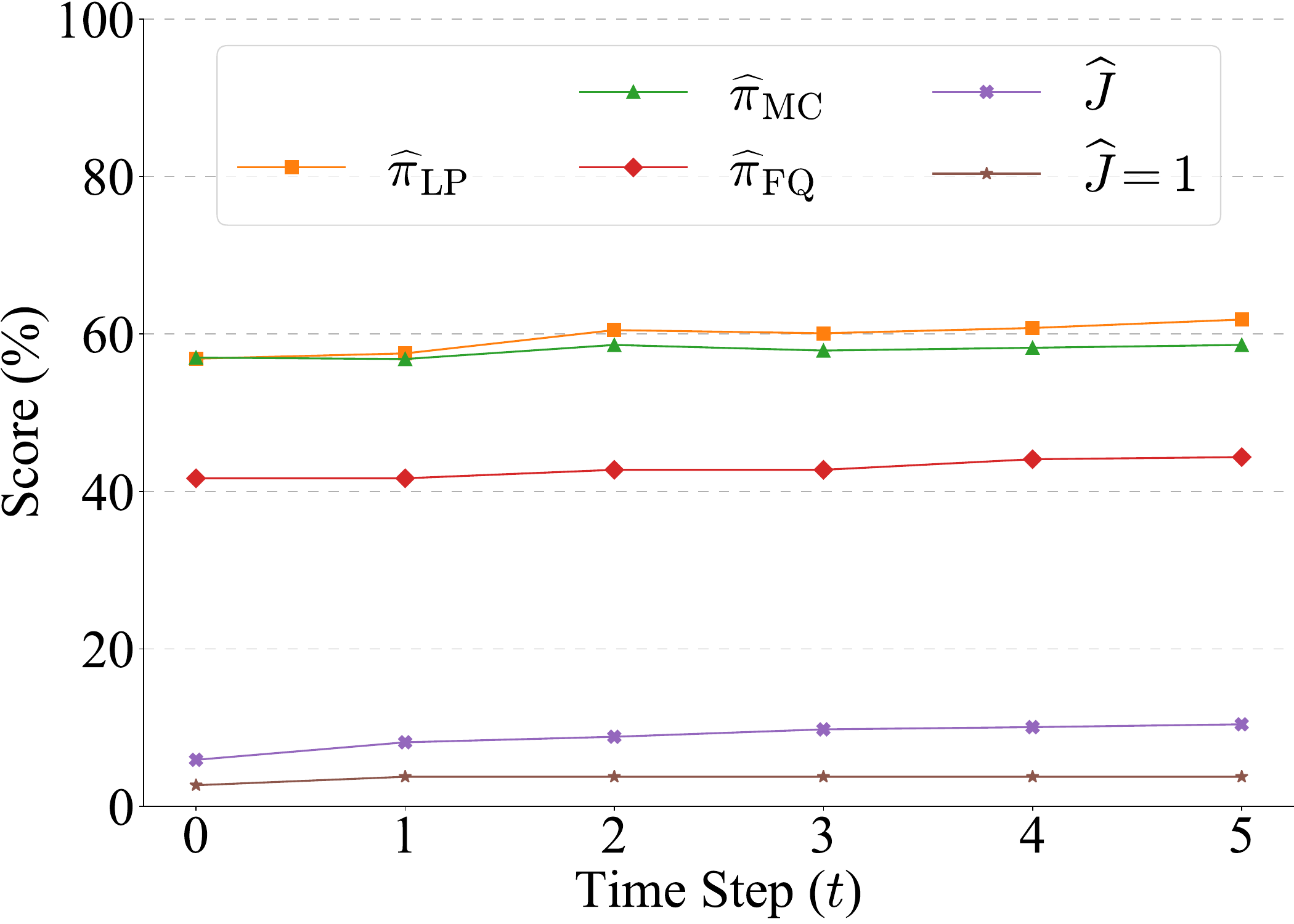}
      \caption{ProofNet (Qwen-Eval)}
      \label{fig:mono_gpt_2_qwen}
    \end{subfigure}
    \caption{Performance of the proposed refinement process with only GPT-5.4 as OOG and FVR.}
    \label{fig:mono_gpt}
\end{figure}

\subsection{Ablation Studies}
We first use a single strong LLM (GPT-5.4) as both the OOG and the FVR in the refinement process. The results are reported in Figure~\ref{fig:mono_gpt}. \textbf{The gains from the proposed refinement process do not inherently depend on an ensemble of multiple LLMs, however, an ensemble of smaller LLMs can outperform a single larger model.} As shown in Figures~\ref{fig:mono_gpt_1} and~\ref{fig:mono_gpt_2}, applying the proposed refinement process with GPT-5.4 yields substantial improvements in formal validity and aggregate formalization quality, with broadly consistent trends across evaluation settings. These improvements are further corroborated by external judges (Figures~\ref{fig:mono_gpt_1_qwen} and~\ref{fig:mono_gpt_2_qwen}). Collectively, these results show that the proposed process can induce consistent iterative improvements even when instantiated with a single LLM, supporting the generality of the underlying optimization framework. However, both formal validity and overall performance are notably lower than when multiple LLMs are assigned to distinct roles. This result further supports our design choice of role specialization within the proposed refinement process.

We further compare against iterative self-refinement (ISR) using three LLMs independently (DeepSeek-Prover-7B, Goedel-Prover-8B, and GPT-5.4), where each model receives error feedback from the theorem prover but no acceptance policy is applied. The results are reported in Table~\ref{tab:isr} in Appendix~\ref{app:add}. \textbf{Performance improvement does not emerge simply from repeatedly refining a formalization with theorem-prover feedback.} Under ISR with DeepSeek-Prover-7B and Goedel-Prover-8B, the scores fluctuate across iterations, with no clear trend of consistent improvement. Although ISR with GPT-5.4 exhibits some degree of monotonicity in formal validity and overall performance, the improvement is substantially weaker than that achieved by the multi-dimensional reference free refinement process with GPT-5.4.

\subsection{Comparison with Other Methods}

We present a comparative evaluation of different full-theorem generation methods, using a policy that selects, for each sample, the formalization with the highest $\widehat{J}$. We consider two generation baselines with DeepSeek-Prover-7B, Goedel-Prover-8B, and GPT-5.4: iterative self-refinement (ISR) and sampling $k$ independent outputs from a single model (analogous to Pass@$k$). The results are reported in Table~\ref{tab:com}.

\begin{table}[!t]
  \caption{Performance of different methods. Best-of-ISR selects the best formalization based on $\widehat{J}$ across 6 iterations for each sample. Best-of-$k$ ($k=8, 64$) selects the best formalization based on $\widehat{J}$ from $k$ outputs of a single model. All numbers are reported as percentages. \textbf{Bold} and \underline{Underline} mark the highest and second-highest scores.}
  \scriptsize
  \centering
  \begin{tabular}{l c c c c c c @{\hspace{2em}} c c c c c c}
    \toprule
    & \multicolumn{6}{c}{miniF2F-Test} & \multicolumn{6}{c}{ProofNet-Test}\\
    \mycdashline{2-13}
    Method & $\pi_\FV$ & $\widehat{\pi}_\LP$ & $\widehat{\pi}_\MC$  & $\widehat{\pi}_\FQ$ & $\widehat{J}$ & $\widehat{J}=1$ & $\pi_\FV$ & $\widehat{\pi}_\LP$ & $\widehat{\pi}_\MC$  & $\widehat{\pi}_\FQ$ & $\widehat{J}$ & $\widehat{J}=1$\\
    \midrule
    \multicolumn{13}{l}{\textit{DeepSeek-Prover-V2-7B}}\\
    \mycdashline{1-13}
    Best-of-ISR & 60.66 & 58.40 & 86.07 & 63.73 & 44.74 & 11.89 & 18.28 & 50.94 & 78.85 & 55.11 & 12.24 & 3.76\\
    Best-of-$8$ & 88.11 & 72.34 & 90.85 & 69.88 & 68.78 & 21.72 & 43.01 & 57.26 & 83.69 & 66.94 & 25.86 & 6.99\\
    Best-of-$64$ & \textbf{100.00} & 83.71 & 96.58 & 86.89 & \underline{89.06} & \underline{47.54} & \textbf{89.25} & 50.94 & 73.30 & 74.46 & \textbf{56.37} & \underline{16.13}\\
    \midrule
    \multicolumn{13}{l}{\textit{Goedel-Prover-V2-8B}}\\
    \mycdashline{1-13}
    Best-of-ISR & 82.38 & 51.33 & 80.19 & 68.44 & 57.49 & 6.15 & 36.56 & 37.23 & 63.98 & 64.25 & 19.94 & 3.23\\
    Best-of-$8$ & 90.57 & 62.50 & 86.34 & 68.44 & 66.16 & 9.43 & 54.84 & 47.04 & 71.15 & 61.02 & 27.27 & 3.76\\
    Best-of-$64$ & \underline{98.36} & 79.71 & 95.36 & 81.97 & 84.25 & 30.74 & \underline{86.02} & 44.35 & 67.92 & 72.31 & 48.92 & 7.53\\
    \midrule
    \multicolumn{13}{l}{\textit{GPT-5.4}}\\
    \mycdashline{1-13}
    Best-of-ISR & 56.97 & 87.60 & 97.27 & 85.45 & 51.88 & 32.79 & 20.43 & \underline{95.43} & \textbf{98.92} & \underline{91.94} & 19.54 & 14.52\\
    Best-of-$8$ & 49.59 & \textbf{90.37} & \underline{97.40} & \underline{88.93} & 45.24 & 32.79 & 15.59 & \textbf{96.91} & \underline{98.57} & 91.40 & 14.46 & 10.22\\
    \midrule
    \multicolumn{13}{l}{\textit{Multi-Dimensional Reference-Free Refinement Process (Ours)}}\\
    \mycdashline{1-13}
    GPT-5.4 & 60.66 & \underline{90.16} & \textbf{97.54} & 87.09 & 55.64 & 38.52 & 21.51 & 94.49 & \underline{98.57} & \textbf{92.47} & 20.19 & 15.59\\
    Ensemble & \textbf{100.00} & 83.71 & 96.72 & \textbf{90.37} & \textbf{90.27} & \textbf{49.18} & 77.96 & 58.06 & 78.32 & 82.80 & \underline{52.45} & \textbf{17.74}\\
    \bottomrule
  \end{tabular}
  \label{tab:com}
\end{table}

\textbf{Best-of-$k$ is more effective for models with less capability, whereas Best-of-ISR is more effective for those with large capability.} For smaller or weaker base models such as DeepSeek-Prover-7B and Goedel-Prover-8B, Best-of-$k$ significantly outperforms Best-of-ISR in terms of the final score. In contrast, for stronger models such as GPT-5.4, which have shown explicit capability for refinement, Best-of-ISR performs better. This suggests that sampling diversity combined with post-hoc selection is more effective than ISR when the underlying model has limited reasoning capability, whereas ISR benefits more capable models.

\textbf{Best-of-$k$ benefits substantially from increased sampling budgets, but exhibits trade-offs across quality dimensions.} Increasing $k$ from 8 to 64 substantially improves both DeepSeek-Prover-7B and Goedel-Prover-8B. On miniF2F, their overall scores increase from 68.78\% to 89.06\% and from 66.16\% to 84.25\%, respectively. On ProofNet, the corresponding scores increase from 25.86\% to 56.37\% and from 27.27\% to 48.92\%. However, these gains are not uniform across dimensions. On ProofNet, DeepSeek-Prover's FV increases from 43.01\% to 89.25\%, while LP and MC decrease from 57.26\% to 50.94\% and from 83.69\% to 73.30\%, respectively. Goedel-Prover exhibits a similar pattern: FV increases from 54.84\% to 86.02\%, while LP and MC decrease from 47.04\% to 44.35\% and from 71.15\% to 67.92\%, respectively. FQ improves for both models. These findings demonstrate that increased sampling with composite-score selection does not necessarily improve all quality dimensions, highlighting the importance of multi-dimensional evaluation and optimization.

\textbf{The proposed refinement process exhibits complementary strengths to large-scale independent sampling.} In the single-model setting, our approach consistently outperforms both Best-of-ISR and Best-of-$8$ using GPT-5.4 in terms of the overall score and perfect-formalization rate. On miniF2F, our ensemble achieves a higher overall score (90.27\%) than DeepSeek-Prover Best-of-$64$ (89.06\%) and Goedel-Prover Best-of-$64$ (84.25\%), as well as higher perfect-formalization rates (49.18\% vs. 47.54\% and 30.74\%, respectively). On ProofNet, our ensemble achieves a higher overall score than Goedel-Prover Best-of-$64$ (52.45\% vs. 48.92\%), although DeepSeek-Prover Best-of-$64$ achieves a higher score (56.37\%). Both Best-of-$64$ baselines attain higher FV than our ensemble, whereas our approach achieves higher scores across all three soft dimensions (LP, MC, and FQ) and a higher perfect-formalization rate (17.74\% vs. 16.13\% and 7.53\%, respectively). These results highlight the trade-offs between formal validity and soft quality dimensions, and demonstrate the complementary value of multi-dimensional refinement relative to large-scale independent sampling.

\textbf{Computational cost vs. performance gain.} We report computational costs in terms of LLM calls in Tables~\ref{tab:cost_minif2f} and~\ref{tab:cost_proofnet} (Appendix~\ref{app:add}). In the single-model setting, our refinement process achieves higher overall performance than Best-of-$8$ and Best-of-ISR with fewer judge calls, albeit with comparable or higher generation budgets. On miniF2F, our ensemble achieves a higher overall score (90.27\%) than DeepSeek-Prover and Goedel-Prover Best-of-$64$ (89.06\% and 84.25\%), while using fewer generator calls (43.10 vs. 64.00 for both) and judge calls (147.39 vs. 175.98 and 297.55, respectively). On ProofNet, our ensemble uses 66.79 generator calls and 249.48 judge calls, compared with 64.00 and 227.42 for DeepSeek-Prover Best-of-$64$, and 64.00 and 248.76 for Goedel-Prover Best-of-$64$. Under this nearly matched call budget with Goedel-Prover, our ensemble achieves a higher overall score (52.45\% vs. 48.92\%) and perfect-formalization rate ($\widehat{J}=1$: 17.74\% vs. 7.53\%). DeepSeek-Prover Best-of-$64$ achieves a higher overall score (56.37\%) on ProofNet with fewer calls but a lower perfect-formalization rate (16.13\%). These results indicate that the computational trade-offs between structured refinement and independent sampling depend on the benchmark and sampling model, motivating more efficient generator allocation strategies.

\section{Related Work}
Recent advances in LLMs have renewed interest in autoformalization, which bridges informal mathematical language and formal proof systems~\citep{wu2022autoformalization,yang2025position,mensfelt2025commonframeworkautoformalization,zhang-etal-2025-autoformalization,wu2025stepfunformalizerunlockingautoformalizationpotential,xuejun2025mathesisformaltheoremproving}. Natural language proofs have further enabled automated theorem-proving pipelines~\citep{jiang2023draft,tarrach2024more,liu2025bootstrapping}, but integrating such reasoning directly into autoformalization remains challenging. Retrieval-augmented methods improve correctness by leveraging existing formal libraries in Isabelle and Lean~\citep{zhang-etal-2024-consistent,yang2023leandojo,liu2025rethinking,wang2025improvingautoformalizationusingdirect,zhang2025driftdecomposeretrieveillustrate}, but are inherently reference-based. In contrast, we propose a reference-free framework. In contrast to prior process- and agent-based approaches~\citep{lu2024processdrivenautoformalizationlean4,zhang-etal-2025-masa} and self-consistency methods~\citep{li2024autoformalize}, our framework explicitly combines multi-dimensional quality assessment, role-specialized refinement, and an estimated-score acceptance policy with a formal analysis of its behavior under judge uncertainty. An extended related work is provided in Appendix~\ref{app:ref}.

\section{Conclusion}
We present a multi-dimensional, reference-free refinement framework for full-theorem autoformalization that integrates a masked composite objective, role-specialized LLMs guided by a responsiveness map, and an iterative acceptance policy. The framework ensures non-decreasing estimated objective scores and provides conditional theoretical guarantees characterizing the refinement process under judge uncertainty. Experiments on miniF2F and ProofNet demonstrate that combining complementary strengths of general-purpose and specialized LLMs improves aggregate formalization quality, with supporting evidence from independent LLM and human evaluations. Future work includes compute-efficient generator allocation, stronger calibration of quality judgments, and evaluation on more complex mathematical formalization tasks.

\bibliography{custom}
\bibliographystyle{iclr2027_conference}

\appendix
\section{Proofs}
\subsection{Proof of Theorem~\ref{thm:lcb}}\label{app:lcb}
\begin{proof}
We have:
\begin{equation}
\Pr\Big(\LCB_\delta(x)\le J_\OA(x)\Big)=\Pr\Big(0\le J_\OA(x), \pi_\FV(x)\sum_{i\in\mathcal{D}_\soft} w_i\big(\widehat{\pi}_i(x)-m_i(\delta_i)\big)\le J_\OA(x)\Big)
\end{equation}

Let $d=|\mathcal{D}_\soft|$. Since $0\le J_\OA(x)$ by definition in Eq.~\ref{eq:J}, we have:
\begin{equation}
    \begin{split}
    &\Pr\Big(\LCB_\delta(x)\le J_\OA(x)\Big)\\
    &\quad=\Pr\Big( \pi_\FV(x)\sum_{i\in\mathcal{D}_\soft} w_i\big(\widehat{\pi}_i(x)-m_i(\delta_i)\big)\le J_\OA(x)\Big)\\
    &\quad\ge \Pr\Big(\sum_{i\in\mathcal{D}_\soft} w_i\big(\widehat{\pi}_i(x)-m_i(\delta_i)\big)\le \sum_{i\in\mathcal{D}_\soft}w_i\pi_i(x)\Big)\\
    &\quad\ge \Pr\Big(w_1\big(\widehat{\pi}_1(x)-m_1(\delta_1)\big)\le w_1\pi_1(x),\dots,w_d\big(\widehat{\pi}_d(x)-m_d(\delta_d)\big)\le w_d\pi_d(x)\Big)\\
    &\quad=\Pr\Big(\widehat{\pi}_1(x) \le \pi_1(x)+m_1(\delta_1),\dots,\widehat{\pi}_d(x)\le \pi_d(x)+m_d(\delta_d)\Big).
    \end{split}
\end{equation}

By the complementary form of Boole's inequality and Eq.~\ref{eq:lcb_i}, we obtain
\begin{equation}
    \begin{split}
    &\Pr\Big(\widehat{\pi}_1(x) \le \pi_1(x)+m_1(\delta_1),\dots,\widehat{\pi}_d(x)\le \pi_d(x)+m_d(\delta_d)\Big)\\
    &\ge 1-\sum_{i\in\mathcal{D}_\soft}\Big[1-\Pr\Big(\widehat{\pi}_i(x)\le \pi_i(x)+m_i(\delta_i)\Big)\Big]\\
    &\ge 1-\sum_{i\in\mathcal{D}_\soft}\delta_i=1-\delta.
    \end{split}
\end{equation}

Combining derivations above, we have: $\Pr\Big(\LCB_\delta(x)\le J_\OA(x)\Big)\ge1-\delta$, meaning $\LCB_\delta(x)$ is a lower confidence bound for $J_\OA(x)$ with confidence level $1-\delta$.
\end{proof}

\subsection{Proof of Theorem~\ref{thm:mono}}\label{app:mono}
\begin{proof}
Let $M_\delta=\sum_{i\in\mathcal{D}_\soft}w_i m_i(\delta_i)$. If $\pi_\FV(x)=1$, by Eq.~\ref{eq:est} and Eq.~\ref{eq:lcb}, we have $\LCB_\delta(x)=\max\{0,\widehat{J}(x)-M_\delta\}$. If $\pi_\FV(x)=0$, we have $\LCB_\delta(x)=0$ and $\widehat{J}(x)\leq\epsilon\leq M_\delta$, therefore, $\LCB_\delta(x)=0=\max\{0,\widehat{J}(x)-M_\delta\}$. In both cases, $\LCB_\delta(x)=\max\{0,\widehat{J}(x)-M_\delta\}$.

By the acceptance policy Eq.~\ref{eq:acc}, for all $t$, we have $\widehat{J}(x_{t+1})\geq\widehat{J}(x_t)$. Consequently,
\begin{equation}
    \LCB_\delta(x_{t+1})=\max\{0,\widehat{J}(x_{t+1})-M_\delta\}\geq\max\{0,\widehat{J}(x_t)-M_\delta\}=\LCB_\delta(x_t).
\end{equation}
\end{proof}

\subsection{Proof of Theorem~\ref{thm:con}}\label{app:con}
\begin{proof}
(i) By Theorem~\ref{thm:mono}, the sequence
$\{\LCB_\delta(x_t)\}$ is non-decreasing.
Since it is bounded above by $1$, the monotone convergence
theorem implies that $\{\LCB_\delta(x_t)\}$ converges.

(ii) Since there exists $x\in\bigcup_t\mathcal{C}_t$ such that $\widehat{J}(x)=1$, the set $\{t\mid\exists x\in\mathcal{C}_t,\widehat{J}(x)=1\}$ is non-empty. Let $t_0=\min\{t\mid\exists x\in\mathcal{C}_t,\widehat{J}(x)=1\}$ and let $x^*=x^\star_{t_0+1}\in\arg\max_{x\in\mathcal{C}_{t_0}}\widehat{J}(x)$. Since $\widehat{J}(x_0)<1$ and $t_0$ is the first step at which a candidate attains the maximum estimated objective, we have $\widehat{J}(x_{t_0})<1=\widehat{J}(x^*)$. By the acceptance rule in Eq.~\ref{eq:acc}, $x_{t_0+1}=x^*$. For any $t$ such that $\widehat{J}(x_t)=1$, we have $\max_{x\in\mathcal{C}_t}\widehat{J}(x)\leq 1=\widehat{J}(x_t).$ Hence, by the acceptance rule, $x_{t+1}=x_t$. By induction, we have $\forall t>t_0,\ x_t=x^*$.

Since $\widehat{J}(x^*)=1$, we have $\pi_\FV(x^*)=1$. By Eq.~\ref{eq:est} and \ref{eq:lcb}, we obtain 
\begin{equation}
    \LCB_\delta(x^*)=\max\left\{0,1-\sum_{i\in\mathcal{D}_\soft} w_i m_i(\delta_i)\right\}.
\end{equation} 
Under the measurable uncertainty assumption applicable to the selected candidate, it follows that
\begin{equation}
    \Pr\Big(1-\sum_{i\in\mathcal{D}_\soft} w_i m_i(\delta_i)\leq J_\OA(x^*)\Big) \geq 1-\delta.
\end{equation}
\end{proof}

\section{Algorithms}\label{app:alg}
We provide the algorithm we used in practice to construct candidate formalizations at each time step in Algorithm~\ref{alg:sample}, and the algorithm of the multi-dimensional reference-free refinement process in Algorithm~\ref{alg:mono}.

\begin{algorithm}[!t]
    \caption{Candidates Construction at Step $t$}
    \small
    \label{alg:sample}
    \begin{algorithmic}[1]
        \STATE{\textbf{Inputs:} Natural language statement and proof $(s_\nl,p_\nl)$, Previous formalization $x_t$, Set of One-Off Generators $\mathcal{G}_i$, Set of FV-Repairers $\mathcal{G}_j$, Set of Recurrent Generators $\mathcal{G}_k$, Theorem Prover $\mathcal{TP}$.}
        \STATE{$\mathcal{C}_t \gets \emptyset$}
        \FOR{$G\in\mathcal{G}_i$}
            \STATE{Sample one $x\sim p_G(x\mid s_\nl,p_\nl)$}
            \IF{$\pi_\FV(x)==0$ from $\mathcal{TP}$}
                \FOR{$G'\in\mathcal{G}_j$}
                    \STATE{Sample one $x'\sim p_{G'}(x'\mid s_\nl,p_\nl,x)$}
                    \STATE{$\mathcal{C}_t \gets \mathcal{C}_t\cup\{x'\}$}
                \ENDFOR
            \ELSE
                \STATE{$\mathcal{C}_t \gets \mathcal{C}_t\cup\{x\}$}
            \ENDIF
        \ENDFOR
        \IF{$t>0$}
            \FOR{$G\in\mathcal{G}_k$}
                \STATE{Sample one $x\sim p_G(x\mid s_\nl,p_\nl,x_t)$}
                \IF{$\pi_\FV(x)==0$ from $\mathcal{TP}$}
                    \FOR{$G'\in\mathcal{G}_j$}
                        \STATE{Sample one $x'\sim p_{G'}(x'\mid s_\nl,p_\nl,x)$}
                        \STATE{$\mathcal{C}_t \gets \mathcal{C}_t\cup\{x'\}$}
                    \ENDFOR
                \ELSE
                    \STATE{$\mathcal{C}_t \gets \mathcal{C}_t\cup\{x\}$}
                \ENDIF
            \ENDFOR
        \ENDIF
    \end{algorithmic}
\end{algorithm}

\begin{algorithm}[!t]
    \caption{The Multi-Dimensional Reference-Free Refinement Process}
    \small
    \label{alg:mono}
    \begin{algorithmic}[1]
        \STATE{\textbf{Inputs:} Natural language statements and proofs $\mathcal{D}=\{(s_\nl,p_\nl)\}$, Soft dimensions $\mathcal{D}_\soft$, Sets of LLM Judges $\big\{\mathcal{J}_i\big\}_{i\in\mathcal{D}_\soft}$, Theorem Prover $\mathcal{TP}$, Maximum number of iterations $T$.}
        \FOR{$(s_\nl,p_\nl) \in \mathcal{D}$}
            \STATE{$x_0 \gets ""$, $J_0 \gets -1$}
            \FOR{$t=0,\dots,T-1$}            
                \STATE{Construct $\mathcal{C}_t$ (e.g., Algorithm~\ref{alg:sample})}
                \STATE{$J \gets J_t$}
                \FOR{$x \in \mathcal{C}_t$}
                    \STATE{Obtain $\pi_\FV(x)$ from theorem prover $\mathcal{TP}$}
                    \IF{$J\leq\epsilon$ or $\pi_\FV(x)==1$}
                        \FOR{$i\in\mathcal{D}_\soft$}
                            \STATE{Obtain estimation $\widehat{\pi}_i(x)$ from $\mathcal{J}_i$ using Eq.~\ref{eq:judge}}
                        \ENDFOR
                        \STATE{Calculate plug-in estimator $\widehat{J}(x)$ using Eq.~\ref{eq:est}}
                        \IF{$\widehat{J}(x)\geq J$}
                            \STATE{$x_{t+1}^\star \gets x$, $J \gets \widehat{J}(x)$}
                        \ENDIF
                    \ENDIF
                \ENDFOR
                \IF{$J>J_t$}
                    \STATE{$x_{t+1} \gets x_{t+1}^\star$, $J_{t+1} \gets J$}
                \ELSE
                    \STATE{$x_{t+1} \gets x_t$, $J_{t+1} \gets J_t$}
                \ENDIF
                \IF{$J_{t+1}==1$}
                    \STATE{\textbf{break}}
                \ENDIF
            \ENDFOR
        \ENDFOR
    \end{algorithmic}
\end{algorithm}

\section{Rationale for Empirical Setups}\label{app:rat}
We describe the existing formal benchmarks and the reasons for selecting miniF2F and ProofNet as follows:
\begin{itemize}
    \item miniF2F~\citep{zheng2022miniff} is a dataset of formal Olympiad-level mathematics problems with both Isabelle~\citep{paulson2000isabelle} and Lean formalizations. The original dataset lacks natural language proofs, but an improved version~\citep{jiang2023draft} provides them. It contains 244 validation instances and 244 test instances. Due to the availability of natural language proofs and its relatively small size, miniF2F is considered as a representative benchmark for validating the methods in this work.
    \item ProofNet~\citep{azerbayev2023proofnetautoformalizingformallyproving} is a benchmark for autoformalization and formal proving of undergraduate-level mathematics, consisting of 185 validation samples and 186 test samples. Compared to miniF2F, it has higher complexity. Its combination of natural language proofs and manageable scale makes it suitable for evaluating our proposed process.
    \item Lean Workbook~\citep{ying2024lean} contains 57,231 formalized math problems (5k with formal solutions) generated via a pipeline of synthetic data generation and filtering. However, it lacks natural language proofs, making it unsuitable for the full-theorem autoformalization task in this work.
    \item FormalMATH~\citep{yu2025formalmathbenchmarkingformalmathematical} is a large-scale Lean4 benchmark with 5,560 formally verified problems ranging from high-school Olympiad challenges to undergraduate-level theorems. While it includes natural language proofs, its large scale makes it less practical as a representative benchmark for testing our methods in this paper.
    \item Herald~\citep{gao2025herald} contains 579,883 samples generated from a structural-information-aware pipeline, of which 44,553 include natural language proofs. Similar to FormalMATH, although it provides natural language proofs, its large scale limits its suitability as a representative benchmark for this work.
\end{itemize}

We elicit the LLMs used in this work and the reasons for their selection as follows:
\begin{itemize}
    \item DeepSeek-Prover-V2~\citep{ren2025deepseekproverv2advancingformalmathematical} is an open-source LLM designed for formal theorem proving in Lean4, available in 7B and 671B sizes. This model family has demonstrated state-of-the-art performance in generating formal proofs~\citep{chen2025seedproverdeepbroadreasoning}. Since this work focuses on full-theorem autoformalization, which requires complete formalization including proofs, its superiority in generating Lean4 proofs makes it a suitable candidate. The 7B version is selected here due to computational resource constraints.
    \item Goedel-Prover-V2~\citep{lin2025goedelproverv2scalingformaltheorem} is another series of open-source LLMs for automated theorem proving in Lean4, available in 8B and 32B sizes. Unlike DeepSeek-Prover-V2, this model can iteratively revise its proofs by leveraging feedback from the Lean theorem prover, which is particularly useful in the FV-Repair setting of this work. We select the 8B version due to computational resource constraints.
    \item Qwen2.5~\citep{qwen2025qwen25technicalreport} is a comprehensive series of open-source models designed to meet diverse needs. Qwen2.5-7B-Instruct is a representative small-scale general-purpose model within our computational budget and is therefore included in this study.
    \item Seed-Coder~\citep{seed2025seedcoderletcodemodel} is a series of open-source LLMs specialized in code generation. While not fully general-purpose, formal languages can be considered a type of code, and thus this model serves as a state-of-the-art coding model suitable for our purposes. The 8B version is chosen due to computational resource limitations.
    \item GPT series~\citep{openai2024gpt4,openai2025gpt5} are the latest and largest state-of-the-art general-purpose closed-source LLM at the time of this work. Full versions of GPT-4.1 and GPT-5.4 are used for generating formalizations, while the smaller GPT-4.1-mini balances performance and cost, making it suitable as LLM judges in our experiments.
\end{itemize}

\section{Extended Related Work}\label{app:ref}
\paragraph{Autoformalization} Autoformalization refers to the automatic translation of natural-language mathematics into formal mathematics~\citep{wu2022autoformalization,yang2025position,mensfelt2025commonframeworkautoformalization}. Most research in this area has focused on statement formalization, including theorem autoformalization~\citep{zhang-etal-2024-consistent,lu2024processdrivenautoformalizationlean4,li2024autoformalize,gao2025herald} and definition autoformalization~\citep{zhang-etal-2025-autoformalization}. On the proof autoformalization side, some works develop pipelines for generating formal proofs~\citep{ying2024lean,gao2025herald}, however, the quality of the generated proofs is often unclear, and such outputs are mainly used for fine-tuning. Other approaches leverage natural-language proof reasoning steps within automated theorem-proving pipelines to improve formal reasoning~\citep{jiang2023draft,tarrach2024more,liu2025bootstrapping}. While integrating natural-language reasoning can increase proof success rates, it also presents an opportunity for full-theorem autoformalization, directly translating claims and proofs from natural language into fully formalized theorems. This paper focuses specifically on this scope.

\paragraph{Retrieval-Augmented Autoformalization.} Retrieval-Augmented Generation (RAG)~\citep{lewis2020rag} has shown improvements in code generation~\citep{lu-etal-2022-reacc, zhang-etal-2023-refsql}. In formal mathematics, \citet{yang2023leandojo} trained a retrieval-augmented language model for premise selection and theorem proving in Lean. In autoformalization, \citet{zhang-etal-2024-consistent} use a BM25 retriever to find similar examples from existing libraries in Isabelle and inject them into prompts. Similarly, \citet{liu2025rethinking,wang2025improvingautoformalizationusingdirect} retrieve potentially dependent objects from formal libraries to augment statement autoformalization in Lean, and \citet{zhang2025driftdecomposeretrieveillustrate} decomposes informal statements into smaller components and retrieves relevant premises from Mathlib. All of these methods rely on content from formal mathematical libraries as references. In contrast, our proposed refinement process operates in a reference-free manner, requiring no retrieval from existing formal libraries to guide generation.

\paragraph{Test-Time Optimization.} Chain-of-Thoughts~\citep{wei2022chain} highlights the importance of studying optimization at test time for LLMs. Building on this idea, token-level methods such as Tree-of-Thoughts~\citep{yao2023tree} focus on optimizing intermediate tokens to produce higher-quality completions, whereas response-level methods such as Self-Consistency~\citep{wang2023selfconsistency} aggregate multiple LLM outputs to improve overall performance. In the domain of autoformalization, \citet{li2024autoformalize} propose a scoring framework that selects the best candidate from multiple autoformalizations using complementary self-consistency measures. Multi-agent frameworks~\citep{zhang-etal-2025-masa} similarly organize outputs from multiple LLMs to collectively optimize results. Our framework extends these ideas by employing both theorem provers and LLM-based judges to iteratively select high-quality autoformalization candidates, resulting in an improvement process at test time.

\section{Complementary Results}\label{app:com}
All smaller LLMs used in this work are run on two 24GB NVIDIA Quadro RTX 6000 GPUs, while GPT models are accessed via their API.

\subsection{Determination of Generator Roles}\label{app:role}

For the following experiments, we use greedy decoding for open-source models to improve reproducibility. We sample one instance per generator to reduce computational cost, while empirical evidence suggests that the results remain indicative.

\subsubsection{One-Off Generators and FV-Repairers}
We assess LLMs' capabilities as (i) One-Off Generators under zero-shot autoformalization, and (ii) FV-Repairers under refinement guided by rigorous error feedback from theorem provers. We report scores in Table~\ref{tab:ref_mini}.

\begin{table}[!t]
  \caption{Performance of LLMs as One-Off Generators (\textbf{OOG}) and FV-Repairers (\textbf{FVR}) for full theorem autoformalization on miniF2F and ProofNet benchmark using Lean4. FVRs are applied to formalizations generated by OOGs (i.e., zero-shot autoformalization).}
  \small
  \centering
  \begin{tabular}{l l c c c c c}
    \toprule
    Type & LLM & $\pi_\FV$ & $\widehat{\pi}_\LP$ & $\widehat{\pi}_\MC$ & $\widehat{\pi}_\FQ$ & $\widehat{J}$\\
    \midrule
    \multicolumn{3}{l}{\textit{miniF2F-Test}}\\
    \midrule
    OOG & Qwen2.5-7B-Instruct & 0.00 & 28.89 & 73.50 & 15.78 & 0.04\\
    \mycdashline{1-7}
    FVR & Qwen2.5-7B-Instruct & 0.00 & \dec0.61 & \inc1.23 & \inc0.41 & 0.00\\
    FVR & Seed-Coder-8B-Instruct & \inc0.82 & \inc0.92 & \dec1.78 & \inc2.05 & \inc0.57\\
    FVR & DeepSeek-Prover-V2-7B & \inc22.95 & \inc18.24 & \inc4.64 & \inc23.36 & \inc15.07\\
    FVR & Goedel-Prover-V2-8B & \inc30.33 & \inc12.60 & \dec6.01 & \inc28.89 & \inc18.71\\
    FVR & GPT-4.1 & \inc12.30 & \inc59.84 & \inc23.91 & \inc53.89 & \inc11.31\\
    \mycdashline{1-7}
     OOG & Seed-Coder-8B-Instruct & 7.38 & 28.38 & 63.39 & 21.11 & 5.35\\
    \mycdashline{1-7}
    FVR & Seed-Coder-8B-Instruct & \inc1.64 & \dec1.43 & \inc3.01 & \inc2.05 & \inc1.21\\
    FVR & DeepSeek-Prover-V2-7B & \inc21.31 & \inc13.73 & \inc10.52 & \inc16.39 & \inc14.45\\
    FVR & Goedel-Prover-V2-8B & \inc20.49 & \inc8.61 & \inc3.28 & \inc18.65 & \inc11.97\\
    FVR & GPT-4.1 & \inc13.11 & \inc54.00 & \inc32.79 & \inc46.93 & \inc11.54\\
    \midrule
    OOG & DeepSeek-Prover-V2-7B & 50.41 & 55.02 & 80.87 & 51.02 & 35.15\\
    \mycdashline{1-7}
    FVR & Qwen2.5-7B-Instruct & 0.00 & \dec1.74 & \inc1.37 & \dec5.33 & 0.00\\
    FVR & Seed-Coder-8B-Instruct & \inc1.64 & \dec4.92 & \dec1.50 & \dec6.97 & \inc1.16\\
    FVR & DeepSeek-Prover-V2-7B & \inc1.64 & \inc0.92 & \inc0.96 & \dec2.05 & \inc0.90\\
    FVR & Goedel-Prover-V2-8B & \inc6.56 & \dec3.38 & \dec2.05 & \inc1.02 & \inc4.02\\
    FVR & GPT-4.1 & \inc3.69 & \inc18.24 & \inc9.97 & \inc11.48 & \inc3.38\\
    \midrule
    OOG & Goedel-Prover-V2-8B & 58.61 & 47.34 & 73.63 & 49.80 & 35.61\\
    \mycdashline{1-7}
    FVR & DeepSeek-Prover-V2-7B & \inc1.64 & 0.00 & \dec0.68 & \dec1.02 & \inc1.15\\
    FVR & Goedel-Prover-V2-8B & \inc7.79 & \dec1.23 & \dec4.51 & \inc1.43 & \inc5.50\\
    FVR & GPT-4.1 & \inc2.05 & \inc16.91 & \inc9.15 & \inc10.25 & \inc1.92\\
    \midrule
    OOG & GPT-4.1 & 9.02 & 90.16 & 98.09 & 58.81 & 8.33\\
    \mycdashline{1-7}
    FVR & DeepSeek-Prover-V2-7B & \inc20.08 & \dec7.99 & \dec3.96 & \dec3.69 & \inc16.36\\
    FVR & Goedel-Prover-V2-8B & \inc27.87 & \dec20.39 & \dec9.15 & \dec2.05 & \inc21.65\\
    FVR & GPT-4.1 & \inc6.56 & \inc2.46 & \inc1.23 & \inc9.22 & \inc6.08\\
    \midrule
    \multicolumn{3}{l}{\textit{ProofNet-Test}}\\
    \midrule
    OOG & Seed-Coder-8B-Instruct & 0.00 & 13.71 & 51.25 & 6.18 & 0.02\\
    \midrule
    OOG & DeepSeek-Prover-V2-7B & 11.83 & 41.80 & 67.38 & 35.75 & 7.14\\
    \mycdashline{1-7}
    FVR & DeepSeek-Prover-V2-7B & \inc0.54 & \inc1.21 & \inc2.15 & \inc3.49 & \inc0.35\\
    FVR & Goedel-Prover-V2-8B & \inc4.30 & \dec2.96 & \dec0.54 & \inc3.76 & \inc2.88\\
    FVR & GPT-4.1 & \inc1.08 & \inc46.77 & \inc27.42 & \inc36.56 & \inc1.11\\
    \midrule
    OOG & Goedel-Prover-V2-8B & 17.74 & 31.59 & 63.44 & 29.84 & 6.89\\
    \mycdashline{1-7}
    FVR & DeepSeek-Prover-V2-7B & 0.00 & \inc0.13 & \dec1.43 & \inc0.81 & 0.00\\
    FVR & Goedel-Prover-V2-8B & \inc2.69 & \dec7.53 & \dec5.38 & \inc4.30 & \inc1.61\\
    FVR & GPT-4.1 & \inc1.08 & \inc48.52 & \inc25.63 & \inc38.44 & \inc1.11\\
    \midrule
    OOG & GPT-4.1 & 1.08 & 91.67 & 99.46 & 66.67 & 1.12\\
    \mycdashline{1-7}
    FVR & DeepSeek-Prover-V2-7B & \inc1.61 & \dec10.22 & \dec5.02 & \dec7.26 & \inc1.51\\
    FVR & Goedel-Prover-V2-8B & \inc6.99 & \dec31.85 & \dec16.31 & \dec6.99 & \inc4.76\\
    FVR & GPT-4.1 & 0.00 & \inc2.96 & \dec0.18 & \inc8.33 & 0.00\\
    \bottomrule
  \end{tabular}
  \label{tab:ref_mini}
\end{table}

\textbf{Models fine-tuned on the specialized domain perform best in the One-Off Generator setting.} Both DeepSeek-Prover-7B and Goedel-Prover-8B achieve higher FV scores than GPT-4.1 under this setting. In particular, Goedel-Prover-8B produces syntactically correct formal theorems and proofs for 58.61\% of miniF2F examples and 17.74\% of ProofNet examples. Under zero-shot autoformalization, specialized LLMs generate formalizations with higher FV and overall scores, making them well suited to the One-Off Generator setting. However, formalizations produced by these specialized LLMs still receive relatively lower scores on soft semantic dimensions such as LP and FQ, indicating persistent challenges in alignment faithfulness.

\textbf{Specialized LLMs and general-purpose larger LLMs can both serve as effective FV-Repairers for high-validity formalizations.} On miniF2F, when acting as FV-Repairers, the general-purpose smaller LLMs Qwen2.5-7B and Seed-Coder-8B show no and limited improvement in FV, respectively. Across both miniF2F and ProofNet, GPT-4.1 achieves larger gains in FV than DeepSeek-Prover-7B when refining formalizations that already have relatively high initial validity, although its improvements still fall short of Goedel-Prover-8B's performance. Goedel-Prover-8B delivers the highest overall gains in FV, but these improvements come at the cost of lower scores on soft-semantic dimensions such as LP and MC. This trade-off indicates that FV alone is insufficient for evaluating overall performance. Taken together, these results highlight the strengths of specialized LLMs and GPT-4.1 as FVRs. Under the FVR setting with Goedel-Prover-8B, we obtain 66.39\% syntactically correct formal theorems with proofs on miniF2F and 20.43\% on ProofNet.

\textbf{Specialized LLMs acting as FV-Repairers on low-validity formalizations do not necessarily outperform their own zero-shot autoformalization results.} On miniF2F, when DeepSeek-Prover-7B and Goedel-Prover-8B refine zero-shot outputs generated by Qwen2.5-7B, Seed-Coder-8B, or GPT-4.1, the percentage of recovered syntactically correct full theorems is substantially lower than their respective zero-shot performances. GPT-4.1 does achieve more correct full theorems when refining the zero-shot outputs of Qwen2.5-7B and Seed-Coder-8B than in its own zero-shot setting, but the magnitude of its gains remains smaller than those achieved by the specialized LLMs.

\subsubsection{Responsiveness Maps of Generators}

We evaluate the ability of LLMs to act as Recurrent Generators under iterative refinement with feedback from LLM judges. In addition to REG-LP, REG-MC, and REG-FQ, we also consider REG-ALL, which uses feedback aggregated across all evaluation dimensions. We consider Qwen2.5-7B, Seed-Coder-8B, DeepSeek-Prover-7B, and Goedel-Prover-8B as generators, but exclude GPT-4.1 to avoid potential bias from GPT-4.1-mini judgments. Refinement is applied to formalizations with high formal validity produced by the specialized LLMs DeepSeek-Prover-7B and Goedel-Prover-8B. For all generators that use LLM-based feedback (FVR and REG), we compute the responsiveness map for each individual instance and report the mean across all samples for the four LLMs in Figure~\ref{fig:res}.

\begin{figure}[!t]
    \centering
    \begin{subfigure}{0.24\textwidth} 
      \centering
      \includegraphics[width=\textwidth]{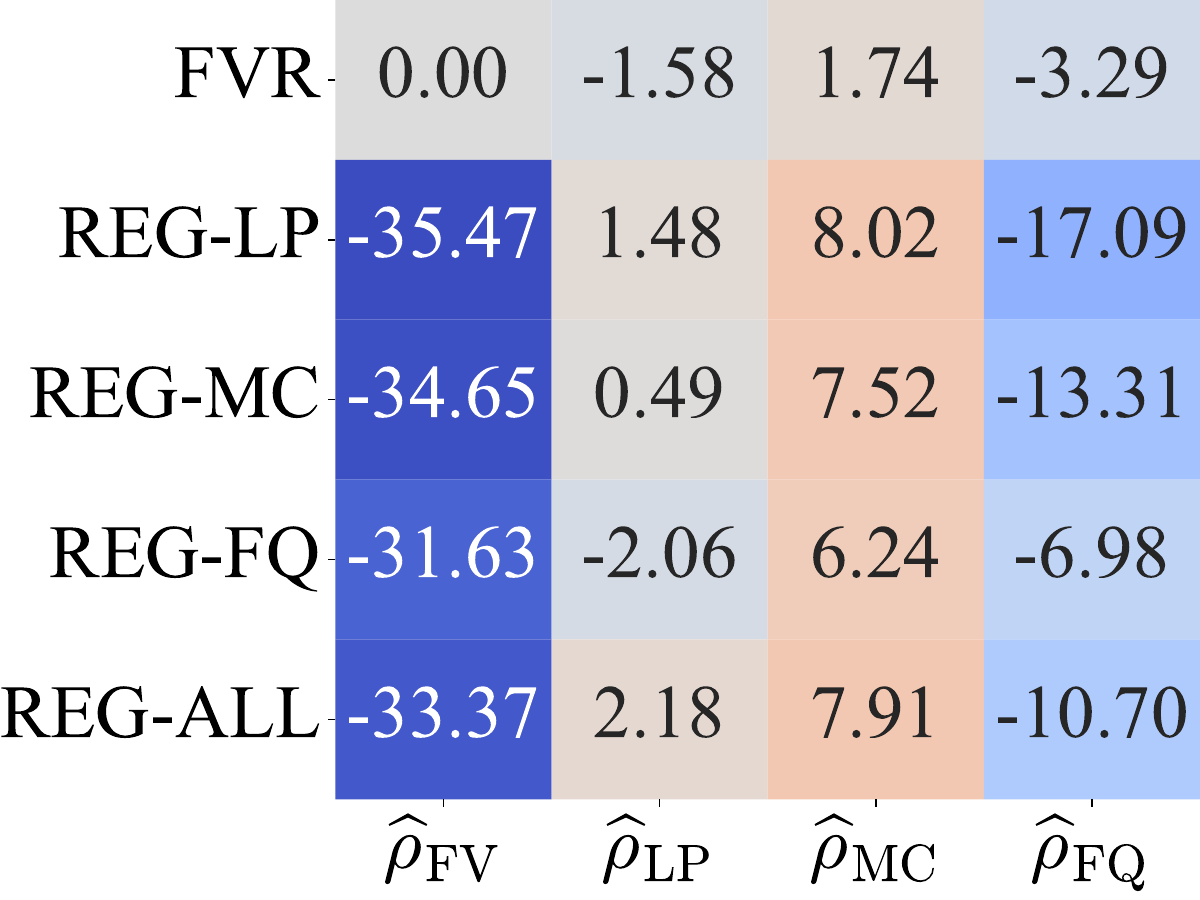}
      \caption{Qwen2.5-7B}
      \label{fig:qwen}
    \end{subfigure}
    \begin{subfigure}{0.24\textwidth} 
      \centering
      \includegraphics[width=\textwidth]{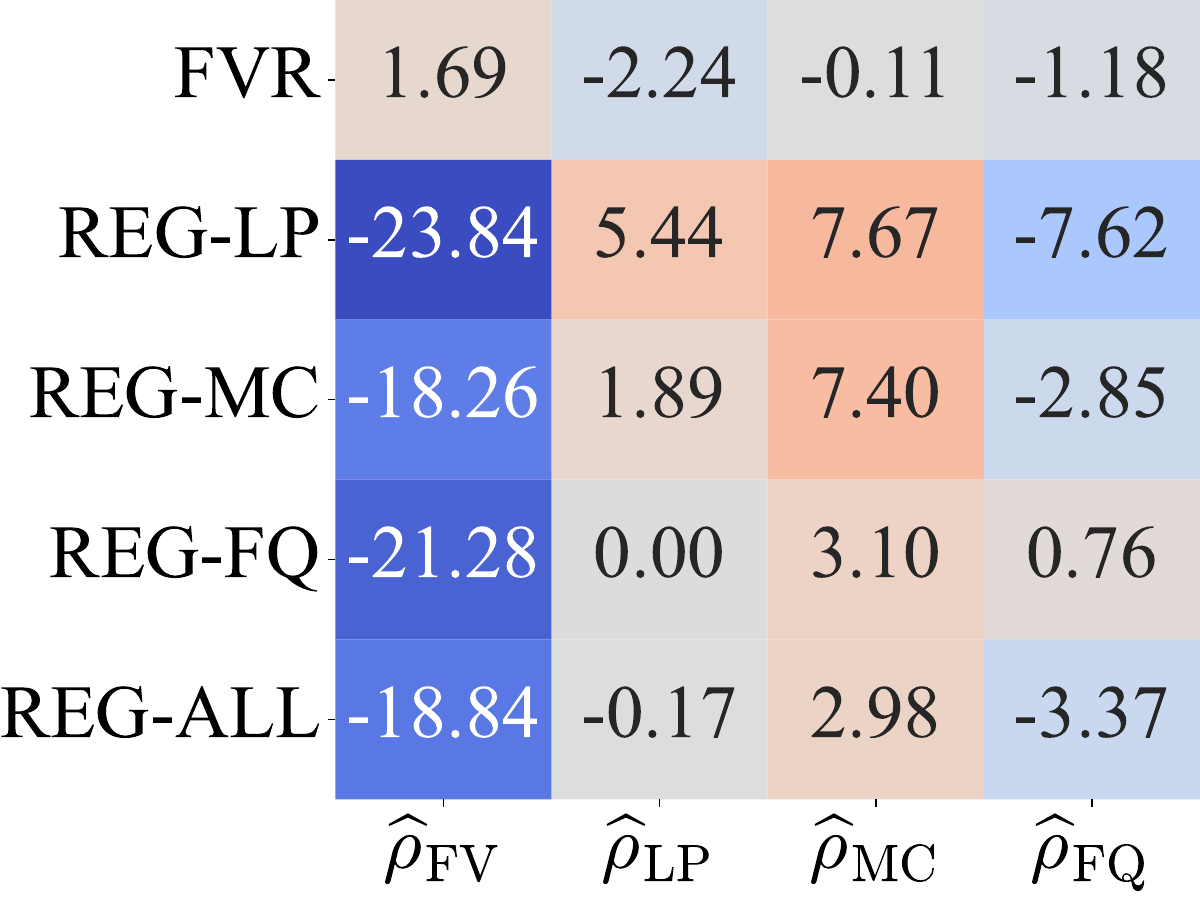}
      \caption{Seed-Coder-8B}
      \label{fig:seed}
    \end{subfigure}
    \begin{subfigure}{0.24\textwidth} 
      \centering
      \includegraphics[width=\textwidth]{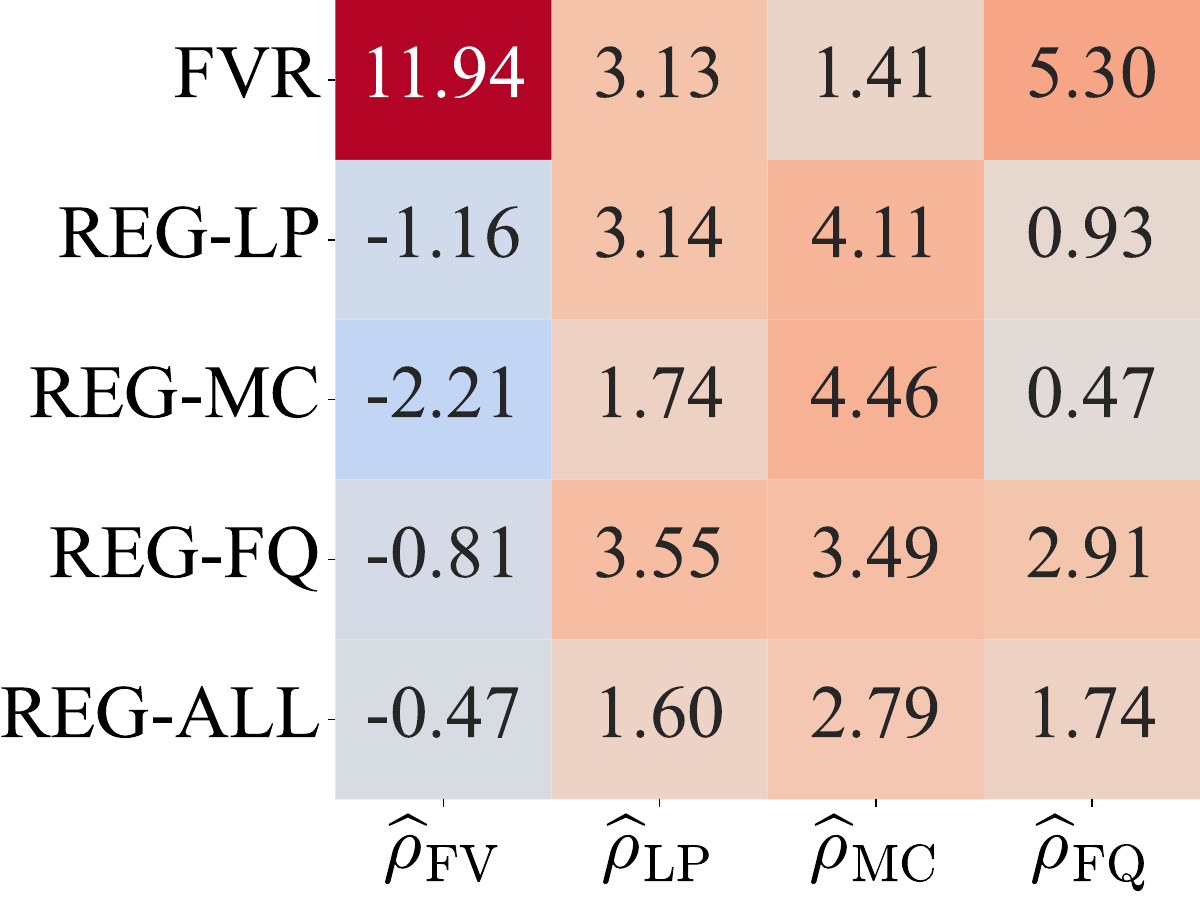}
      \caption{DeepSeek-Prover-7B}
      \label{fig:dsp}
    \end{subfigure}
    \begin{subfigure}{0.24\textwidth} 
      \centering
      \includegraphics[width=\textwidth]{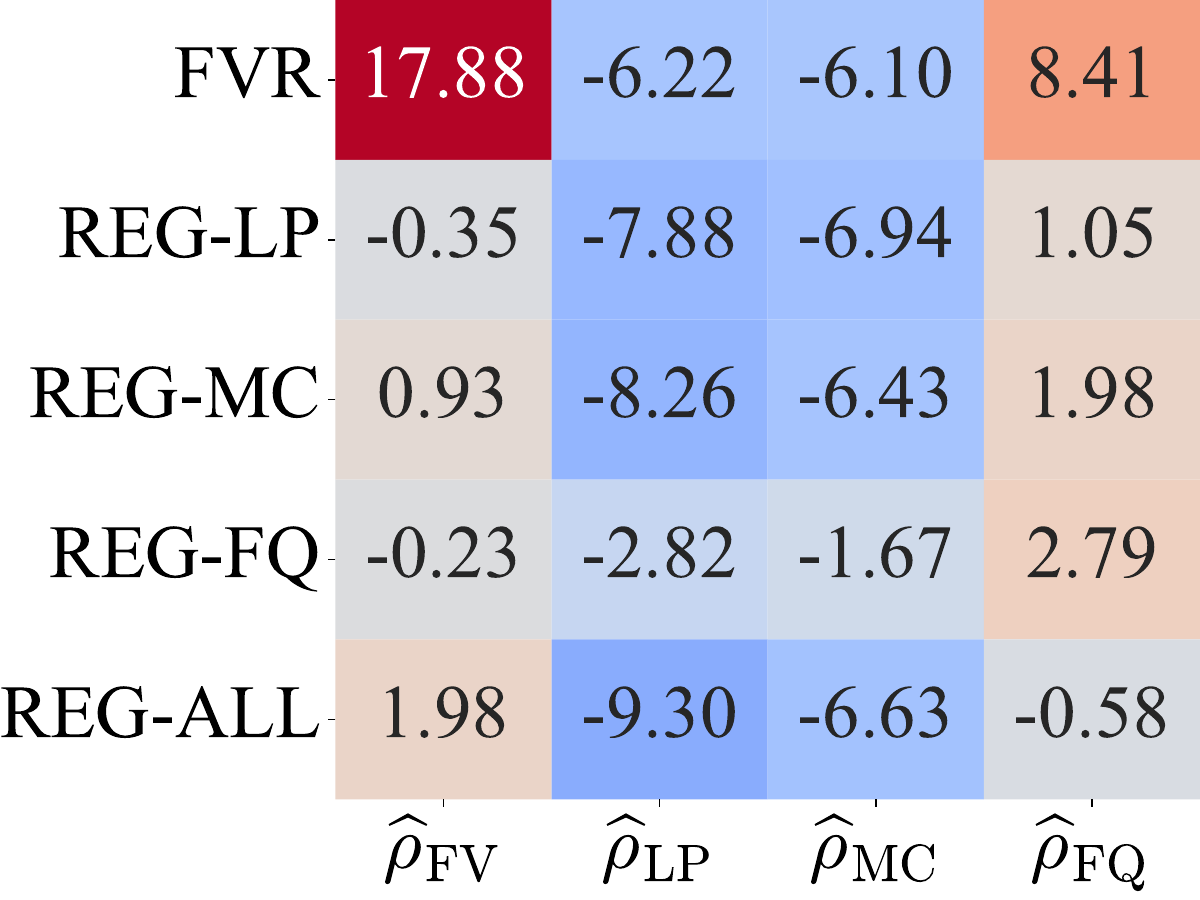}
      \caption{Goedel-Prover-8B}
      \label{fig:gdp}
    \end{subfigure}
    \caption{Average of local responsiveness (\%) for LLMs as different generators.}
    \label{fig:res}
\end{figure}

\textbf{Responsiveness maps can illustratively prioritize settings that exhibit the greatest expected gain across the soft dimensions.} Among all REG settings, Seed-Coder-8B with REG-LP achieves the highest average expected gain with respect to LP (5.44\%), Qwen2.5-7B with REG-MC attains the highest gain for MC (8.02\%), and DeepSeek-Prover-7B with REG-FQ yields the largest gain for FQ (2.91\%).

\textbf{Specialized LLMs are more responsive to formal validity, whereas general-purpose LLMs tend to focus on soft semantic properties.} Across all settings, Qwen2.5-7B exhibits no positive local responsiveness (Figure~\ref{fig:qwen}). Seed-Coder-8B likewise shows negative responsiveness toward FV in REG settings and only limited responsiveness when configured as FVR (Figure~\ref{fig:seed}). In contrast, both DeepSeek-Prover-7B and Goedel-Prover-8B demonstrate substantial responsiveness ($>$10\%) when configured as FVR (Figures~\ref{fig:dsp}, \ref{fig:gdp}). These results indicate that specialized LLMs are more responsive to FV. Specialized models also exhibit stronger responsiveness in formal quality compared to general-purpose LLMs. However, for dimensions less directly tied to formal structure (LP and MC), specialized LLMs show limited refinement capability, particularly Goedel-Prover-8B. This contrast suggests that general-purpose and specialized LLMs emphasize different aspects of refinement, reflecting a trade-off between formal and informal dimensions.

\subsection{Dataset-Level Performance of Recurrent Generators}
We report the average improvement and Cohen's d on the miniF2F and ProofNet datasets in Figure~\ref{fig:cohen}.

\begin{figure}[!t]
    \centering
    \begin{subfigure}{0.24\textwidth} 
      \centering
      \includegraphics[width=\textwidth]{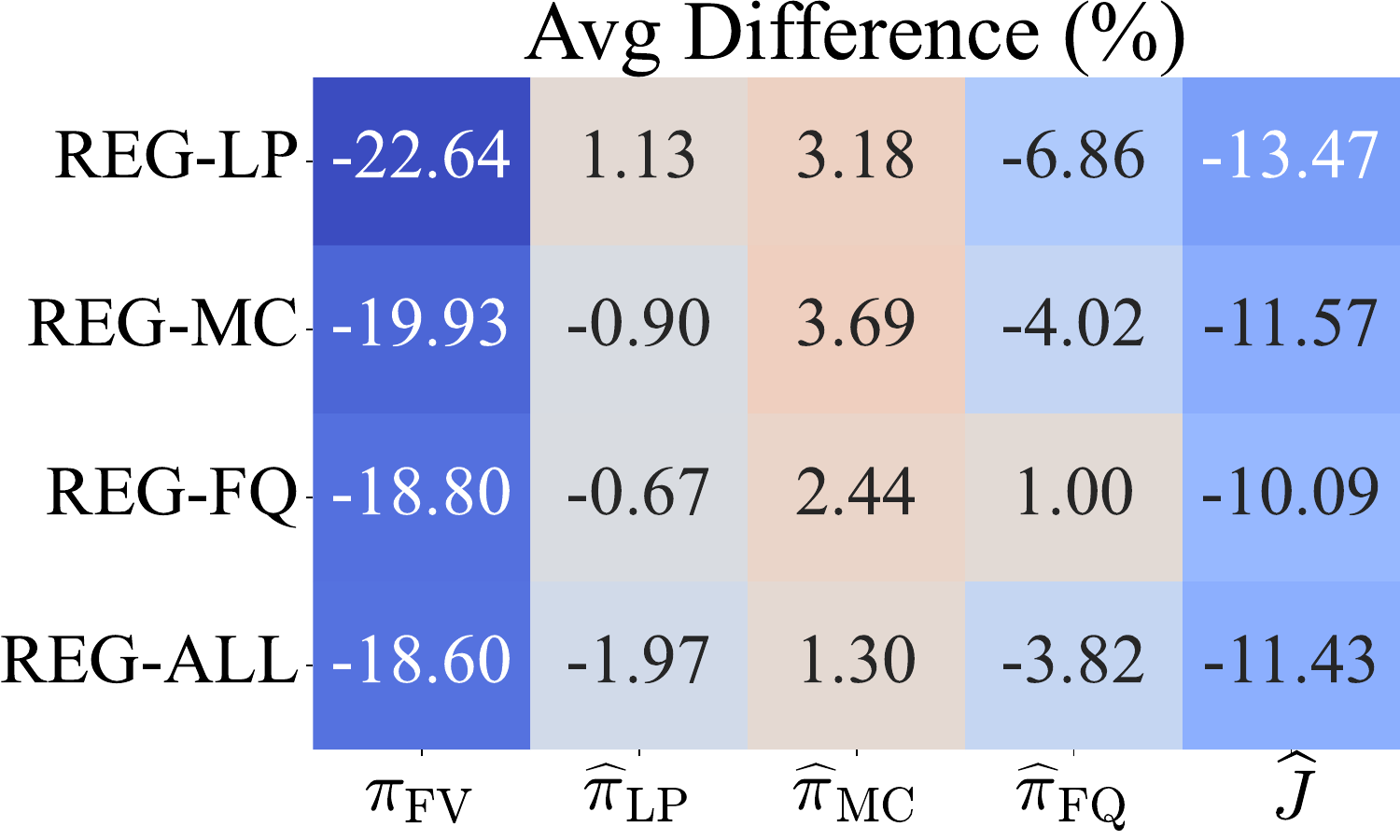}
      \caption{miniF2F}
      \label{fig:ma}
    \end{subfigure}
    \begin{subfigure}{0.24\textwidth} 
      \centering
      \includegraphics[width=\textwidth]{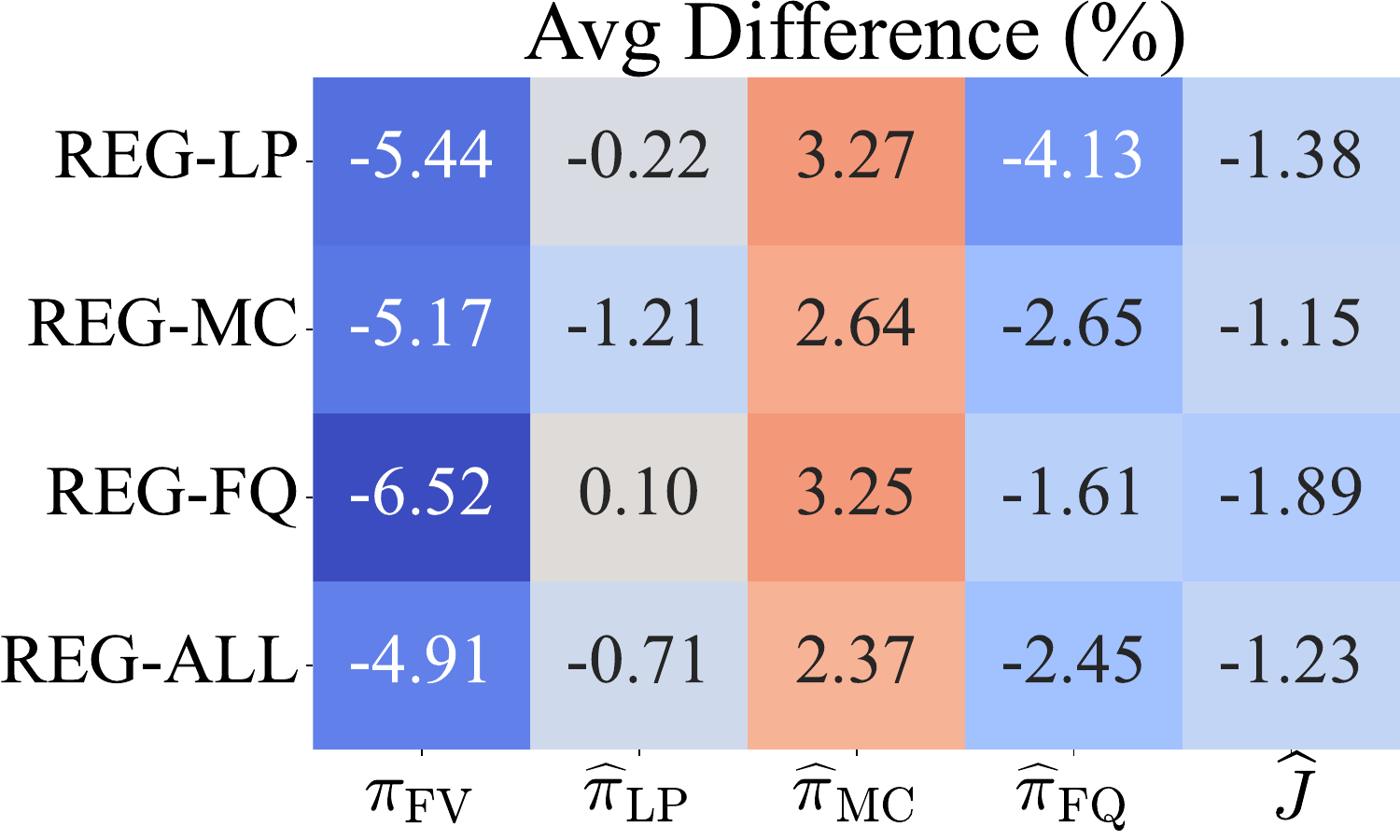}
      \caption{ProofNet}
      \label{fig:pa}
    \end{subfigure}
    \begin{subfigure}{0.24\textwidth} 
      \centering
      \includegraphics[width=\textwidth]{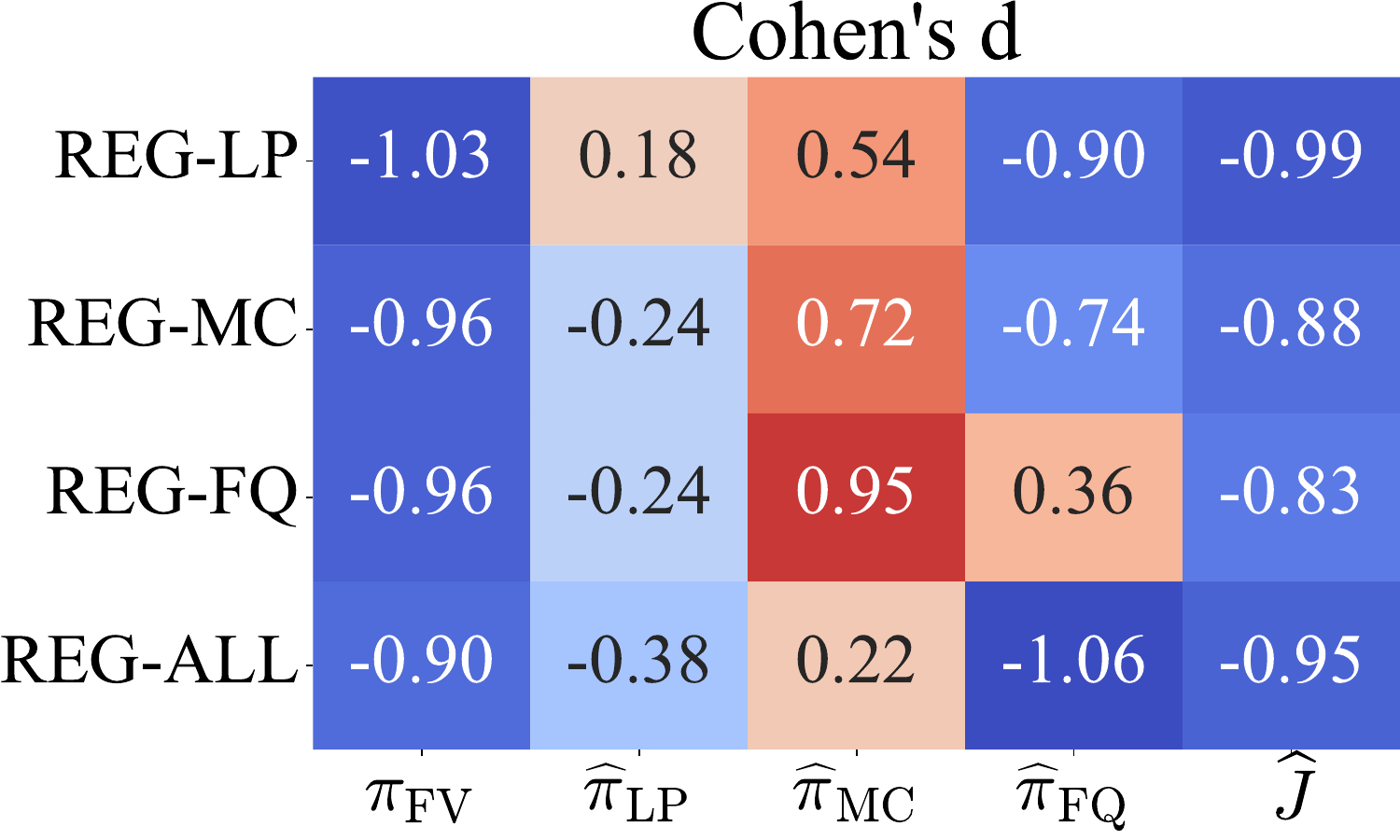}
      \caption{miniF2F}
      \label{fig:mc}
    \end{subfigure}
    \begin{subfigure}{0.24\textwidth} 
      \centering
      \includegraphics[width=\textwidth]{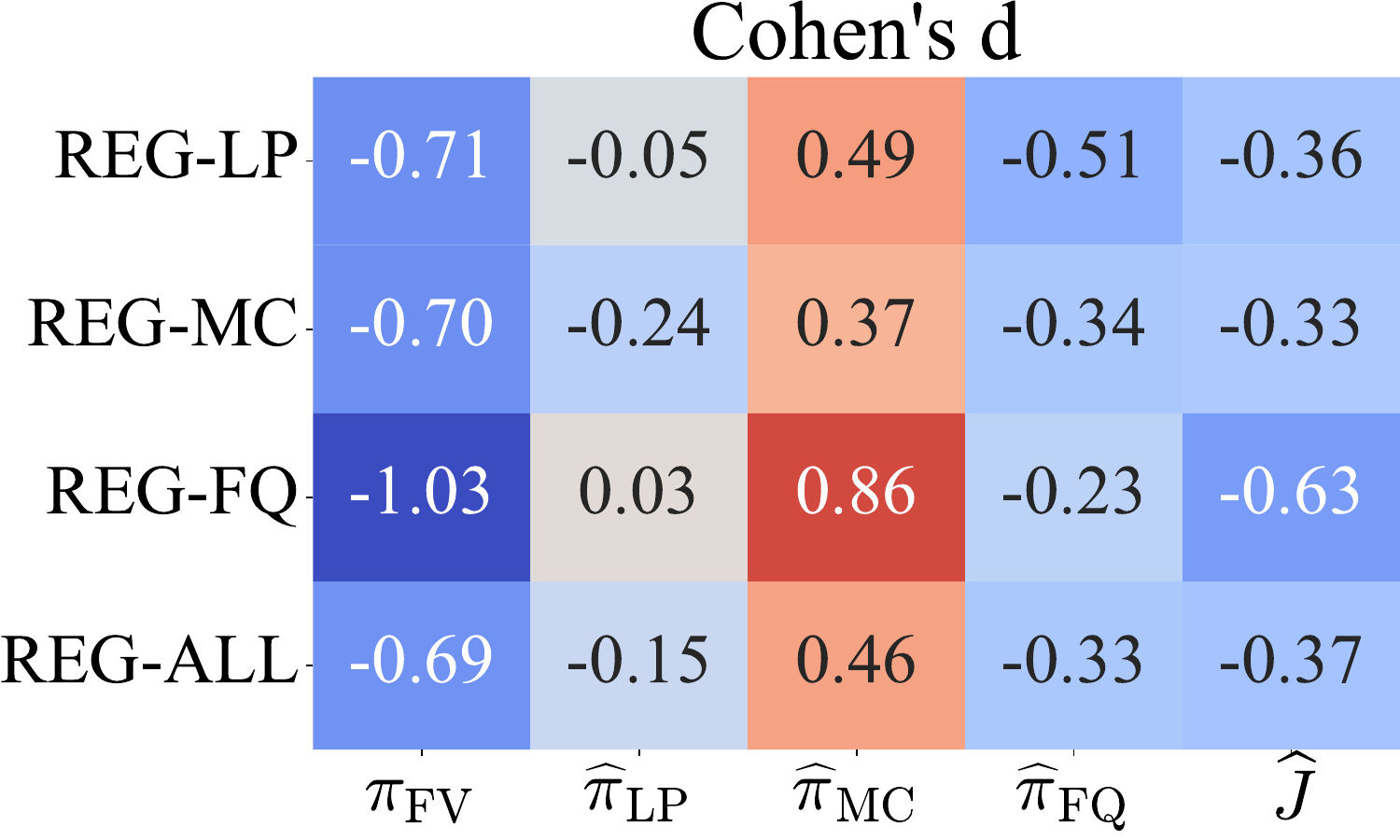}
      \caption{ProofNet}
      \label{fig:pc}
    \end{subfigure}
    \caption{Average and Cohen's d of improvements using different feedback on each dataset.}
    \label{fig:cohen}
\end{figure}

\textbf{Providing feedback on a single dimension yields targeted improvements on that dimension.} In Figure~\ref{fig:ma}, on miniF2F, REG-LP, REG-MC, REG-FQ consistently lead to average improvements in $\widehat{\pi}_\LP$, $\widehat{\pi}_\MC$, $\widehat{\pi}_\FQ$, respectively. In Figure~\ref{fig:pa} and \ref{fig:pc}, on ProofNet, REG-MC likewise exhibits a positive effect on MC judgments. Taken together, these results suggest that providing LLMs with feedback focused on a specific dimension is an effective strategy for improving performance along that dimension.

\textbf{However, direct feedback on a given dimension is not necessarily the most effective way to improve that dimension.} Across both datasets, all REG settings lead to improvements in MC. Notably, the method achieving the largest Cohen's d for $\widehat{\pi}_\MC$ is REG-FQ rather than REG-MC (Figure~\ref{fig:mc}, \ref{fig:pc}). This result suggests that refining formalizations using feedback on formal quality can have a statistically larger impact on mathematical consistency than providing MC-specific feedback alone.

\textbf{Refining one dimension does not guarantee the preservation of other dimensions.} All REG settings lead to decreases in formal validity, even when error details are provided alongside LLM judgments. Moreover, in some cases, non-targeted dimensions are negatively affected by giving feedback focused on a specific target dimension. This pattern highlights the difficulty of jointly optimizing multiple dimensions of formalizations and further underscores the necessity of our proposed iterative refinement process for more reliable autoformalization.

\subsection{Uncertainty Estimation}\label{app:unc}

\begin{figure}[!t]
    \centering
    \begin{subfigure}{0.16\textwidth} 
      \centering
      \includegraphics[width=\textwidth]{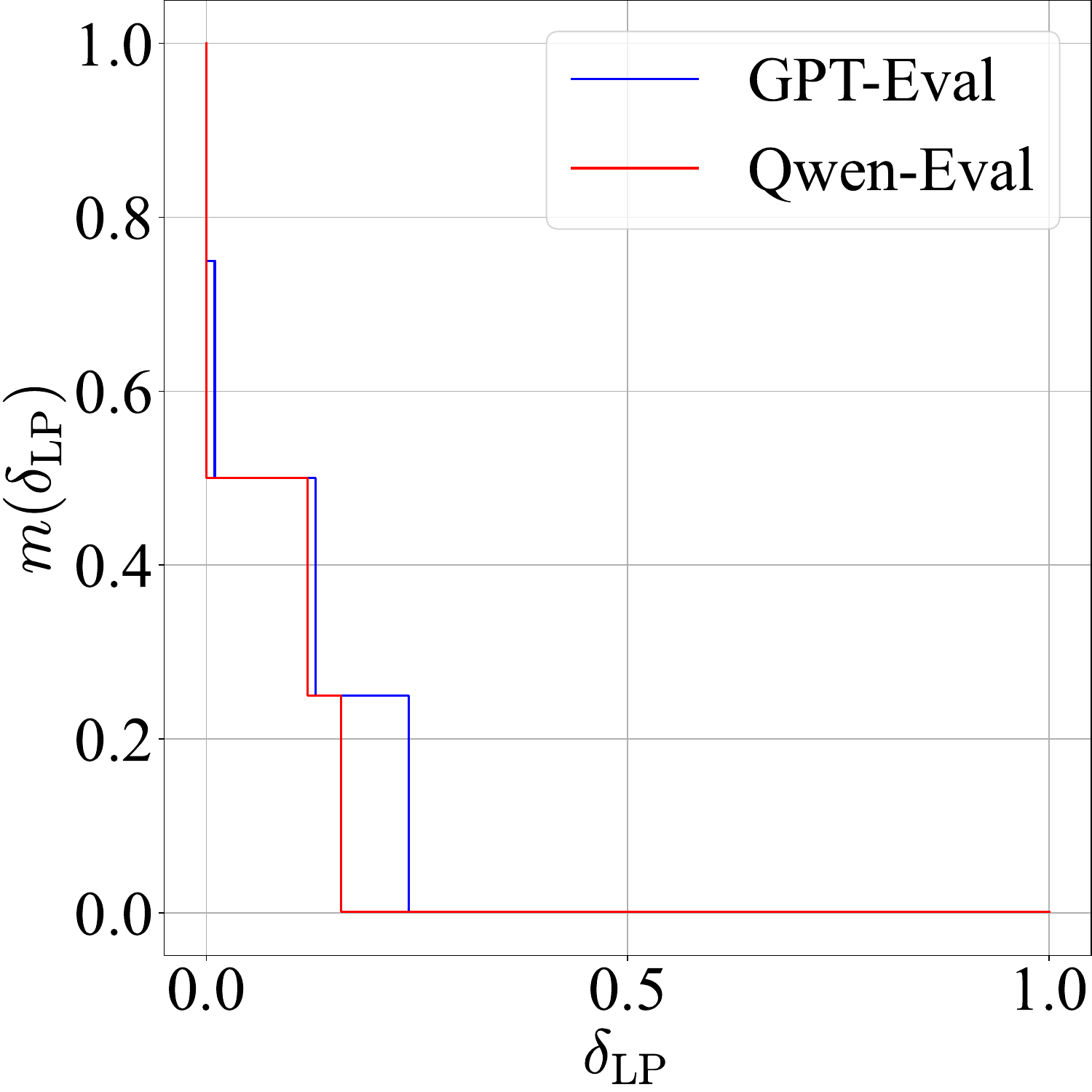}
      \caption{miniF2F LP}
      \label{fig:lp_1}
    \end{subfigure}
    \begin{subfigure}{0.16\textwidth} 
      \centering
      \includegraphics[width=\textwidth]{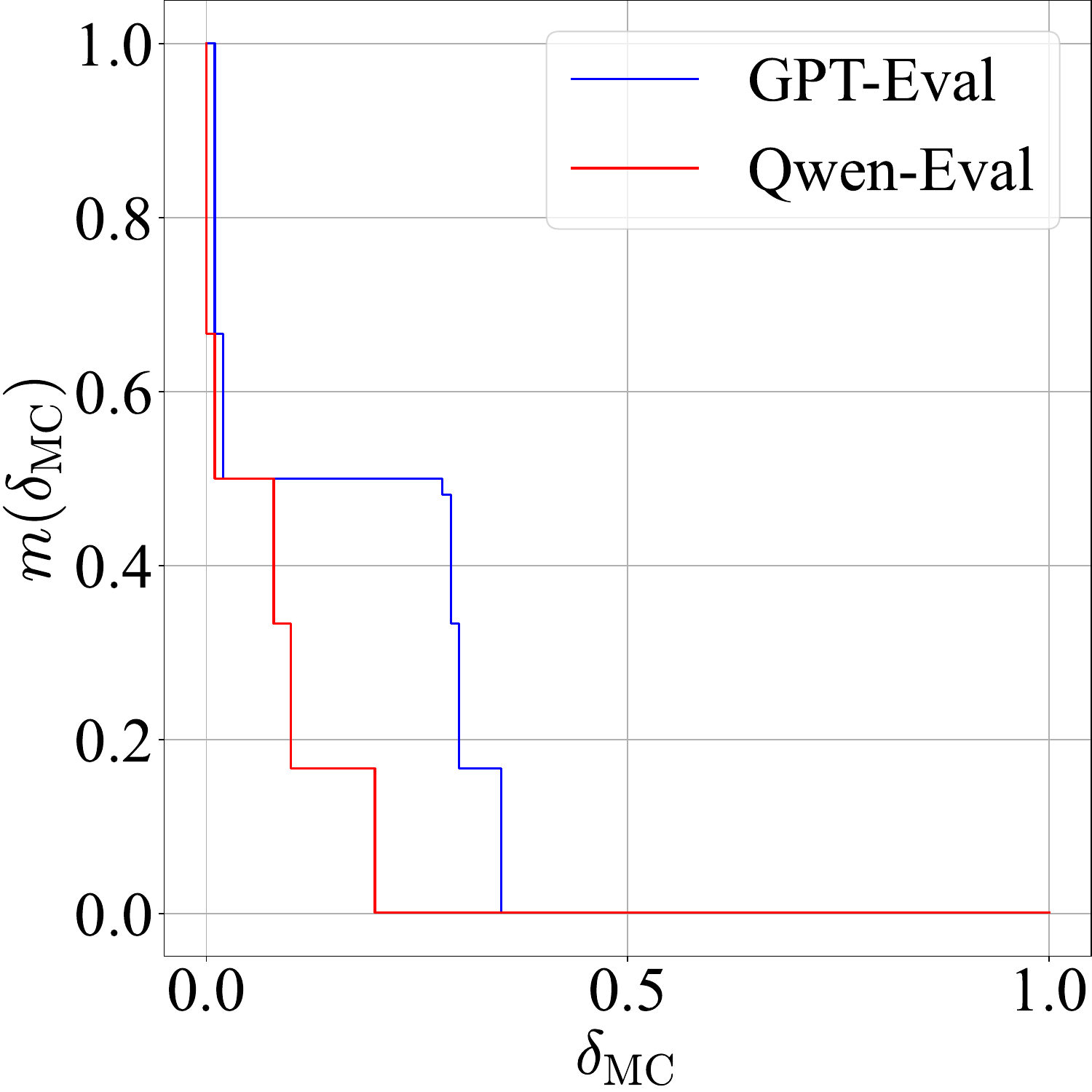}
      \caption{miniF2F MC}
      \label{fig:mc_1}
    \end{subfigure}
    \begin{subfigure}{0.16\textwidth} 
      \centering
      \includegraphics[width=\textwidth]{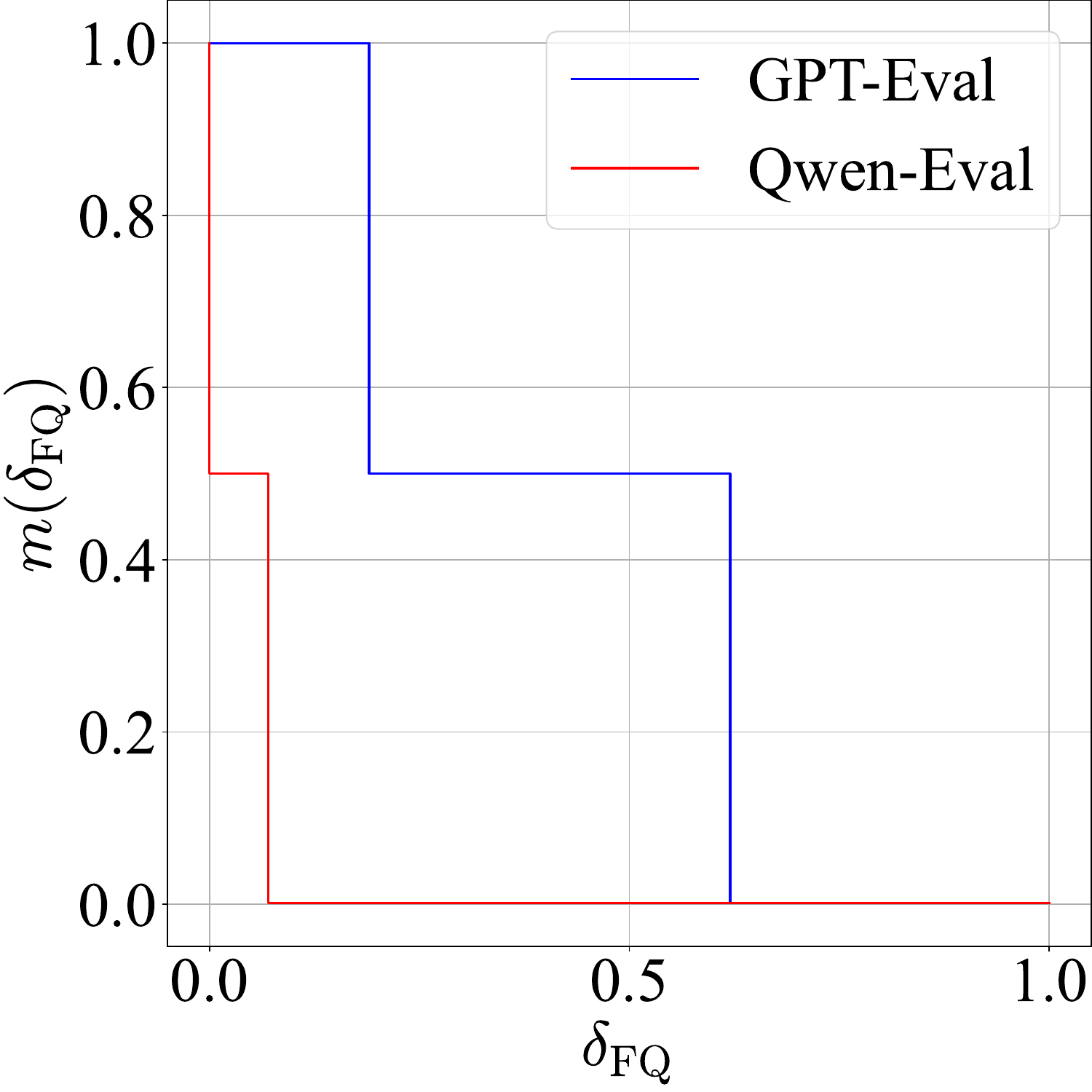}
      \caption{miniF2F FQ}
      \label{fig:fq_1}
    \end{subfigure}
    \begin{subfigure}{0.16\textwidth} 
      \centering
      \includegraphics[width=\textwidth]{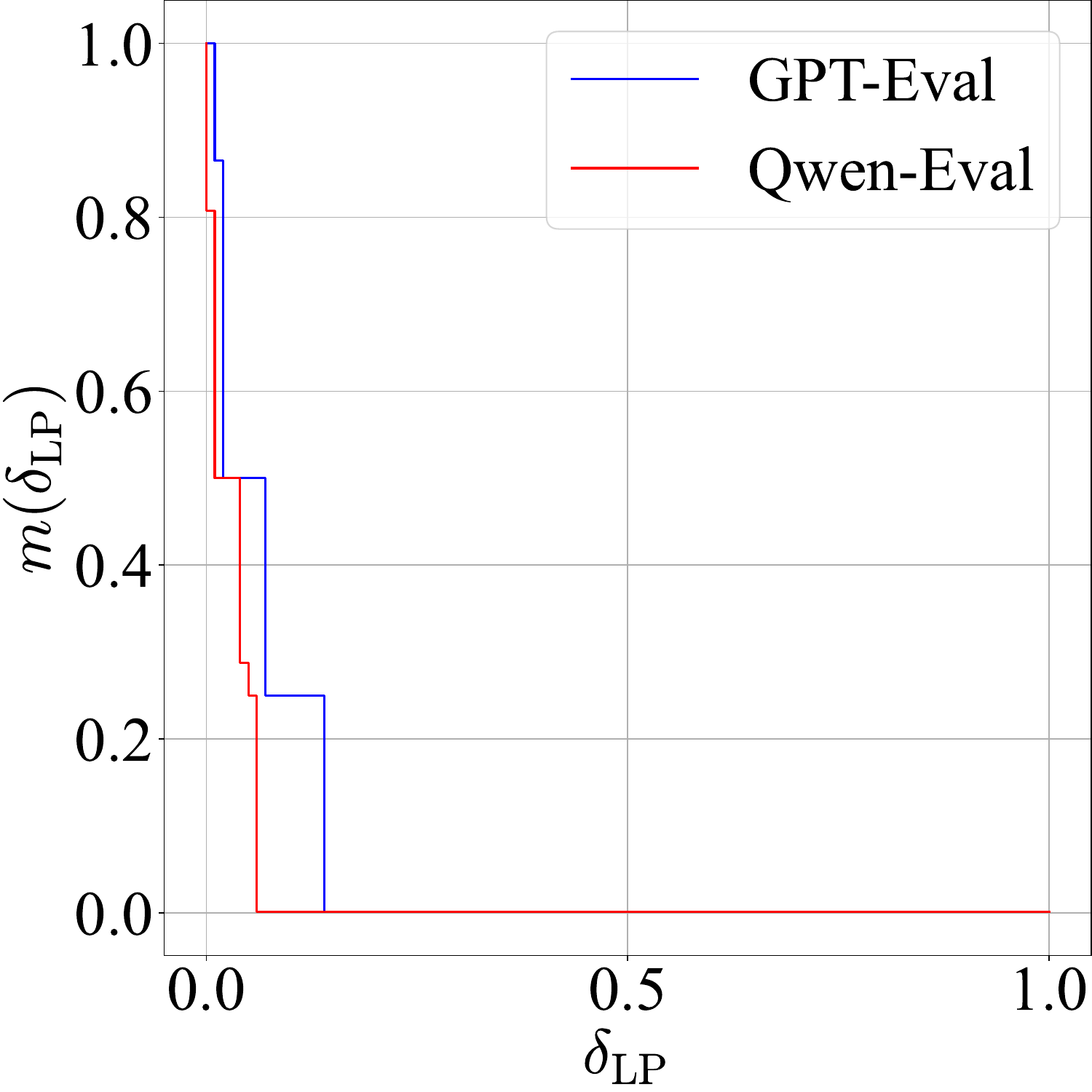}
      \caption{Proofnet LP}
      \label{fig:lp_2}
    \end{subfigure}
    \begin{subfigure}{0.16\textwidth} 
      \centering
      \includegraphics[width=\textwidth]{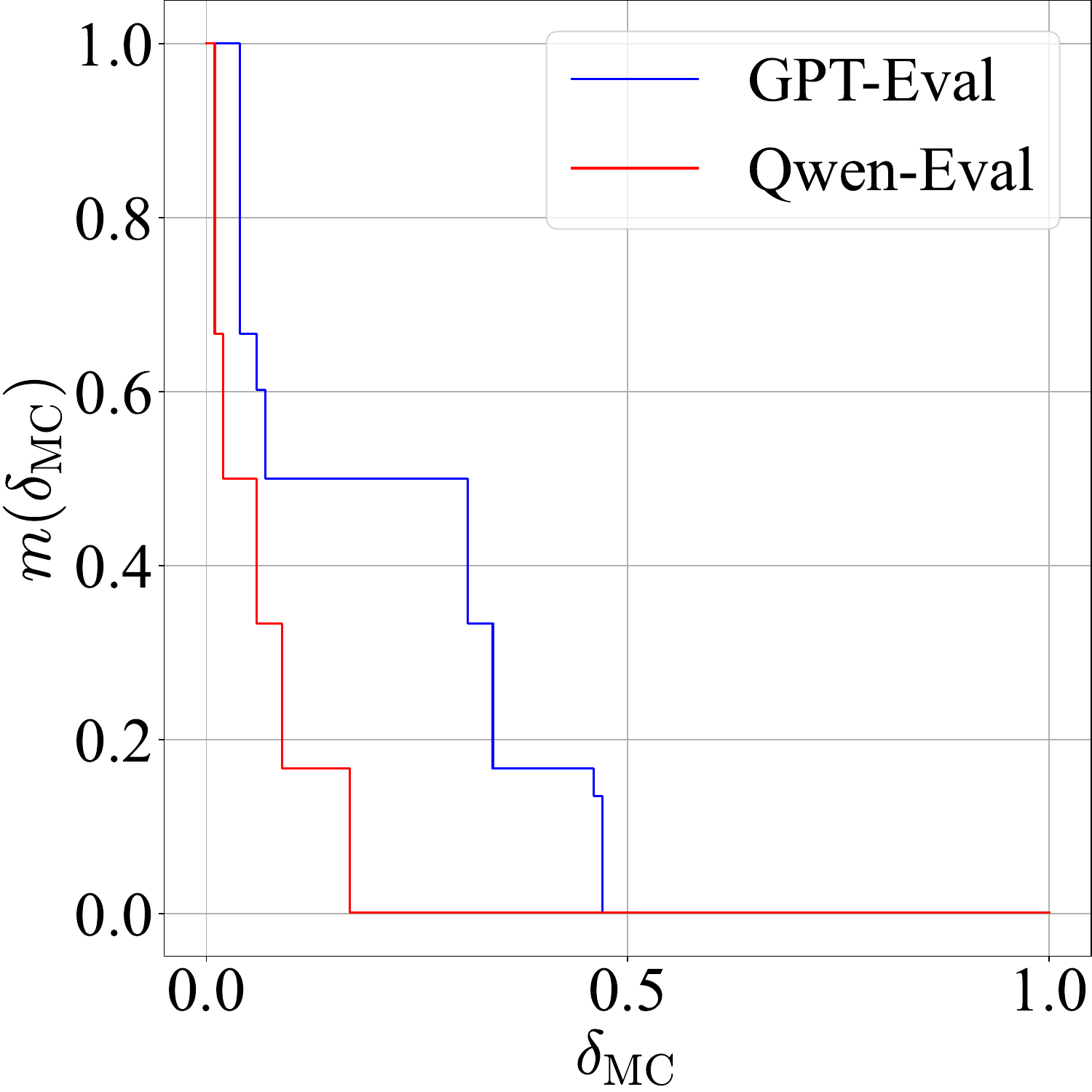}
      \caption{ProofNet MC}
      \label{fig:mc_2}
    \end{subfigure}
    \begin{subfigure}{0.16\textwidth} 
      \centering
      \includegraphics[width=\textwidth]{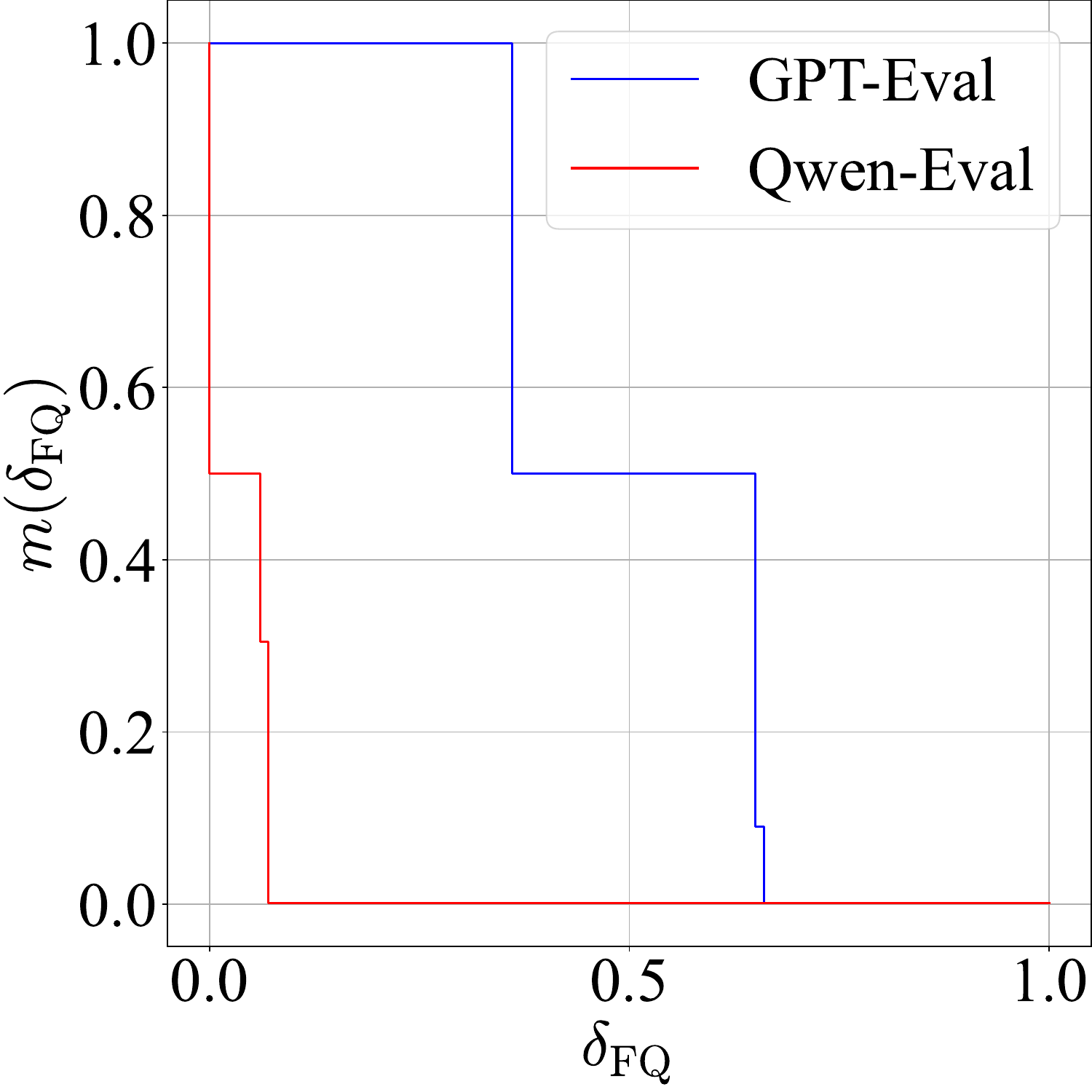}
      \caption{ProofNet FQ}
      \label{fig:fq_2}
    \end{subfigure}
    \caption{Uncertainty estimation of LLM judges.}
    \label{fig:uncertainty}
\end{figure}

Using the results of the human evaluation as an empirical reference of $\pi_i(x)$ in Eq.~\ref{eq:lcb_i}, we compute the curves of $\delta_i$ and $m(\delta_i)$, which correspond to the maximum admissible margin $m_i(\delta_i)$ for the LLM judges used in our experiments (GPT-4.1-mini and Qwen2.5-Coder-7B). Figure~\ref{fig:uncertainty} illustrates the behavior of LLM judges under uncertainty.

\textbf{There is a trade-off between uncertainty level and the effort required to calibrate LLM judges.} Across all settings, we observe that the required margin $m(\delta_i)$ increases as the uncertainty level $\delta_i$ decreases. This suggests that achieving lower uncertainty requires stronger calibration of LLM judgments. Conversely, using less-calibrated LLM judges leads to higher uncertainty levels.

\textbf{LLM judges exhibit different levels of uncertainty across different aspects.} On both datasets, both LLM judges are more confident in logical preservation (LP) than in mathematical consistency (MC), and least confident in formal quality (FQ). This behavior indicates that MC and FQ are harder to align with human evaluations for LLM judges.

\textbf{While Qwen2.5-Coder-7B requires less calibration at the same uncertainty level, it is a more conservative judge than GPT-4.1-mini for autoformalization evaluation and is less informative for selection.} The curves corresponding to GPT-4.1-mini consistently lie above those of Qwen2.5-Coder-7B, indicating that larger margins are needed to guarantee the same confidence level. This suggests that GPT-4.1-mini exhibits higher variability and requires stronger correction to obtain reliable positive judgments. However, Qwen2.5-Coder-7B is a conservative evaluator that tends to assign lower scores. Although it has lower uncertainty in the lower-bound sense, it also does not provide sufficient discrimination for selecting better formalizations.

\section{Details of Prompts}\label{app:prompt}

\subsection{Prompts for Generators}
\begin{tcolorbox}[title={One-Off Generator}, breakable]
\textbf{System Prompt:}

You are an expert in formal language Lean4.

You will be given a mathematical statement and its proof written in natural language and LaTeX symbols.

Your task is to provide the formal code of the given natural language mathematical statement and its proof in Lean4 with the following instructions:

1. You should give the formal code directly without any additional comments or explanations.

2. In case that you need to import any necessary preambles, you should not import any fake (non-exist) preambles.

3. You should wrap the formal code in a way illustrated as the following:

\%\%\%\%\%\%\%\%\%\%

Your Formal Code

\%\%\%\%\%\%\%\%\%\%

Strictly follow the instructions that have been claimed.

\vspace{1em}

\textbf{User Prompt:}

Natural language statement: \{Natural Language Statement\}

Natural language proof: \{{Natural Language Proof}\}

Give me the Lean4 formal code of them:
\end{tcolorbox}

\begin{tcolorbox}[title={FV-Repairer}, breakable]
\textbf{System Prompt:}

You are an expert in formal language Lean4.

You will be given a mathematical statement and its proof written in natural language and LaTeX symbols.

You will also be given a formal code which attempted to describe the given mathematical statement and its proof in Lean4.

Your task is to refine the given formal code to make it correct while maintaining the alignment with the given natural language mathematical statement and proof.

Here are some instructions for your task:

1. You should give the formal code directly without any additional comments or explanations.

2. In case that you need to import any necessary preambles, you should not import any fake (non-exist) preambles.

3. You should wrap the formal code in a way illustrated as the following:

\%\%\%\%\%\%\%\%\%\%

Your Formal Code

\%\%\%\%\%\%\%\%\%\%

Strictly follow the instructions that have been claimed.

\vspace{1em}

\textbf{User Prompt:}

Natural language statement: \{Natural Language Statement\}

Natural language proof: \{{Natural Language Proof}\}

There are some Lean4 formal codes describing the given mathematical statement and its proof: \{Formal Code\}

You should refine the formal code for your task to make it correct.

Here are some feedbacks about the formal code which can be used to help your task: \{According to the theorem prover, the error details of the provided formal code are:

Error Details

\}

\tcblower

\textbf{Example of Error Details:}

Error on line 5, start column 0, end column 1: unexpected token '\#'; expected command

Error on line 138, start column 0, end column None: unexpected end of input; expected '\{' or indented tactic sequence

\end{tcolorbox}

\begin{tcolorbox}[title={Recurrent Generator}, breakable]
\textbf{System Prompt:}

You are an expert in formal language Lean4.

You will be given a mathematical statement and its proof written in natural language and LaTeX symbols.

You will also be given a formal code which attempted to describe the given mathematical statement and its proof in Lean4.

Your task is to refine the given formal code to make it correct while maintaining the alignment with the given natural language mathematical statement and proof.

Here are some instructions for your task:

1. You should give the formal code directly without any additional comments or explanations.

2. In case that you need to import any necessary preambles, you should not import any fake (non-exist) preambles.

3. You should wrap the formal code in a way illustrated as the following:

\%\%\%\%\%\%\%\%\%\%

Your Formal Code

\%\%\%\%\%\%\%\%\%\%

Strictly follow the instructions that have been claimed.

\vspace{1em}

\textbf{User Prompt:}

Natural language statement: \{Natural Language Statement\}

Natural language proof: \{{Natural Language Proof}\}

There are some Lean4 formal codes describing the given mathematical statement and its proof: \{Formal Code\}

You should refine the formal code for your task to make it correct.

Here are some feedbacks about the formal code which can be used to help your task: \{Feedback from LLM Judges\}

\tcblower

\textbf{Example of Feedback from LLM Judges:}

Aspect: Is the formalized code expressed in a minimal, non-redundant form, avoiding unnecessary repetition or complexity?

Evaluation: Explanation: The formalized code is not expressed in a minimal, non-redundant form. It contains multiple intermediate steps (h4, h5, h6, h7, h8) that are stated but left as ``sorry," indicating incomplete proofs or placeholders. While these steps reflect the logical progression of the proof, they are somewhat verbose and could be streamlined by directly manipulating the logarithmic equalities or by combining steps to avoid repetition. For example, the relationships between the logarithms of x, y, and z could be derived more compactly without separately stating each intermediate equality. Additionally, the final step h9 is also a ``sorry," so the proof is incomplete and does not demonstrate minimality or efficiency in its current form. Overall, the code structure suggests redundancy and unnecessary complexity rather than a minimal, concise formalization. 

Judgement: False\newline

Aspect: Is the formalized code \emph{internally coherent and contains no contradictions} under the logical rules of the relevant formal system?

Evaluation: Explanation: The formalized code is internally coherent and contains no contradictions under the logical rules of Lean4 and the real number theory it uses. The theorem statement is well-formed, assuming the standard definitions of logarithms (Real.logb) and the domain conditions ($x, y, z > 1, w > 0$). The proof structure is consistent: it introduces variables and hypotheses, then states intermediate lemmas (h4 to h9) that relate the logarithms of w and the bases $x, y, z$. Although the proof steps are marked as ``sorry" (placeholders for omitted proofs), this does not introduce contradictions; it simply means the proof is incomplete but not inconsistent. The chain of equalities and substitutions aligns with the natural language proof logic, and the final conclusion matches the problem statement. Therefore, the code is internally coherent and free of contradictions.

Judgement: True
\end{tcolorbox}

\subsection{Prompts for LLM Judges}
We use the Operable Atomic Properties (OAPs) in ~\citep{zhang2025goldstandardsepistemicensemble} to estimate scores for soft-dimensions. OAPs within one soft-dimension are considered equally (i.e., for $i\in\{\LP,\MC,\FQ\}$, $w_{i,k}$ is the same for every $k$ in Eq.~\ref{eq:judge}). The descriptions of OAPs used in the prompts are listed as follows:
\begin{description}
    \item[For $\widehat{\pi}_\LP$] 1. \textit{Pre-arg Structure}: `Does the formalized code reflect the \emph{inherent predicate-argument structure} of the natural language statement?'; 2. \textit{Quantification}: `Does the formalized code accurately formalize all \emph{quantifiers}, such as universal and existential, present in the natural language statement?'; 3. \textit{Formula}: `Are all \emph{mathematical formulas and expressions} in the natural language statement, such as equations and inequalities, correctly and completely represented in the formalized code?'; 4. \textit{Relation}: `Are the \emph{logical and mathematical relationships between propositions} in the natural language statement preserved in the formalized code?'.
    \item[For $\widehat{\pi}_\MC$] 1. \textit{Concept}: `Are all \emph{mathematical concepts} mentioned in the natural language statement, such as integers, fractions, real numbers, complex numbers, derivatives, integrals, vectors, matrices, probabilities, expectations, and variances, correctly formalized in the formalized code?'; 2. \textit{Constant}: `Are all \emph{mathematical constants} mentioned in the natural language statement, such as 1, $\frac{2}{3}$, $\pi$, $e$, properly included in the formalized code?'; 3. \textit{Operator}: `Are all \emph{mathematical operators} used in the natural language statement, such as addition, subtraction, multiplication, division, summation, exponentiation, and product, correctly represented in the formalized code?'.
    \item[For $\widehat{\pi}_\FQ$] 1. \textit{Conciseness}: `Is the formalized code expressed in a minimal, non-redundant form, avoiding unnecessary repetition or complexity?'; 2. \textit{Logical Consistency}: `Is the formalized code \emph{internally coherent and contains no contradictions} under the logical rules of the relevant formal system?'.
\end{description}

\begin{tcolorbox}[title={LLM Judges}, breakable]
\textbf{System Prompt:}

You are an expert in formal language Lean4.

You will be given a mathematical statement and its proof written in natural language and LaTeX symbols.

You will also be given a formal code which attempted to describe the given mathematical statement and its proof in Lean4.

Your task is to evaluate a specific aspect of the formal code.

The description of the aspect is: \{OAP Description\}

Your need to give two things about your evaluation:

1. the judgement of whether the formalization meets this aspect. This should be a binary value in ``True" or ``False".

2. the detailed explanation of your judgement.

You should wrap your final results in a way illustrated as the following:

\%\%\%\%\%\%\%\%\%\%

Explanation: Your Detailed Explanation

Judgement: Your Binary Judgement

\%\%\%\%\%\%\%\%\%\%

Strictly follow the instructions that have been claimed.

\vspace{1em}

\textbf{User Prompt:}

Natural language statement: \{Natural Language Statement\}

Natural language proof: \{{Natural Language Proof}\}

There are some Lean4 formal codes describing the given mathematical statement and its proof: \{Formal Code\}
\end{tcolorbox}

\section{Human Evaluation Criteria}\label{app:human}
We provide the detailed human evaluation criteria of formalizations in Table~\ref{tab:anno}. Logical Preservation (LP), Mathematical Consistency (MC), and Formal Quality (FQ) are rated on a 3-point scale. The criteria are intended to support consistent judgments, although inter-annotator agreement has not been established.

\begin{table}[!t]
  \caption{Annotation criteria}
  \centering
  \begin{tabular}{p{0.95\textwidth}}
    \toprule
    \textbf{Logical Preservation (LP)}\newline
    From your point of view, does the formal code capture the logical structure and content of the natural language statement?\newline
    \textbf{1}: Fully captures the logical structure and content.\newline
    \textbf{0.5}: Partially captures the logical structure and content.\newline
    \textbf{0}: Does not capture the logical structure or content.\\
    \midrule
    \textbf{Mathematical Consistency (MC)}\newline
    From your point of view, does the formal code accurately represent mathematical objects and operations present in the natural language statement?\newline
    \textbf{1}: Accurately represents all mathematical objects and operations.\newline
    \textbf{0.5}: Partially represents the involved mathematical objects and operations.\newline
    \textbf{0}: Does not represent the involved mathematical objects or operations.\\
    \midrule
    \textbf{Formal Quality (FQ)}\newline
    From your point of view, does the formalized code alone demonstrate high quality in terms of structural clarity, readability, and usefulness?\newline
    \textbf{1}: Demonstrates high structural clarity, readability, and usefulness.\newline
    \textbf{0.5}: Demonstrates moderate structural clarity, readability, or usefulness.\newline
    \textbf{0}: Does not demonstrate adequate structural clarity, readability, or usefulness.\\
    \bottomrule
  \end{tabular}
  \label{tab:anno}
\end{table}

\section{Limitations}\label{app:lim}

\textbf{One-sided uncertainty modeling of LLM judges.} The uncertainty modeling of LLM judges in this work focuses on a one-sided lower confidence bound (LCB), rather than full two-sided uncertainty intervals. This choice is motivated by the design of our optimization framework, where candidate acceptance is based on estimated composite scores, while lower confidence bounds characterize theoretical behavior. As a result, our notion of reliability is also defined in terms of lower-bound guarantees. A more comprehensive treatment of bidirectional uncertainty estimation is an interesting direction but is beyond the scope of this study.

\textbf{Potential bias from using LLM judges in both optimization and evaluation.} LLM-based judges are used both in the test-time optimization loop and in reporting soft-dimension evaluation metrics, which may introduce systematic bias due to shared model preferences. To mitigate this concern, we decouple formal correctness from LLM judgments by using the Lean4 theorem prover as an external verifier for formal validity. In addition, human and cross-model evaluation are conducted to validate the observed trends of the process. In our framework, LLM judges are explicitly modeled as noisy and imperfect estimators rather than ground-truth oracles. The acceptance rule is designed to reduce sensitivity to judge noise and promote conservative updates. While these design choices reduce potential bias, we acknowledge that some residual coupling between optimization and evaluation may remain, particularly in soft metrics.

\textbf{Dependence on single-annotator human evaluation for uncertainty estimation.} The empirical estimation of uncertainty parameters in the LCB formulation relies on human evaluation conducted by a single annotator. While this provides a consistent labeling signal, it may not fully capture inter-annotator variability, and thus may introduce noise in the absolute calibration of uncertainty estimates. Nevertheless, the evaluation is primarily used to characterize relative trends rather than to establish absolute ground-truth values. Future work could explore multi-annotator protocols or probabilistic annotation models to refine uncertainty estimation. Importantly, the proposed process itself does not depend on perfect calibration of these parameters, and its qualitative behavior remains consistent under reasonable perturbations of the estimates.

\textbf{Computational cost of ensemble-based optimization.} The proposed process involves multiple LLM generations and evaluations, and thus incurs higher computational cost compared to single-pass inference methods. This overhead is a direct consequence of its ensemble-style decision mechanism and iterative refinement procedure. However, this cost is offset by improved performance and stability across tasks. In addition, the framework naturally allows performance to scale with increased computation budget (e.g., more iterations or stronger judges), suggesting a form of test-time scaling behavior. Exploring more efficient variants of the iterative refinement process is a promising direction for future work.

\section{Additional Results, Tables, and Figures}\label{app:add}
The explicit results of the multi-dimensional reference-free refinement process with ensemble, with GPT-5.4 only, and the iterative self-refinement process are provided in Table~\ref{tab:mono}, \ref{tab:mono_gpt}, and \ref{tab:isr}, respectively. The costs of different settings, in terms of the number of LLM calls, are reported in Table~\ref{tab:cost_minif2f} and \ref{tab:cost_proofnet}. We provide a running example of full-theorem autoformalization with our process for ``Dummit-Foote\_exercise\_1\_1\_20'' from ProoNet in Figure~\ref{fig:example}. The explicit results of Recurrent Generators (REGs) are reported in Table~\ref{tab:reg}.

\begin{table}[!t]
  \caption{Performance of the multi-dimensional reference-free refinement process with ensemble. All numbers are reported in percentage.}
  \scriptsize
  \centering
  \begin{tabular}{c c c c c c c @{\hspace{2em}} c c c c c c}
    \toprule
    & \multicolumn{6}{c}{miniF2F-Test} & \multicolumn{6}{c}{ProofNet-Test}\\
    \mycdashline{2-13}
    $t$ & $\pi_\FV$ & $\widehat{\pi}_\LP$ & $\widehat{\pi}_\MC$  & $\widehat{\pi}_\FQ$ & $\widehat{J}$ & $\widehat{J}=1$ & $\pi_\FV$ & $\widehat{\pi}_\LP$ & $\widehat{\pi}_\MC$  & $\widehat{\pi}_\FQ$ & $\widehat{J}$ & $\widehat{J}=1$\\
    \midrule
    \multicolumn{13}{l}{\textit{GPT-4.1-mini Judges}}\\
    \midrule
    0 & 86.07 & 57.58 & 80.60 & 60.86 & 57.03 & 8.61 & 36.02 & 48.52 & 72.04 & 57.80 & 18.75 & 3.23\\
    1 & 95.90 & 70.90 & 89.21 & 75.82 & 75.33 & 27.46 & 54.84 & 59.01 & 80.47 & 75.54 & 34.12 & 9.68\\
    2 & 99.59 & 75.41 & 91.53 & 82.58 & 82.76 & 35.66 & 62.90 & 61.83 & 81.00 & 79.57 & 41.37 & 12.37\\
    3 & 100.00 & 79.82 & 94.54 & 86.27 & 86.87 & 42.21 & 67.74 & 61.56 & 80.47 & 82.53 & 45.86 & 16.13\\
    4 & 100.00 & 82.79 & 96.17 & 89.34 & 89.44 & 46.72 & 73.12 & 60.48 & 79.39 & 82.53 & 49.26 & 17.20\\
    5 & 100.00 & 83.71 & 96.72 & 90.37 & 90.27 & 49.18 & 77.96 & 58.06 & 78.32 & 82.80 & 52.45 & 17.74\\
    \midrule
    \multicolumn{13}{l}{\textit{Qwen2.5-Coder-7B Judges}}\\
    \midrule
    0 & - & 32.68 & 46.99 & 15.98 & 26.10 & 0.82 & - & 26.75 & 35.13 & 12.63 & 7.78 & 0.54\\
    1 & - & 36.89 & 50.41 & 21.52 & 34.14 & 1.64 & - & 23.79 & 34.77 & 16.94 & 11.23 & 1.61\\
    2 & - & 39.34 & 52.46 & 22.95 & 37.96 & 2.46 & - & 22.45 & 33.69 & 17.20 & 12.44 & 2.15\\
    3 & - & 41.09 & 54.64 & 23.57 & 39.77 & 2.87 & - & 19.89 & 33.51 & 17.20 & 12.70 & 2.15\\
    4 & - & 40.47 & 56.01 & 23.77 & 40.08 & 3.28 & - & 19.49 & 32.97 & 16.13 & 13.78 & 2.15\\
    5 & - & 41.39 & 56.28 & 24.80 & 40.82 & 2.87 & - & 18.68 & 32.62 & 14.52 & 14.46 & 2.15\\
    \midrule
    \multicolumn{13}{l}{\textit{Human Evaluation}}\\
    \midrule
    0 & - & 68.24 & 68.03 & 33.61 & 49.46 & 1.23& - & 62.63 & 57.53 & 24.46 & 15.71 & 1.61\\
    1 & - & 73.98 & 75.82 & 36.89 & 59.63 & 4.92 & - & 69.62 & 63.44 & 22.58 & 28.25 & 5.91\\
    2 & - & 75.41 & 77.46 & 38.11 & 63.32 & 7.38 & - & 72.04 & 65.86 & 21.78 & 33.62 & 6.45\\
    3 & - & 77.46 & 79.71 & 38.32 & 65.16 & 8.20 & - & 70.70 & 65.32 & 20.70 & 35.77 & 6.99\\
    4 & - & 78.69 & 81.35 & 38.11 & 66.05 & 8.20 & - & 70.97 & 66.94 & 21.24 & 38.63 & 8.06\\
    5 & - & 78.89 & 81.76 & 39.14 & 66.60 & 9.02 & - & 68.55 & 68.55 & 21.77 & 40.96 & 8.06\\
    \bottomrule
  \end{tabular}
  \label{tab:mono}
\end{table}

\begin{table}[!t]
  \caption{Performance of the multi-dimensional reference-free refinement process with a single LLM (GPT-5.4).}
  \scriptsize
  \centering
  \begin{tabular}{c c c c c c c @{\hspace{2em}} c c c c c c}
    \toprule
    & \multicolumn{6}{c}{miniF2F-Test} & \multicolumn{6}{c}{ProofNet-Test}\\
    \mycdashline{2-13}
    $t$ & $\pi_\FV$ & $\widehat{\pi}_\LP$ & $\widehat{\pi}_\MC$  & $\widehat{\pi}_\FQ$ & $\widehat{J}$ & $\widehat{J}=1$ & $\pi_\FV$ & $\widehat{\pi}_\LP$ & $\widehat{\pi}_\MC$  & $\widehat{\pi}_\FQ$ & $\widehat{J}$ & $\widehat{J}=1$\\
    \midrule
    \multicolumn{6}{l}{\textit{GPT-4.1-mini Judges}}\\
    \midrule
    0 & 43.44 & 81.15 & 94.26 & 70.49 & 37.20 & 17.62 & 11.29 & 86.29 & 95.16 & 75.00 & 10.23 & 5.91\\
    1 & 52.87 & 86.78 & 96.72 & 80.74 & 46.76 & 26.64 & 16.13 & 91.40 & 97.49 & 82.26 & 14.83 & 9.68\\
    2 & 56.56 & 86.78 & 96.45 & 84.43 & 50.27 & 31.15 & 17.20 & 92.88 & 98.39 & 86.56 & 16.10 & 11.83\\
    3 & 59.02 & 88.01 & 96.58 & 85.04 & 52.96 & 34.84 & 20.43 & 93.28 & 98.21 & 89.78 & 19.03 & 13.98\\
    4 & 59.43 & 89.65 & 97.54 & 86.48 & 54.34 & 37.30 & 20.97 & 93.55 & 98.39 & 91.67 & 19.65 & 15.05\\
    5 & 60.66 & 90.16 & 97.54 & 87.09 & 55.64 & 38.52 & 21.51 & 94.49 & 98.57 & 92.47 & 20.19 & 15.59\\
    \midrule
    \multicolumn{6}{l}{\textit{Qwen2.5-Coder-7B Judges}}\\
    \midrule
    0 & - & 57.48 & 62.57 & 34.22 & 20.94 & 4.10 & - & 56.85 & 56.99 & 41.67 & 5.92 & 2.69\\
    1 & - & 59.53 & 62.70 & 35.66 & 24.74 & 4.92 & - & 57.53 & 56.81 & 41.67 & 8.14 & 3.76\\
    2 & - & 59.73 & 62.70 & 37.09 & 26.88 & 4.92 & - & 60.48 & 58.60 & 42.74 & 8.84 & 3.76\\
    3 & - & 59.53 & 63.93 & 36.89 & 28.26 & 4.92 & - & 60.08 & 57.89 & 42.74 & 9.78 & 3.76\\
    4 & - & 59.43 & 63.80 & 36.07 & 28.17 & 4.92 & - & 60.75 & 58.24 & 44.09 & 10.07 & 3.76\\
    5 & - & 59.22 & 63.66 & 35.86 & 29.01 & 4.92 & - & 61.83 & 58.60 & 44.35 & 10.42 & 3.76\\
    \bottomrule
  \end{tabular}
  \label{tab:mono_gpt}
\end{table}

\begin{table}[!t]
  \caption{Performance of iterative self-refinement (ISR) with error feedback from the theorem prover and no acceptance policy.}
  \scriptsize
  \centering
  \begin{tabular}{c c c c c c c @{\hspace{2em}} c c c c c c}
    \toprule
    & \multicolumn{6}{c}{miniF2F-Test} & \multicolumn{6}{c}{ProofNet-Test}\\
    \mycdashline{2-13}
    $t$ & $\pi_\FV$ & $\widehat{\pi}_\LP$ & $\widehat{\pi}_\MC$  & $\widehat{\pi}_\FQ$ & $\widehat{J}$ & $\widehat{J}=1$ & $\pi_\FV$ & $\widehat{\pi}_\LP$ & $\widehat{\pi}_\MC$  & $\widehat{\pi}_\FQ$ & $\widehat{J}$ & $\widehat{J}=1$\\
    \midrule
    \multicolumn{6}{l}{\textit{ISR-DeepSeek-Prover-V2-7B}}\\
    \midrule
    0 & 50.41 & 47.75 & 80.05 & 45.08 & 32.23 & 6.97 & 14.52 & 41.94 & 69.71 & 38.17 & 7.75 & 1.08 \\
    1 & 47.13 & 48.16 & 78.28 & 43.03 & 29.47 & 6.15 & 8.06 & 40.59 & 68.64 & 32.53 & 5.46 & 2.15\\
    2 & 48.77 & 49.90 & 81.42 & 40.57 & 30.87 & 5.74 & 9.14 & 39.38 & 69.18 & 29.57 & 6.91 & 2.15\\
    3 & 49.59 & 51.33 & 80.05 & 39.34 & 31.74 & 5.74 & 8.06 & 40.19 & 69.53 & 34.95 & 5.76 & 2.69\\
    4 & 48.77 & 50.20 & 78.83 & 37.70 & 30.56 & 4.92 & 7.53 & 42.20 & 72.58 & 30.38 & 5.48 & 2.69\\
    5 & 50.41 & 52.46 & 81.01 & 40.78 & 33.06 & 6.15 & 6.99 & 42.20 & 71.68 & 31.72 & 5.26 & 2.15\\
    \midrule
    \multicolumn{6}{l}{\textit{ISR-Goedel-Prover-V2-8B}}\\
    \midrule
    0 & 65.16 & 46.62 & 75.82 & 47.54 & 39.12 & 2.46 & 22.04 & 33.47 & 58.60 & 36.02 & 9.93 & 1.61\\
    1 & 60.66 & 40.06 & 69.67 & 48.98 & 36.04 & 2.05 & 17.74 & 25.40 & 53.05 & 35.48 & 8.91 & 2.69\\
    2 & 60.66 & 38.52 & 68.44 & 46.31 & 35.08 & 1.64 & 18.28 & 26.08 & 54.84 & 33.33 & 7.60 & 2.15\\
    3 & 59.43 & 38.32 & 65.98 & 44.88 & 34.37 & 2.46 & 21.51 & 27.55 & 53.76 & 36.02 & 9.20 & 1.61\\
    4 & 57.38 & 37.60 & 65.44 & 48.98 & 32.70 & 2.05 & 20.97 & 25.13 & 52.51 & 36.56 & 8.38 & 1.08\\
    5 & 55.74 & 37.30 & 65.98 & 46.72 & 30.72 & 1.64 & 23.66 & 23.52 & 52.33 & 35.22 & 10.21 & 1.08\\
    \midrule
    \multicolumn{6}{l}{\textit{ISR-GPT-5.4}}\\
    \midrule
    0 & 30.74 & 80.94 & 94.54 & 72.54 & 26.84 & 14.75 & 9.68 & 90.32 & 97.31 & 73.92 & 9.04 & 3.76\\
    1 & 44.67 & 80.43 & 95.22 & 78.48 & 39.32 & 21.31 & 10.75 & 86.69 & 94.09 & 75.81 & 9.99 & 6.45\\
    2 & 49.18 & 81.66 & 93.99 & 75.20 & 42.93 & 22.13 & 15.05 & 86.56 & 96.06 & 77.42 & 14.02 & 8.60\\
    3 & 53.28 & 80.84 & 94.26 & 76.64 & 46.28 & 22.54 & 17.74 & 84.68 & 96.06 & 71.51 & 16.60 & 11.83\\
    4 & 52.46 & 79.82 & 95.49 & 74.18 & 45.89 & 23.77 & 18.82 & 88.71 & 95.70 & 77.42 & 17.54 & 11.83\\
    5 & 54.92 & 81.25 & 95.08 & 76.23 & 48.01 & 23.36 & 18.82 & 86.56 & 95.70 & 75.00 & 17.63 & 12.37\\
    \bottomrule
  \end{tabular}
  \label{tab:isr}
\end{table}

\begin{table}[!t]
  \caption{Costs of different settings measured in LLM calls for miniF2F experiments. Mean and standard deviation are computed on a per-sample basis. "Total" denotes the aggregate over all samples in the datasets.}
  \centering
  \begin{tabular}{l c c c @{\hspace{2em}} c c c}
    \toprule
    & \multicolumn{3}{c}{Generator Calls} & \multicolumn{3}{c}{Judge Calls}\\
    \mycdashline{2-7}
    & Mean & Std & Total & Mean & Std & Total\\
    \midrule
    \multicolumn{7}{l}{\textit{Best-of-$k$, $k=8$}}\\
    \midrule
    DeepSeek-Prover-V2-7B & 8.00 & - & 1,952 & 46.99 & 20.22 & 11,466\\
    Goedel-Prover-V2-8B & 8.00 & - & 1,952 & 54.26 & 17.67 & 13,239\\
    GPT-5.4 & 8.00 & - & 1,952 & 51.82 & 25.59 & 12,645\\
    \midrule
    \multicolumn{7}{l}{\textit{Best-of-$k$, $k=64$}}\\
    \midrule
    DeepSeek-Prover-V2-7B & 64.00 & - & 15,616 & 175.98 & 143.70 & 42,939\\
    Goedel-Prover-V2-8B & 64.00 & - & 15,616 & 297.55 & 167.60 & 72,603\\
    \midrule
    \multicolumn{7}{l}{\textit{Iterative Self-Refinement (6 Iterations)}}\\
    \midrule
    Per Iteration & 1.00 & - & 244 & 9.00 & - & 2,196\\
    Overall & 6.00 & - & 1,464 & 54.00 & - & 13,176\\
    \midrule
    \multicolumn{7}{l}{\textit{Multi-Dimensional Reference-Free Refinement Process -- Single (GPT-5.4)}}\\
    \midrule
    Iteration 0 & 1.67 & 0.47 & 408 & 9.00 & 0.00 & 2,196\\
    Iteration 1 & 1.45 & 0.77 & 354 & 6.97 & 3.76 & 1,701\\
    Iteration 2 & 1.33 & 0.87 & 324 & 5.83 & 4.30 & 1,422\\
    Iteration 3 & 1.24 & 0.90 & 302 & 5.53 & 4.38 & 1,350\\
    Iteration 4 & 1.19 & 0.92 & 290 & 5.02 & 4.47 & 1,224\\
    Iteration 5 & 1.16 & 0.94 & 284 & 4.87 & 4.48 & 1,188\\
    Overall & 8.04 & 4.40 & 1,962 & 37.22 & 18.37 & 9,081\\
    \midrule
    \multicolumn{7}{l}{\textit{Multi-Dimensional Reference-Free Refinement Process -- Ensemble}}\\
    \midrule
    Iteration 0 & 3.78 & 1.48 & 922 & 26.00 & 6.67 & 6,345\\
    Iteration 1 & 10.39 & 4.02 & 2,535 & 39.87 & 25.92 & 9,729\\
    Iteration 2 & 8.32 & 5.49 & 2,031 & 25.67 & 23.48 & 6,264\\
    Iteration 3 & 7.55 & 5.87 & 1,841 & 21.06 & 21.32 & 5,139\\
    Iteration 4 & 6.78 & 6.06 & 1,655 & 18.00 & 20.84 & 4,392\\
    Iteration 5 & 6.28 & 6.12 & 1,532 & 16.78 & 20.24 & 4,095\\
    Overall & 43.10 & 25.43 & 10,516 & 147.39 & 81.62 & 35,964\\
    \bottomrule
  \end{tabular}
  \label{tab:cost_minif2f}
\end{table}

\begin{table}[!t]
  \caption{Costs of different settings measured in LLM calls for ProofNet experiments. Mean and standard deviation are computed on a per-sample basis. "Total" denotes the aggregate over all samples in the datasets.}
  \centering
  \begin{tabular}{l c c c @{\hspace{2em}} c c c}
    \toprule
    & \multicolumn{3}{c}{Generator Calls} & \multicolumn{3}{c}{Judge Calls}\\
    \mycdashline{2-7}
    & Mean & Std & Total & Mean & Std & Total\\
    \midrule
    \multicolumn{7}{l}{\textit{Best-of-$k$, $k=8$}}\\
    \midrule
    DeepSeek-Prover-V2-7B & 8.00 & - & 1,488 & 59.52 & 18.92 & 11,070\\
    Goedel-Prover-V2-8B & 8.00 & - & 1,488 & 55.89 & 19.70 & 10,395\\
    GPT-5.4 & 8.00 & - & 1,488 & 66.19 & 16.39 & 12,312\\
    \midrule
    \multicolumn{7}{l}{\textit{Best-of-$k$, $k=64$}}\\
    \midrule
    DeepSeek-Prover-V2-7B & 64.00 & - & 11,904 & 227.42 & 172.93 & 42,300\\
    Goedel-Prover-V2-8B & 64.00 & - & 11,904 & 248.76 & 171.93 & 46,269\\
    \midrule
    \multicolumn{7}{l}{\textit{Iterative Self-Refinement (6 Iterations)}}\\
    \midrule
    Per Iteration & 1.00 & - & 186 & 9.00 & - & 1,674\\
    Overall & 6.00 & - & 1,116 & 54.00 & - & 10,044\\
    \midrule
    \multicolumn{7}{l}{\textit{Multi-Dimensional Reference-Free Refinement Process -- Single (GPT-5.4)}}\\
    \midrule
    Iteration 0 & 1.92 & 0.26 & 358 & 9.00 & 0.00 & 1,674\\
    Iteration 1 & 1.87 & 0.48 & 347 & 8.23 & 2.52 & 1,530\\
    Iteration 2 & 1.78 & 0.60 & 331 & 7.84 & 3.02 & 1,458\\
    Iteration 3 & 1.75 & 0.65 & 325 & 7.60 & 3.26 & 1,413\\
    Iteration 4 & 1.71 & 0.70 & 318 & 7.31 & 3.52 & 1,359\\
    Iteration 5 & 1.69 & 0.72 & 315 & 7.26 & 3.56 & 1,350\\
    Overall & 10.72 & 3.09 & 1,994 & 47.23 & 14.38 & 8,784\\
    \midrule
    \multicolumn{7}{l}{\textit{Multi-Dimensional Reference-Free Refinement Process -- Ensemble}}\\
    \midrule
    Iteration 0 & 5.26 & 1.17 & 978 & 32.66 & 5.25 & 6,075\\
    Iteration 1 & 13.61 & 2.92 & 2,532 & 61.84 & 35.96 & 11,502\\
    Iteration 2 & 12.58 & 4.40 & 2,340 & 47.08 & 38.89 & 8,757\\
    Iteration 3 & 12.21 & 4.80 & 2,271 & 41.13 & 38.44 & 7,650\\
    Iteration 4 & 11.66 & 5.34 & 2,168 & 36.05 & 37.90 & 6,705\\
    Iteration 5 & 11.47 & 5.44 & 2,134 & 30.73 & 36.26 & 5,715\\
    Overall & 66.79 & 21.31 & 12,423 & 249.48 & 166.59 & 46,404\\
    \bottomrule
  \end{tabular}
  \label{tab:cost_proofnet}
\end{table}

\begin{figure}[!t]
    \centering
    \includegraphics[width=\textwidth]{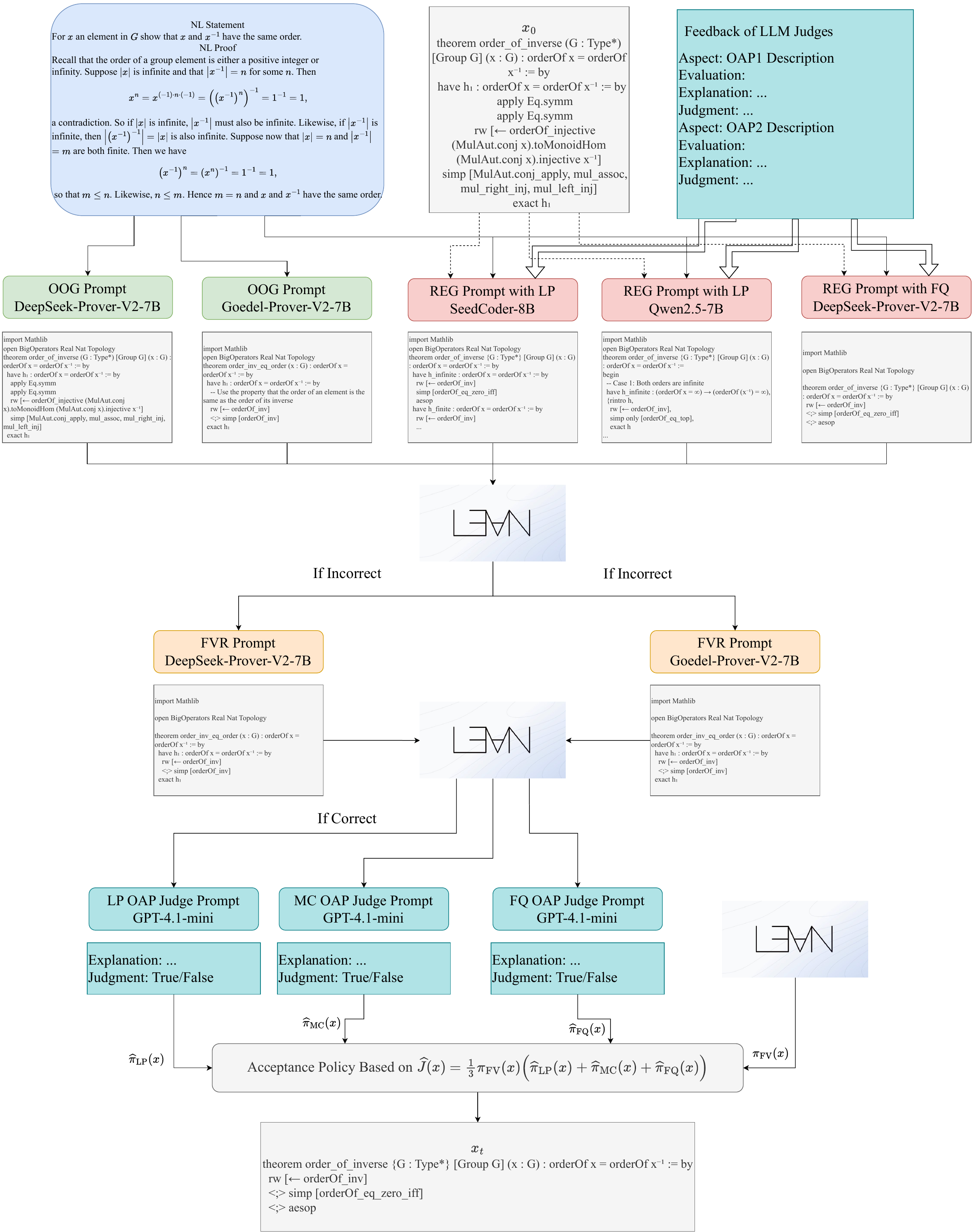}
    \caption{Illustrative example of a single step in the multi-dimensional reference-free refinement process. The process enables the construction of high-quality formalizations that are both formally valid, as verified by the theorem prover, and semantically aligned, as evaluated by LLM judges.}
    \label{fig:example}
\end{figure}

\begin{table}[!t]
  \caption{Performance of refinement with feedback. All results are based on zero-shot autoformalization.}
  \tiny
  \centering
  \begin{tabular}{l l l l c c c c c}
    \toprule
    & Generator Type & LLM & Feedback & $\pi_\FV$ & $\widehat{\pi}_\LP$ & $\widehat{\pi}_\MC$  & $\widehat{\pi}_\FQ$ & $\widehat{J}$\\
    \midrule
    \multicolumn{9}{l}{\textit{miniF2F-Test}}\\
    \midrule
    Baseline & One-Off Generator & DeepSeek-Prover-V2-7B & - & 50.41 & 55.02 & 80.87 & 51.02 & 35.15\\
    \mycdashline{1-9}
     & Recurrent Generator & Qwen2.5-7B-Instruct & LP & 4.10 & 56.05 & 86.20 & 34.84 & 3.53\\
     & Recurrent Generator & Qwen2.5-7B-Instruct & MC & 4.92 & 55.43 & 85.52 & 39.34 & 4.17\\
     & Recurrent Generator & Qwen2.5-7B-Instruct & FQ & 7.79 & 50.72 & 84.84 & 49.59 & 6.89\\
     & Recurrent Generator & Qwen2.5-7B-Instruct & ALL & 4.10 & 53.89 & 87.30 & 43.44 & 3.72\\
     & Recurrent Generator & Seed-Coder-8B-Instruct & LP & 22.54 & 64.86 & 89.34 & 45.90 & 19.51\\
     & Recurrent Generator & Seed-Coder-8B-Instruct & MC & 31.97 & 57.07 & 87.70 & 52.25 & 26.14\\
     & Recurrent Generator & Seed-Coder-8B-Instruct & FQ & 25.41 & 56.66 & 82.65 & 56.15 & 21.94\\
     & Recurrent Generator & Seed-Coder-8B-Instruct & ALL & 28.28 & 55.53 & 81.83 & 49.39 & 22.42\\
     & Recurrent Generator & DeepSeek-Prover-V2-7B & LP & 50.82 & 59.02 & 84.29 & 51.64 & 36.04\\
     & Recurrent Generator & DeepSeek-Prover-V2-7B & MC & 49.18 & 56.15 & 84.43 & 51.43 & 35.44\\
     & Recurrent Generator & DeepSeek-Prover-V2-7B & FQ & 51.23 & 59.53 & 84.15 & 53.07 & 37.45\\
     & Recurrent Generator & DeepSeek-Prover-V2-7B & ALL & 50.82 & 56.86 & 83.06 & 51.43 & 36.79\\
     & Recurrent Generator & Goedel-Prover-V2-8B & LP & 50.00 & 44.98 & 73.50 & 48.57 & 33.45\\
     & Recurrent Generator & Goedel-Prover-V2-8B & MC & 51.23 & 47.44 & 75.82 & 48.77 & 35.36\\
     & Recurrent Generator & Goedel-Prover-V2-8B & FQ & 53.28 & 52.66 & 78.96 & 53.48 & 37.93\\
     & Recurrent Generator & Goedel-Prover-V2-8B & ALL & 52.46 & 41.91 & 71.58 & 45.70 & 32.50\\
    \midrule
    Baseline & One-Off Generator & Goedel-Prover-V2-8B & - & 58.61 & 47.34 & 73.63 & 49.80 & 35.61\\
    \mycdashline{1-9}
     & Recurrent Generator & Qwen2.5-7B-Instruct & LP & 0.82 & 51.23 & 83.06 & 28.69 & 0.82\\
     & Recurrent Generator & Qwen2.5-7B-Instruct & MC & 2.05 & 46.82 & 83.61 & 35.25 & 1.64\\
     & Recurrent Generator & Qwen2.5-7B-Instruct & FQ & 10.66 & 44.88 & 81.01 & 45.08 & 9.10\\
     & Recurrent Generator & Qwen2.5-7B-Instruct & ALL & 8.20 & 51.74 & 82.79 & 39.14 & 7.50\\
     & Recurrent Generator & Seed-Coder-8B-Instruct & LP & 17.62 & 52.77 & 81.56 & 39.96 & 13.62\\
     & Recurrent Generator & Seed-Coder-8B-Instruct & MC & 27.05 & 50.00 & 82.92 & 47.34 & 19.92\\
     & Recurrent Generator & Seed-Coder-8B-Instruct & FQ & 25.00 & 44.98 & 76.64 & 52.46 & 17.88\\
     & Recurrent Generator & Seed-Coder-8B-Instruct & ALL & 27.46 & 47.44 & 79.10 & 48.57 & 19.77\\
     & Recurrent Generator & DeepSeek-Prover-V2-7B & LP & 56.56 & 49.28 & 77.19 & 50.61 & 35.83\\
     & Recurrent Generator & DeepSeek-Prover-V2-7B & MC & 55.74 & 48.77 & 77.32 & 47.95 & 34.57\\
     & Recurrent Generator & DeepSeek-Prover-V2-7B & FQ & 57.38 & 49.08 & 75.68 & 50.82 & 35.54\\
     & Recurrent Generator & DeepSeek-Prover-V2-7B & ALL & 56.97 & 45.90 & 74.86 & 49.18 & 35.04\\
     & Recurrent Generator & Goedel-Prover-V2-8B & LP & 52.46 & 40.27 & 68.31 & 48.16 & 32.17\\
     & Recurrent Generator & Goedel-Prover-V2-8B & MC & 54.51 & 40.57 & 70.22 & 48.77 & 32.91\\
     & Recurrent Generator & Goedel-Prover-V2-8B & FQ & 54.92 & 45.59 & 73.63 & 50.61 & 35.30\\
     & Recurrent Generator & Goedel-Prover-V2-8B & ALL & 59.02 & 40.37 & 67.90 & 45.90 & 33.53\\
    \midrule
    \multicolumn{9}{l}{\textit{ProofNet-Test}}\\
    \midrule
    Baseline & One-Off Generator & DeepSeek-Prover-V2-7B & - & 11.83 & 41.80 & 67.38 & 35.75 & 7.14\\
    \mycdashline{1-9}
     & Recurrent Generator & Qwen2.5-7B-Instruct & LP & 0.54 & 40.46 & 75.99 & 22.04 & 0.31\\
     & Recurrent Generator & Qwen2.5-7B-Instruct & MC & 2.15 & 42.07 & 74.55 & 22.85 & 1.82\\
     & Recurrent Generator & Qwen2.5-7B-Instruct & FQ & 0.54 & 40.86 & 74.73 & 23.12 & 0.54\\
     & Recurrent Generator & Qwen2.5-7B-Instruct & ALL & 0.54 & 43.55 & 76.34 & 22.58 & 0.54\\
     & Recurrent Generator & Seed-Coder-8B-Instruct & LP & 6.45 & 45.83 & 76.16 & 31.45 & 4.69\\
     & Recurrent Generator & Seed-Coder-8B-Instruct & MC & 6.45 & 42.34 & 73.48 & 31.45 & 5.29\\
     & Recurrent Generator & Seed-Coder-8B-Instruct & FQ & 4.84 & 42.20 & 72.22 & 33.60 & 4.26\\
     & Recurrent Generator & Seed-Coder-8B-Instruct & ALL & 7.53 & 40.86 & 72.22 & 29.30 & 5.29\\
     & Recurrent Generator & DeepSeek-Prover-V2-7B & LP & 11.29 & 44.62 & 73.12 & 35.22 & 6.66\\
     & Recurrent Generator & DeepSeek-Prover-V2-7B & MC & 10.22 & 43.95 & 73.48 & 39.52 & 6.54\\
     & Recurrent Generator & DeepSeek-Prover-V2-7B & FQ & 11.83 & 46.24 & 72.76 & 40.86 & 7.50\\
     & Recurrent Generator & DeepSeek-Prover-V2-7B & ALL & 12.90 & 45.30 & 72.94 & 40.05 & 8.23\\
     & Recurrent Generator & Goedel-Prover-V2-8B & LP & 20.43 & 37.50 & 62.01 & 41.67 & 12.40\\
     & Recurrent Generator & Goedel-Prover-V2-8B & MC & 18.82 & 32.53 & 60.93 & 44.35 & 11.81\\
     & Recurrent Generator & Goedel-Prover-V2-8B & FQ & 11.83 & 40.46 & 67.03 & 39.52 & 8.12\\
     & Recurrent Generator & Goedel-Prover-V2-8B & ALL & 16.67 & 32.53 & 63.08 & 38.17 & 10.26\\
    \midrule
    Baseline & One-Off Generator & Goedel-Prover-V2-8B & - & 17.74 & 31.59 & 63.44 & 29.84 & 6.89\\
    \mycdashline{1-9}
     & Recurrent Generator & Qwen2.5-7B-Instruct & LP & 1.61 & 33.33 & 72.58 & 13.44 & 1.30\\
     & Recurrent Generator & Qwen2.5-7B-Instruct & MC & 1.08 & 33.74 & 71.86 & 15.59 & 1.08\\
     & Recurrent Generator & Qwen2.5-7B-Instruct & FQ & 1.61 & 31.85 & 70.07 & 18.28 & 1.57\\
     & Recurrent Generator & Qwen2.5-7B-Instruct & ALL & 1.61 & 35.62 & 70.61 & 17.47 & 1.52\\
     & Recurrent Generator & Seed-Coder-8B-Instruct & LP & 3.23 & 32.66 & 68.64 & 18.55 & 2.76\\
     & Recurrent Generator & Seed-Coder-8B-Instruct & MC & 4.30 & 33.60 & 70.43 & 22.58 & 3.29\\
     & Recurrent Generator & Seed-Coder-8B-Instruct & FQ & 3.23 & 32.12 & 66.67 & 25.27 & 2.66\\
     & Recurrent Generator & Seed-Coder-8B-Instruct & ALL & 4.84 & 30.91 & 63.98 & 24.46 & 3.54\\
     & Recurrent Generator & DeepSeek-Prover-V2-7B & LP & 15.05 & 35.48 & 67.56 & 32.80 & 7.97\\
     & Recurrent Generator & DeepSeek-Prover-V2-7B & MC & 14.52 & 34.14 & 68.46 & 30.11 & 6.82\\
     & Recurrent Generator & DeepSeek-Prover-V2-7B & FQ & 14.52 & 35.35 & 67.20 & 34.14 & 7.26\\
     & Recurrent Generator & DeepSeek-Prover-V2-7B & ALL & 16.13 & 34.95 & 66.31 & 33.87 & 7.57\\
     & Recurrent Generator & Goedel-Prover-V2-8B & LP & 16.13 & 21.91 & 53.41 & 34.14 & 8.65\\
     & Recurrent Generator & Goedel-Prover-V2-8B & MC & 19.35 & 21.51 & 51.25 & 34.68 & 9.98\\
     & Recurrent Generator & Goedel-Prover-V2-8B & FQ & 17.74 & 25.27 & 58.60 & 34.68 & 8.74\\
     & Recurrent Generator & Goedel-Prover-V2-8B & ALL & 18.82 & 24.19 & 56.81 & 36.83 & 8.96\\
    \bottomrule
  \end{tabular}
  \label{tab:reg}
\end{table}

\end{document}